\documentclass[11pt]{article}
\usepackage{macros}
\usepackage{setspace}
\usepackage{bbm}
\usepackage{amsfonts}
\usepackage{scalerel}

\usepackage{thmtools, thm-restate}
\usepackage{mathtools}
\usepackage[title]{appendix}
\usepackage{amsmath}
\DeclareMathOperator*{\argmax}{arg\,max}
\DeclareMathOperator*{\argmin}{arg\,min}

\title{Sharp Statistical Thresholds for Attributed Gaussian Graph Matching}
\author{Mohammad Hassan Ahmad Yarandi\textsuperscript{1}, Luca Ganassali\textsuperscript{2}}

\usepackage[round]{natbib}
\begin{document}

\maketitle

\begin{center}
\begin{tabular}{c}
\textsuperscript{1} Technical University of Munich
\end{tabular}

\texttt{mh.ahmadyarandi@tum.de}
\end{center}

\begin{center}
\begin{tabular}{c}
\textsuperscript{2} Université Paris-Saclay, Laboratoire de Mathématiques d’Orsay
\end{tabular}

\texttt{luca.ganassali@universite-paris-saclay.fr}
\end{center}

\begin{abstract}
We investigate the problem of matching networks where both edge weights and node features are correlated following a specific node correspondence. While prior work has established sharp thresholds for matching either graphs or features individually, the joint setting has been significantly less explored. 
In this study, we derive precise information-theoretic thresholds for exact recovery, almost-exact recovery, and partial recovery of the underlying alignment. These are expressed in terms of graph and feature correlation strengths, the number of nodes, and the dimension of the features. 
Interestingly, whereas an all-or-nothing phase transition is observed in the standard graph-matching scenario, the inclusion of contextual information at the node level introduces a richer structure: the exact-recovery threshold separates from the shared threshold for almost-exact and partial recovery. 
Closing gaps left by prior works, we provide a sharp characterization of the thresholds governing how structural and contextual information interact in graph matching, and thereby establish a benchmark for designing efficient algorithms in the high-dimensional setting. 
\end{abstract}

\section{Introduction and related work}
Database matching aims to recover node correspondences between two correlated, anonymized databases.
We first review the two main studied database matching problems, namely feature-based and graph matching, and then introduce a variant which combines the first two and serves as the main focus of this paper.

\subsection{Feature-based database matching} 
In feature-based database matching~\citep{cullina2018fundamental,dai2023fundamental,zeynep2024database}, a database consists of $n$ users, each represented by a feature vector in $\dR^d$.
Consider two databases $\mathbf{X}, \mathbf{Y} \in \dR^{n \times d}$ referring to the same underlying set of users, where pairs of feature vectors of the form $(\mathbf{X}_{i,\cdot}, \mathbf{Y}_{\pi^{\star}(i),\cdot})$ are associated with the same user $i$ and are correlated across databases. The one-to-one map $\pi^{\star}$ represents the correspondence between users, which is unknown.
The objective is to recover $\pi^{\star}$ based on the observation of $(\mathbf{X}, \mathbf{Y})$.

For example, a user’s features across two platforms---such as age, preferences, and ratings---are often correlated.
By leveraging the information available in one of the databases, one may be able to de-anonymize the second one~\citep{narayanan2008robust}.

One of the standard statistical models for this problem is the Gaussian model. In the Gaussian setting, all pairs of corresponding features across users $(\mathbf{X}_{i,k}, \mathbf{Y}_{\pi^{\star}(i),k})_{1 \leq i \leq n, 1 \leq k \leq d}$ are drawn i.i.d. from a joint Gaussian distribution with zero mean, unit variance, and correlation coefficient $\eta \in [0,1]$. Among the earliest works on feature-based database matching, \cite{cullina2018fundamental} established a sharp threshold for exact recovery when features belong to a finite alphabet. In the Gaussian setting, fundamental limits of exact and almost exact recovery~\citep{dai2019database,dai2023gaussian}, as well as correlation detection~\citep{zeynep2022detecting} have been studied. Notably, for exact recovery in the Gaussian setting, \cite{dai2019database} establish a sharp threshold for the correlation coefficient $\eta$ at $\eta_c$, defined by $\log\left(\frac{1}{1-\eta_c^2}\right) = 4\frac{\log(n)}{d}$, below which exact recovery of $\pi^{\star}$ is impossible and above which it is achievable. In the limit as $\eta$ tends to $0$, the threshold simplifies to $\eta_c = 2\sqrt{\frac{\log(n)}{d}}$.

\subsection{Graph matching}
Graph matching (also known as graph alignment) is a well-known instance of database matching~\citep{fan2023spectral,ding2021efficient,ganassali2022graph}.
Consider two graphs $\mG_1$ and $\mG_2$ with adjacency matrices $\bA$ and $\bB$, 
such that each pair $(\bA_{i,j}, \bB_{\pi^{\star}(i),\pi^{\star}(j)})_{1 \leq i<j \leq n}$ encodes  interactions between the same pair of entities across the two databases. 
Unlike the previous setting, where information was based on individual-level attributes (nodes), here the information is encoded in the interactions between users (edges). 
With the growing availability of graph-structured data, the challenge of recovering such correspondences has become increasingly significant, with applications in pattern recognition~\citep{berg2005shape}, bioinformatics~\citep{kazemi2016proper}, and social network analysis~\citep{korula2013efficient}. 

The two most widely studied statistical models for graph matching are the Gaussian model and the \ER model. In the Gaussian setting, edge pairs $(\bA_{i,j}, \bB_{\pi^{\star}(i),\pi^{\star}(j)})_{1 \leq i<j \leq n}$ are i.i.d. jointly Gaussian with standard $\cN(0,1)$ marginals and correlation coefficient $\rho \in [0,1]$. In the \ER setting, $\mG_1$ and $\mG_2$ are \ER graphs, with edge pairs following an i.i.d. correlated Bernoulli distribution specified by  $\dP\!\left(\bA_{i,j}=a,\, \bB_{\pi^{\star}(i),\pi^{\star}(j)}=b\right) = p_{ab}$ for $a,b \in \{0,1\},$  where the parameters $p_{ab}$ characterize the joint distribution of corresponding edges across the two graphs.

The information-theoretic limits of graph matching for correlated \ER graphs have been investigated for exact recovery~\citep{pedarsani2011privacy,cullina2017exact,WJY22}, almost exact and partial recovery~\citep{cullina2020partial,ganassali2021impossibility,hall2023partial}, as well as correlation detection~\citep{wu2023testing}. 
In the Gaussian setting, information-theoretic limits of graph matching have been established by~\cite{ganassali2022sharp,WJY22}. In particular,~\cite{WJY22} establishes a sharp threshold for the correlation coefficient $\rho$ at $\rho_c$ given by
\begin{equation*}
\rho_c = 2\sqrt{\frac{\log(n)}{n}},
\end{equation*} below which no algorithm can recover any positive fraction of the correct matches, and above which the maximum-a-posteriori estimator succeeds in exact recovery. Such a phenomenon is referred to as an all-or-nothing phase transition. 

Another thread of research focuses on the algorithmic aspects of the graph matching problem~\citep{GLM19,fan2023spectral,ding2021efficient,barak2019nearly,varma2025graph,mao2024testing}. We also mention recent works on multi-graph alignment~\citep{ameen2025detecting,vassaux2025feasibility,even2025}, which consider the graph matching problem for $m \geq 2$ correlated graphs. 

\subsection{Graph matching with feature information: attributed graph matching} 
In many real-world applications, both node features and graph structure are often available simultaneously, which motivates modeling graph matching with feature information
\citep{yang2024exact,yang2025exact}. Consider two databases $(\bX,\mG_1)$ and $(\bY,\mG_2)$ which contain information about the same set of users through the feature matrices $\bX,\bY\in\dR^{n\times d}$ and the interaction graphs $\mG_1,\mG_2$ with associated adjacency matrices $\bA,\bB$. 
Pairs of feature vectors $(\mathbf{X}_{i,\cdot}, \mathbf{Y}_{\pi^{\star}(i),\cdot})_{1 \leq i \leq n}$ in $\dR^d$ (resp. edge pairs $(\bA_{i,j}, \bB_{\pi^{\star}(i),\pi^{\star}(j)})_{1 \leq i<j \leq n}$) are associated with the same user $i$ (resp. pair of users $(i,j)$).
Utilizing both sources of information can enhance the identification of the underlying alignment $\pi^\star$, both from computational and information-theoretic perspectives.

Similar to the previous matching problems, two statistical models can be considered for attributed graph matching. In the Gaussian model, the graphs $\mG_1$ and $\mG_2$ are assumed to be correlated Gaussian graphs with correlation parameter $\rho$, while the feature matrices $\bX$ and $\bY$ consist of correlated Gaussian features with correlation parameter $\eta$. In the \ER model with Gaussian features, the graphs are assumed to be correlated \ER graphs with parameters $p_{ab}$ for $a,b\in\{0,1\}$, and the features are drawn from a joint Gaussian distribution with correlation coefficient $\eta$.

Although graph matching and feature-based matching problems have been extensively studied in the literature, the joint modeling of these two sources of information has received significantly less attention. 
Recently, \cite{yang2024exact} characterized the information-theoretic limits of exact recovery in correlated \ER graphs with correlated Gaussian features, establishing a sharp threshold for exact recovery of the form
\[
np_{11}+\frac{d}{4}\log\left(\frac{1}{1-\eta^2}\right) = \log(n),
\]
under additional technical conditions. In~\cite{yang2025exact}, the setting is extended to
correlated stochastic block models with two communities. 
Gaussian graphs with correlated Gaussian features have remained unexplored, and form the focus of our investigation in this paper.\\

During the completion of this paper, the authors were made aware of a recent and independent work
conducted by \cite{huang2026attributed} which studies the same problem, giving only information-theoretic upper and lower bounds with different proof techniques, and proposing an algorithm.
We want to emphasize here that our converse results both for exact and partial recovery are more precise than those developed in this former work. More precisely, the authors prove that for any estimator $\widehat{\pi}$, small constant $c>0$ and $\delta\in(0,1)$,
\begin{align*}
n\log\left(\frac{1}{1-\rho^2}\right)+2d\log\left(\frac{1}{1-\eta^2}\right)\leq c\log n \implies    \mathbb{P}\left(\textrm{overlap}(\widehat{\pi},\pi^\star)<\delta\right)\geq 1-\frac{c}{4\delta} \, .
\end{align*}
With the above result, if one wants to guarantee that partial recovery is impossible with high probability, this would require
\begin{equation}\label{eq:previous-impossibility_criterion}
n\log\left(\frac{1}{1-\rho^2}\right)+2d\log\left(\frac{1}{1-\eta^2}\right)=o(\log n) \, .
\end{equation}
To compare this with our impossibility criterion in this paper, note that when \eqref{eq:previous-impossibility_criterion} holds, \(\rho=o(1)\), and since we assume \(d=\omega(\log n)\), we also have \(\eta=o(1)\). Via a Taylor expansion, the previous impossibility criterion of~\cite{huang2026attributed} can be rewritten as
\[
\frac{\rho^2 n}{1-\rho^2}+\frac{2\eta^2 d}{1-\eta^2}=o(\log n),
\]
which is order of magnitude weaker than our sharp threshold obtained in \Cref{thm:impossibility_partial_recovery}.

For the exact recovery part, the same authors prove that if $\eta^2\geq\frac{40}{d}$ and $n\log\left(\frac{1}{1-\rho^2}\right)+d\log\left(\frac{1}{1-\eta^2}\right)+4\log d\leq (4-\varepsilon)\log n$ then exact recovery is impossible with high probability. If $d=n^{\alpha}$ for some $\alpha >0$, their lower bound writes $\frac{\rho^2n}{1-\rho^2}+\frac{\eta^2d}{1-\eta^2}\leq (4-4\alpha-\varepsilon)\log n$ which is not optimal compared to our sharp threshold in \Cref{thm:exact_recovery}---in particular, $(4-4\alpha-\eps)\log n$ can be negative if $\alpha \geq 1-\varepsilon/4$.
 
Our improvement comes from different techniques to prove the converse results. For partial recovery, \cite{huang2026attributed} use a packing argument combined with Fano’s inequality, which reduces approximate recovery to decoding a global packing and controls the error only through total mutual information. This approach ignores the local geometry of the permutation space and is therefore unable to capture the sharp all-or-nothing behavior. In contrast, we work directly with the Bayes-optimal estimator for the overlap loss and control the posterior mass of Hamming balls by comparing their unnormalized mass with the posterior partition function. For exact recovery, although both proofs use a second-moment argument over transpositions, \cite{huang2026attributed} lower bound the probability of a bad transposition by treating the graph and feature fluctuations separately and invoking a dimension-dependent feature-tail bound. We instead condition on one side of the observations and analyze the combined graph–feature likelihood gap as a single Gaussian variable with random variance, together with sharp Gaussian tail estimates for correlated transpositions. This joint analysis avoids the dimension-dependent loss appearing in their argument and leads to our sharper converse results.

\label{sec:intro}

\section{Problem formulation and main results}\label{sec:problem_main_results}
\subsection{Notations}
We denote by $ \mS_n$ the set of all permutations of $ [n]=\{1,\ldots,n\} $. For brevity, we denote by $\pi(i,j) := \{\pi(i),\pi(j)\}$ the non-oriented edge between $\pi(i)$ and $\pi(j)$ for some $\pi \in \mS_n$. 
For each permutation $ \pi \in \mS_n $, $ D_{\pi} $ and $ F_{\pi} $ (respectively, $ D^{E}_{\pi} $ and $ F^{E}_{\pi} $) denote the sets of unfixed and fixed points of permutation $\pi$ acting on the nodes (respectively, on the edges), that is
\begin{align*}
D_\pi & \vcentcolon=\left\{i\in[n]:\pi(i)\neq i\right\}, & F_\pi & \vcentcolon=[n] \setminus D_\pi,\\
D^{E}_{\pi} & \vcentcolon=\big\{\{i,j\}\in E: \pi(i,j)\neq \{i,j\}\big\}, & F^{E}_{\pi} &  \vcentcolon=E\setminus D^{E}_{\pi} \, .
\end{align*} We partition $ \mS_n $ as
\(\mS_n = \{\mathrm{id}\} \cup \bigcup_{t=2}^{n} \mS_{n,t},
\) where $\mS_{n,t}$ is the set of permutations of $\mS_n$ that differ from $\mathrm{id}$ by exactly $t$ unfixed points, that is $\mS_{n,t}:= \{\pi \in \mS_n, |D_\pi|=t\}$.
For $\pi,\pi' \in \cS_n$, define their overlap:
$$\ov(\pi,\pi') :=\frac{1}{n} |\{ i \in [n], \pi(i)=\pi'(i) \}| = \frac{1}{n}|F_{\pi^{-1} \circ \pi'}| = 1 - \frac{1}{n}|D_{\pi^{-1} \circ \pi'}|\, .$$
We say that a sequence of events $\mathcal{E}_n$ happens \emph{with high probability (w.h.p.)} if 
$\dP(\mathcal{E}_n) \underset{n\to\infty}{\longrightarrow} 1$.

\noindent A sequence of estimators $(\widehat{\pi}_n)_{n \geq 1}$ of $(\pi^{\star}_n)_{n \geq 1}$ where for all $n \geq 1$, $\widehat{\pi}_n,\pi^{\star}_n \in \mS_n$, is said to achieve
\begin{itemize}
    \item \emph{exact recovery} if $\widehat{\pi}_n=\pi^{\star}_n$ w.h.p.;
    \item \emph{almost exact recovery} if for all $\delta \in (0,1)$, $\ov(\widehat{\pi}_n,\pi^{\star}_n)>\delta$ w.h.p.;
    \item \emph{partial recovery} if there exists  $\delta \in (0,1)$ such that $\ov(\widehat{\pi}_n,\pi^{\star}_n)>\delta$ w.h.p.
\end{itemize}
We use the term \emph{impossibility} in the following strong sense. 
For each $n$, the supremum below is taken over all measurable estimators
$\widehat{\pi}_n=\widehat{\pi}_n(D_n)$ taking values in $\mS_n$.
We say that
\begin{itemize}
    \item exact recovery is impossible if:
    \(
        \sup_{\widehat{\pi}_n}
        \mathbb{P}\left(\widehat{\pi}_n=\pi_n^\star\right)
        =o(1);
    \)

    \item almost-exact recovery is impossible if $\exists\delta\in(0,1)$: $\sup_{\widehat{\pi}_n}\mathbb{P}\left(\operatorname{overlap}(\widehat{\pi}_n,\pi_n^\star)>\delta\right)=o(1)$;
    \item partial recovery is impossible if, for every
    $\delta\in(0,1)$:
    \(
        \sup_{\widehat{\pi}_n}
        \mathbb{P}\left(
            \operatorname{overlap}
            (\widehat{\pi}_n,\pi_n^\star)>\delta
        \right)
        =o(1).
    \)
\end{itemize}
Throughout, we may omit the dependence of $\widehat{\pi}_n$ and $\pi^\star_n$ on $n$.

\subsection{Problem formulation}

In our setting, the statistical model is defined as follows. We first sample two graphs $\mG_1$ and $\widetilde{\mG}_2$ with the same node set $[n]$ and weighted adjacency matrices $\bA$ and $\widetilde{\bB}$ where the edge weight pairs $\{(\bA_{i,j}, \widetilde{\bB}_{i,j}), 1 \leq i < j \leq n\}$ are i.i.d. standard Gaussian random variables with correlation $\rho\in[0,1]$. Then, $\mG_2$ is obtained by permuting the vertices of $\widetilde{\mG}_2$ according to a uniform random permutation $\pi^{\star}$ in $\cS_n$. That is, conditioned on $\pi^\star$, the joint distribution of edge weights is given by
\begin{equation}
\label{eq:model_AB}
\left(\bA_{i,j}, \bB_{\pi^{\star}(i), \pi^{\star}(j)}\right) \sim \mathcal{N}\left(\begin{bmatrix}
0 \\ 0
\end{bmatrix}, \begin{bmatrix}
1 & \rho \\ 
\rho & 1
\end{bmatrix}\right).
\end{equation}

For the feature information, we draw two matrices $\mathbf{X}, \widetilde{\mathbf{Y}} \in \dR^{n \times d}$ such that each pair of corresponding entries $\{(\bX_{i,k}, \widetilde{\bY}_{i,k}), 1 \leq i \leq n, 1 \leq k \leq d\}$ are also i.i.d. standard Gaussian random variables with correlation $\eta\in[0,1]$. The matrix $\mathbf{Y}$ is then formed by permuting the rows of $\widetilde{\mathbf{Y}}$ using the \emph{same} permutation $\pi^{\star}$ used for the graph structure. In short, conditioned on $\pi^\star$, the feature-level data is modeled as
\begin{equation}
\label{eq:model_XY}
\left(\mathbf{X}_{i,k}, \mathbf{Y}_{\pi^{\star}(i),k}\right) \sim \mathcal{N}\left(\begin{bmatrix}
0 \\ 0
\end{bmatrix}, \begin{bmatrix}
1 & \eta \\ 
\eta & 1
\end{bmatrix}\right).
\end{equation}
Using the permutation matrix $\Pi^{\star}$ corresponding to $\pi^{\star}$, this probabilistic model can also be expressed as:
\begin{equation}
\label{eq:def_model}
\begin{aligned}
\bB &= \rho(\Pi^{\star})^{\top}\bA\Pi^{\star} + \sqrt{1-\rho^2}\mathbf{Z}, \\  
\mathbf{Y} &= \eta(\Pi^{\star})^{\top}\mathbf{X} + \sqrt{1-\eta^2}\mathbf{Z}',
\end{aligned}
\end{equation}
where $\mathbf{Z}$ is a symmetric matrix with independent standard Gaussian entries above the diagonal, and $\mathbf{Z}'$ has independent standard Gaussian entries.
In the above statistical model \eqref{eq:def_model}, the goal is to infer the underlying permutation $\pi^{\star}$ using an estimator $\widehat{\pi}$ based on the observed data: $\widehat{\pi}= \widehat{\pi}(\bA, \bB,\mathbf{X}, \mathbf{Y})$.

\subsection{Main results}
\label{MR}
We now present our main results: sharp statistical thresholds for attributed graph matching in the Gaussian setting. We begin with exact recovery.
\begin{theorem}[\textbf{Exact Recovery}]\label{thm:exact_recovery} \,
\begin{itemize}
\item[$(i)$] (Achievability Result) If $d = \omega(\log n)$ and for sufficiently large $n$, the following condition holds:
\begin{equation*}
\frac{\rho^2n}{1-\rho^2} + \frac{\eta^2d}{1-\eta^2} \geq 4(1 + \eps) \log n, \end{equation*} for some $\eps > 0$, then there exists an estimator $\widehat{\pi}:(\bA,\bB,\bX,\bY) \to \mathcal{S}_{n}$ such that
\begin{equation*}
\dP\left(\widehat{\pi} = \pi^{\star}\right) = 1 - o(1) \, .
\end{equation*} 

\item[$(ii)$] (Converse Result) If $d = \omega(\log n)$ and
\begin{equation*}
\frac{\rho^2n}{1-\rho^2} + \frac{\eta^2d}{1-\eta^2} \leq 4\log n - \log\log n - \omega(1),
\end{equation*}
then, for any estimator $\widehat{\pi}: (\bA, \bB, \bX, \bY) \to \mS_n$,
\begin{equation*}
\dP(\widehat{\pi} = \pi^{\star}) = o(1) \, .
\end{equation*} \label{thc}
\end{itemize}
\end{theorem} We also provide a sufficient condition for achievability of almost exact recovery. 
\begin{theorem}[\textbf{Achievability of almost exact recovery}]\label{thm:almost_exact_recovery}
If $d = \omega(\log n)$ and for sufficiently large $n$, the following condition holds:
\begin{equation*}
\frac{\rho^2n}{1-\rho^2} + \frac{2\eta^2d}{1-\eta^2} \geq 4(1 + \eps) \log n,
\end{equation*} for some $\eps > 0$, then there exists an estimator $\widehat{\pi}:(\bA,\bB,\bX,\bY) \to \mathcal{S}_{n}$ such that for all $\delta\in (0,1)$,
\begin{equation*}
\dP\left(\ov(\widehat{\pi}, \pi^{\star})>\delta\right) = 1 - o(1) \, .
\end{equation*}
\end{theorem}
Finally, we derive an information-theoretic impossibility result for partial recovery when the signal is weaker than in \Cref{thm:almost_exact_recovery}, and further show that no fixed fraction of the nodes can be correctly matched under this bound. 

\begin{theorem}[\textbf{Impossibility of partial recovery}]\label{thm:impossibility_partial_recovery}
If $d = \omega\left(\log n\right)$ and
\begin{equation*}
\frac{\rho^2n}{1-\rho^2} + \frac{2\eta^2d}{1-\eta^2} \leq 4(1 - \eps) \log n,
\end{equation*}
then, for any estimator $\widehat{\pi}: (\bA, \bB, \bX, \bY) \to \mS_n$ and for all $\delta\in(0,1)$,
\begin{equation*}
\dP\left(\ov(\widehat{\pi}, \pi^{\star})>\delta\right) = o(1).
\end{equation*} 
\end{theorem}

\Cref{fig:phase_diagram} gives the phase diagram which illustrates and gathers our results.

\begin{remark}
Taking $\rho=0$, our results also characterize the statistical limits of feature-based database matching. In particular, \Cref{thm:almost_exact_recovery,thm:impossibility_partial_recovery} give the sharp impossibility result for partial recovery in feature-based database matching, addressing the conjecture made by~\cite{kunisky2022strong} (Conjecture 1.4, Part 3). This conjecture claims that if $d=\omega(\log n)$ and $\frac{\eta^2d}{1-\eta^2}\leq4(1-\varepsilon)\log n$, then $\textrm{overlap}(\widehat{\pi}_{\mathrm{MAP}},\pi^\star)\leq 1-c$ for some $c=c(\varepsilon)>0$ with high probability (i.e., the MAP estimator makes at least $cn$ errors), where $\widehat{\pi}_{\mathrm{MAP}}$ is the maximum-a-posteriori (MAP) estimator. 

To discuss this conjecture, two regimes should be considered. First, when
\[
2(1+\varepsilon')\log n
\leq
\frac{\eta^2 d}{1-\eta^2}
\leq
4(1-\varepsilon)\log n,
\]
\cite{dai2020achievability,dai2023gaussian} disproved the conjecture by showing that almost exact recovery is possible in this regime using the MAP estimator. What we prove in \Cref{thm:almost_exact_recovery} is a related result: instead of considering the MAP estimator, we show that the Bayes-optimal estimator achieves almost exact recovery in this regime.

Second, in the regime where
\[
\frac{\eta^2 d}{1-\eta^2}
\leq
2(1-\varepsilon)\log n,
\]
the conjecture was, at the time this work was developed, open. We have since become aware of the contemporaneous preprint by \cite{wee2026geometric}, which proves an equivalent result for geometric planted matching under an equivalent Gaussian noise parameterization. 
These results show that, for any constant $c\in(0,1)$, the overlap between $\widehat{\pi}$ and $\pi^\star$ is less than $c$ with high probability, for any estimator $\widehat{\pi}$ -- note that in \Cref{thm:impossibility_partial_recovery}, we prove the failure of the Bayes-optimal estimator, which is \emph{not} the MAP estimator. Since no estimator can outperform the Bayes-optimal estimator, this implies the failure of every estimator, including the MAP estimator.
\end{remark}

\begin{figure}[t]
\centering
\begin{tikzpicture}[scale=1]
\foreach \x in {1,2,4}
  \draw (\x,0) -- (\x,-0.08) node[below] {$\x$};
\foreach \y in {1,2,4}
  \draw (0,\y) -- (-0.08,\y) node[left] {$\y$};


\fill[red!40] (0,0) -- (4,0) -- (0,2) -- cycle;
\fill[pattern={Lines[angle=45,distance=5pt,line width=0.5pt]},pattern color=black!70]
  (0,0) -- (4,0) -- (0,2) -- cycle;

\fill[blue!40] (0,2) -- (4,0) -- (4,4) -- (0,4) -- cycle;
\fill[pattern=dots, pattern color=black!70]
(0,2) -- (4,0) -- (4,4) -- (0,4) -- cycle;

\fill[green!70]
  (0,4.5) -- (0,4) -- (4,0) -- (4.5,0) -- (4.5,4.5) -- cycle;
\fill[pattern=horizontal lines,pattern color=black!50]
  (0,4.5) -- (0,4) -- (4,0) -- (4.5,0) -- (4.5,4.5) -- cycle;

\draw[thick, blue] (0,4) -- (4,0);              
\draw[thick, red!60!black] (0,2) -- (4,0);    

\node[below] at (2.4,-0.55) {\textbf{Edge signal}};
\node[rotate=90] at (-0.8,2.35) {\textbf{Feature signal}};
\draw[->,line width=1.2pt] (0,0) -- (4.7,0) node[right] {$\frac{\rho^2 n}{\log n}$};
\draw[->,line width=1.2pt] (0,0) -- (0,4.7) node[above] {$\frac{\eta^2 d}{\log n}$};

\begin{scope}[shift={(5.7,2.4)}]
\fill[green!70,draw=black] (0,0) rectangle +(0.5,0.5);
\fill[pattern=horizontal lines,pattern color=black!70] (0,0) rectangle +(0.5,0.5);
\node[right] at (0.6,0.2) {\footnotesize Exact recovery feasible};

\draw[fill=blue!40, draw=black] (0,-0.6) rectangle +(0.5,0.5);
\fill[pattern=dots, pattern color=black!70] (0,-0.6) rectangle +(0.5,0.5);
\node[right] at (0.6,-0.4) {\footnotesize Almost exact recovery feasible};

\fill[red!40,draw=black] (0,-1.2) rectangle +(0.5,0.5);
\fill[pattern={Lines[angle=45,distance=5pt,line width=0.5pt]},pattern color=black!70] (0,-1.2) rectangle +(0.5,0.5);
\node[right] at (0.6,-1) {\footnotesize Partial recovery impossible};
\end{scope}
\end{tikzpicture}
\caption{Recovery phase diagram for attributed Gaussian graph matching in the regime $\rho,\eta=o(1)$.}
\label{fig:phase_diagram}
\end{figure}

The novelty of our results lies primarily in the attributed graph matching setting rather than in the feature-only setting. Our theorems establish the sharp information-theoretic limits for this problem, improving upon the previous results of \cite{huang2026attributed}.

\subsection{Discussion}
The phase diagram of \Cref{fig:phase_diagram} highlights the complementary roles of edge and feature information. To see this, write
\(
x:=\frac{1}{\log n}\frac{\rho^2 n}{1-\rho^2},
y:=\frac{1}{\log n}\frac{\eta^2 d}{1-\eta^2}.
\)
In the rectangle $x\leq 4$ and $y\leq 2$, neither the graph signal alone nor the feature signal alone is sufficient for partial recovery. Nevertheless, there are points in this region where recovery becomes possible once the two sources of information are combined: almost-exact recovery is feasible whenever $x+2y>4$, and exact recovery is feasible in the stronger region $x+y>4$. Thus, feature information can compensate for insufficient graph information, and conversely graph information can compensate for insufficient feature information. At the same time, the phase diagram also indicates when one source of information is statistically negligible. In the regime $\rho,\eta=o(1)$, if $\eta^2 d \ll \rho^2 n$, then the feature signal does not affect the recovery threshold and is non-informative for recovery; conversely, if $\rho^2 n \ll \eta^2 d$, then the graph signal is negligible compared to the feature signal. This provides a quantitative criterion for when one of the two data sources can be ignored without changing the statistical threshold. 

\paragraph{No all-or-nothing phase transition occurs.} As highlighted in the phase diagram of \Cref{fig:phase_diagram}, our results show that the thresholds for exact and almost exact recovery do not coincide. Indeed, there exists a nontrivial region (the blue region in \Cref{fig:phase_diagram}) where exact recovery is information-theoretically impossible, while almost exact recovery remains achievable with high probability.

This phenomenon sharply contrasts with the all-or-nothing phase transition observed in the pure graph matching setting~\citep{WJY22}, where either exact recovery is possible or partial recovery is impossible, with no intermediate regime (note that this is consistent with the x-axis of the phase diagram).
This means that, interestingly, the presence of feature information on the users fundamentally alters the recovery landscape, leading to a richer phase diagram (see the double point at $(x,y)=(4,0)$ in \Cref{fig:phase_diagram}).

\paragraph{On the $d = \omega(\log n)$ assumption.}
The condition $d=\omega(\log n)$ appearing in our results is common in the literature of database matching \citep{dai2019database,dai2023gaussian}, and most of the concentration arguments used in the proofs rely heavily on this high-dimensional assumption. 

We believe that different criteria for cases $d=\Theta(\log n)$ and $d=o(\log n)$ could be obtained, using different proof techniques -- as done in \cite{kunisky2022strong}.

Note, however, that in the feature-only regime, as studied by \cite{kunisky2022strong}, increasing the feature dimension $d$ can only facilitate the matching problem by amplifying the available signal. When $d = O(\log n)$ and $\eta =o(1)$, we have $\eta^2 d = o(\log n)$. In this regime, for attributed graph alignment, our results show that the contribution of feature information is negligible at the level of (partial, almost exact and exact) recovery thresholds, and the problem essentially reduces to the graph-only setting.
This suggests that, in the low correlation regime, the role of features becomes information-theoretically relevant only when $d$ grows faster than $\log n$, making our assumptions natural.

\subsection{Paper organization}
In \Cref{sec:optimal-est}, we introduce the Bayes-optimal estimators for each of the three recovery tasks and establish some of their concentration properties. These properties are then useful in \Cref{sec:exact} where we give the main ideas of the proof of the sharp exact recovery threshold (\Cref{thm:exact_recovery}). Finally, in \Cref{sec:almost-exact-partial}, we provide the proof of our almost exact and partial recovery results (\Cref{thm:almost_exact_recovery,thm:impossibility_partial_recovery}). Additional proofs are deferred to Appendices.

\section{Bayes-optimal estimators}\label{sec:optimal-est}
As previously done by \cite{vassaux2025feasibility}, 
we look at the problem through the lens of Bayesian inference where our Bayes-optimal estimators are minimizers of some expected loss.
To this end, we define our loss functions as follows, for $\pi, \pi' \in \mS_n$ and $r \in [0,1)$:
\begin{align*}
d(\pi,\pi') :=1-\ov(\pi,\pi') \qquad \mbox{and} \qquad \ell_r(\pi,\pi') :=\mathbf{1}_{\{d(\pi,\pi')>r\}} \, .
\end{align*}

The expected loss of an estimator $\widehat{\pi}= \widehat{\pi}(D)$ where $D$ is the observed data $D=(\bA,\bB,\bX,\bY)$ is defined by
\begin{align*}
L_r(\widehat{\pi}) :=\dE\left[\ell_r(\widehat{\pi}(D),\pi^{\star})\right]\stackrel{(a)}{=} \dP\left(d(\widehat{\pi}(D),\pi^{\star})>r\right).
\end{align*}
In the above definition, the expectation is taken with respect to $D=(\bA,\bB,\bX,\bY)$ and $\pi^\star$.
Equality $(a)$ shows the concrete correspondence between the Bayesian point of view and our initial problem definition from \Cref{sec:problem_main_results}.
More precisely, $r=0$ gives $L_0(\widehat{\pi})=\dP\left(\widehat{\pi} \neq \pi^{\star}\right)$, which is the exact recovery criterion and if $r\in(0,1)$ then $L_r(\widehat{\pi})$ for all (resp. some) $r \in (0,1)$ is the criterion for almost exact (resp. partial) recovery. 

Next, we denote by $\dP_{\textrm{post}}:=\dP_{\pi^{\star}|D}$ the posterior distribution of $\pi^{\star}$ conditioned on data $D=(\bA,\bB,\bX,\bY)$.
Based on this Bayesian expression of the recovery problems, the Bayes-optimal estimator (for partial/almost exact/exact recovery) is the one which minimizes $L_r(\widehat{\pi})$ for the corresponding $r$. Note that
\begin{align}\label{eq:opt-est}
\min_{\widehat{\pi}} L_r(\widehat{\pi})&=\min_{\widehat{\pi}} \dE_{D}\left[\dE_{\textrm{post}}\left[\ell_r(\widehat{\pi}(D),\pi^{\star})\right]\right] \geq  \dE_{D}\left[\min_{\pi\in\mS_n}\dE_{\textrm{post}}\left[\ell_r(\pi,\pi^{\star})\right]\right].
\end{align} 
Since
\[
\widehat{\pi}_{\textrm{opt},r}(D) :=  \argmin_{\pi\in\mS_n}\dE_{\textrm{post}}\left[\ell_r(\pi,\pi^{\star})\right]
\]is a valid estimator (it is indeed a measurable function of $D$ since $D$ determines the posterior $\dP_{\textrm{post}}$), \eqref{eq:opt-est} implies that $\widehat{\pi}_{\textrm{opt},r}(D)$ is Bayes-optimal for a given $r$. 
Based on the distance metric defined earlier, one can rewrite the optimal estimator for $\ell_r$ as
\begin{align*}
\widehat{\pi}_{\textrm{opt},r}=\argmin_{\pi\in\mS_n} \dP_{\textrm{post}}\left(B(\pi,r)^c\right)=\argmax_{\pi\in\mS_n} \dP_{\textrm{post}}\left(B(\pi,r)\right),
\end{align*}  
where $B(\pi,r)^c$ is the complement of $B(\pi,r)$, the closed ball of radius $r$ at $\pi$ for metric $d$ in $\mS_n$. Furthermore, the optimal expected loss is
\begin{align*}
L_{\textrm{opt},r}:= L_r(\widehat{\pi}_{\textrm{opt},r}) = 1-\dE_{D}\left[\max_{\pi\in\mS_n} \dP_{\textrm{post}}\left(B(\pi,r)\right)\right].
\end{align*}
The optimal estimator and optimal expected loss directly yield alternative expressions for the recovery criteria, which we gather in the following proposition. 
The following Bayesian criterion follows the posterior characterization developed by~\cite{vassaux2025feasibility}. We state it in our notation for completeness, since it will be used repeatedly to translate recovery guarantees into posterior-mass statements.
\begin{proposition}\label{prop:crit_bayes}
\begin{enumerate}[(i)]
    \item Exact recovery is possible iff $\widehat{\pi}_{\textrm{opt},0}$, namely the MAP estimator, equals $\pi^\star$ with probability at least $1-o(1)$ and impossible if $\pi^{\star}\notin\argmax_{\pi\in\mS_n} \dP_{\textrm{post}}\left(\pi\right)$ with probability at least $1-o(1)$.       
    \item Almost exact recovery is possible iff $\forall r\in(0,1), \ \max_{\pi\in\mS_n} \dP_{\textrm{post}}\left(B(\pi,r)\right)\stackrel{\dP_D}{\longrightarrow}1$ and impossible iff $\exists r\in(0,1), \ \max_{\pi\in\mS_n} \dP_{\textrm{post}}\left(B(\pi,r)\right)\stackrel{\dP_D}{\longrightarrow}0$.
    \item Partial recovery is possible iff $\exists r\in(0,1): \ \max_{\pi\in\mS_n} \dP_{\textrm{post}}\left(B(\pi,r)\right)\stackrel{\dP_D}{\longrightarrow}1$ and impossible iff $\forall r\in(0,1): \ \max_{\pi\in\mS_n} \dP_{\textrm{post}}\left(B(\pi,r)\right)\stackrel{\dP_D}{\longrightarrow}0$.
\end{enumerate}
\end{proposition} 
These are the ultimate equivalent recovery criteria that we investigate to prove our main theorems. As suggested by \Cref{prop:crit_bayes}, we need an explicit expression of the posterior distribution $\dP_{\mathrm{post}}$.
Denoting by $p(\cdot)$ (resp. $p(\cdot \, | \, \pi)$) the joint density of $D=(\bA,\bB,\bX,\bY)$ under $\dP$ (resp. conditioned on $\pi^\star=\pi$), we have, for all $\pi \in \mS_n$,
\begin{align*}
\dP_{\mathrm{post}}\left(\pi|\bA,\bB,\bX,\bY\right)& =\frac{p\left(\bA,\bB,\bX,\bY|\pi\right)\dP\left(\pi^\star=\pi\right)}{p\left(\bA,\bB,\bX,\bY\right)} 
\end{align*}
Since \(\left(\bA,\bB\right)\) and \(\left(\bX,\bY\right)\) are independent given \(\pi\), we can write, based on \eqref{eq:model_AB} and \eqref{eq:model_XY},
\begin{multline*}
\dP_{\mathrm{post}}\left(\pi|\bA,\bB,\bX,\bY\right) \\
= C \exp\bigg[-\frac{1}{2(1-\rho^2)} \sum\limits_{1\leq i<j\leq n} \left( B_{\pi(i,j)} - \rho A_{i,j} \right)^2  -\frac{1}{2(1-\eta^2)} \sum\limits_{\substack{1\leq i \leq n\\1\leq k \leq d}} \left( Y_{\pi(i),k} - \eta X_{i,k} \right)^2\bigg],
\end{multline*} where \(C\) is a constant independent of permutation \(\pi\).
Without loss of generality, we can assume that
$\pi^{*} = \mathrm{id}$.
Expanding the squares in the last expression, and writing $\dP_{\mathrm{post}}\left(\pi|\bA,\bB,\bX,\bY\right) = \dP_{\mathrm{post}}(\pi)$ for brevity, we get
\begin{equation}\label{eq:Ppost}
\dP_{\textrm{post}}(\pi)=\frac{1}{Z}\exp\bigg(-\frac{\rho}{1-\rho^2}\left(V^*_{G}(\pi)-V_{G}(\pi)\right) -\frac{\eta}{1-\eta^2}\left(V^*_{F}(\pi)-V_{F}(\pi)\right)\bigg),
\end{equation}
where $Z$ denotes the normalizing constant, and
\begin{align*} V^*_{G}(\pi) =\sum\limits_{\substack{1\leq i<j \leq n\\ \pi(i,j)\neq \{i,j\}}}B_{i,j}A_{i,j}, \quad V_{G}(\pi)=\sum\limits_{\substack{1\leq i<j \leq n\\ \pi(i,j)\neq\{i,j\}}}B_{\pi(i,j)}A_{i,j}, \\
V^*_{F}(\pi) =\sum\limits_{\substack{1\leq i \leq n, \ 1\leq k\leq d\\ \pi(i)\neq i}}Y_{i,k}X_{i,k}, \quad V_{F}(\pi) =\sum\limits_{\substack{1\leq i \leq n, \ 1\leq k\leq d\\ \pi(i)\neq i}}Y_{\pi(i),k}X_{i,k} \, .
\end{align*}
It is worth noting that this probability measure is a Gibbs distribution corresponding to the following Hamiltonian:
\begin{equation}
\label{eq:hamiltonian}
V(\pi):=\frac{\rho}{1-\rho^2}\left(V^*_{G}(\pi)-V_{G}(\pi)\right)+\frac{\eta}{1-\eta^2}\left(V^*_{F}(\pi)-V_{F}(\pi)\right).
\end{equation}
For later use, we define the unnormalized posterior masses
\begin{align}
Z_r(\pi)&:=\sum_{\pi'\in B(\pi,r)}e^{-V(\pi')}, &
\overline Z_r(\pi)&:=\sum_{\pi'\notin B(\pi,r)}e^{-V(\pi')},
\end{align}
so that $\dP_{\textrm{post}}(B(\pi,r))=Z_r(\pi)/Z$ and $\dP_{\textrm{post}}(B(\pi,r)^c)=\overline Z_r(\pi)/Z$.

As stated in Proposition \ref{prop:crit_bayes}, to prove the possibility (resp. impossibility) of almost exact and partial recovery results, we need to 
upper bound the posterior distribution over outside (resp. inside) given balls centered at arbitrary permutations with high probability. 
To do so, we derive separate bounds for the normalizing constant $Z$ and the exponential components in \eqref{eq:Ppost}, following the spirit of \cite{vassaux2025feasibility}.
These bounds are proposed in the next three subsections.

\subsection{Concentration bound for $V^*_{G}$ and $V^*_{F}$}
We begin with the exponential components $V^*_{G}$ and $V^*_{F}$ and show that they are highly concentrated around their expectations with high probability. The following result is proved in Appendix \ref{app:proof:lem:main1}.

\begin{lemma}
\label{lem:main1}
Consider the event $\mathcal{H}_1$ defined by
\begin{multline}
\mathcal{H}_1 := 
\Bigg\{\forall t\in[n], \forall \pi\in\mS_{n,t}: \left|V^*_{G}(\pi)-\dE\left[V^*_{G}(\pi)\right]\right| \leq Ct\sqrt{n\log\left(\tfrac{en}{t}\right)}, \\  \left|V^*_{F}(\pi)-\dE\left[V^*_{F}(\pi)\right]\right| \leq Ct\sqrt{\max\{d,\log\left(\tfrac{en}{t}\right)\}\log\left(\tfrac{en}{t}\right)}\Bigg\},\label{H1}
\end{multline}
 where $C>0$ is a universal constant. Then, $\mathcal{H}_1$ occurs with probability at least $1-o(1)$.
\end{lemma}
\begin{remark}
\cite{vassaux2025feasibility} establish a uniform concentration bound for $V^*_{G}(\pi)$ of order $O(n^{3/2}(\log n)^{1/4})$. Note that the bound in \Cref{lem:main1} is strictly tighter:
for instance, if $t=\Theta(1)$, our bound is $O\big(\sqrt{n\log n}\big)$, which is substantially smaller than $O\!\big(n^{3/2}(\log n)^{1/4}\big)$. The reason for our improvement is that in our proof, we leverage the fact that $V^*_{G}(\pi)$ and $V^*_{F}(\pi)$ depend on $\pi$ only through the set $D_\pi$ rather than on the specific mapping induced by $\pi$. Since a lot of permutations have same set of unfixed points, our union bound is tighter. 
\end{remark}

\subsection{Upper Bound for $V_F$ and $V_G$}
Now, we bound the remaining terms in the exponential component, namely $V_F$ and $V_G$. Unlike $V^*_F$ and $V^*_G$, which are highly concentrated around their means, the fluctuations of $V_F$ and $V_G$ are significantly larger. We adopt a different approach: instead of directly bounding $V_F(\pi)$ and $V_G(\pi)$, we control their Laplace transforms.
The proof of the following lemma is given in Appendix~\ref{app:proof:lem:main2}.

\begin{lemma}
\label{lem:main2}
There exists an event $\mathcal{H}_2$, depending only on $(\bB,\bY)$, which occurs with probability $1-o(1)$, such that for all $\pi \in \mS_n$ we have

\begin{align*}
\dE\left[e^{\frac{\rho}{1-\rho^2} V_G(\pi)}\mathbf{1}_{\mathcal{H}_2}\right]
& \leq \exp\left({\frac{\rho^2}{2(1-\rho^2)}\left(\left|D^E_{\pi}\right|+C\left|D_{\pi}\right|\sqrt{n\log(n)}\right)}\right),\\
\dE\left[e^{\frac{\eta}{1-\eta^2} V_F(\pi)}\mathbf{1}_{\mathcal{H}_2}\right]
&  \leq\exp\left({\frac{\eta^2}{2(1-\eta^2)}\left(d\left|D_{\pi}\right|+C\left|D_{\pi}\right|\sqrt{\max\left\{d,\log\left(n\right)\right\}\log\left(n\right)}\right)}\right),
\end{align*}
where $C>0$ is a universal constant. 
\end{lemma}



\begin{remark}
In~\cite{vassaux2025feasibility}, the authors bound $\dE\!\left[e^{\beta V_G(\pi)}\mathbf{1}_{\mathcal{H}_2}\right]$ for small $\beta$ using a series expansion and retain only the leading terms.  
In contrast, our approach derives an upper bound of a Laplace transform on a high-probability event; since the impossibility result will hold with high probability, conditioning on $\mathcal{H}_2$ does not weaken the argument.

An important advantage of our method is that it applies to all $\rho,\eta\in[0,1)$, whereas the approach of~\cite{vassaux2025feasibility} requires $\beta$ (and thus $\rho,\eta$) to be small. However, the regime where $\rho,\eta$ may tend to $1$ is particularly crucial in our problem---for instance, for the $V_F$ term, as when $d=o(\log n)$ the quantity $\frac{\eta}{1-\eta^2}$ may diverge.
\end{remark}

\subsection{Lower Bound for $Z$}
A trivial lower bound on $Z$ is given by
\(1=\exp\left(-V(\textrm{id})\right)\leq Z\).
In fact, this inequality is sufficient for the proof of possibility of almost exact recovery (\Cref{thm:almost_exact_recovery}), because as we will show, for every fixed $r\in(0,1)$ we have $\overline Z_r(\mathrm{id})\leq\exp\left(-\Omega(n\log n)\right)$ with high probability. Thus, $Z\geq 1$ is enough to ensure that $\max_{\pi\in\mS_n} \dP_{\textrm{post}}\left(B(\pi,r)\right)\stackrel{\dP_D}{\longrightarrow}1$.

However, the situation becomes more complicated for the proof of impossibility of partial recovery in \Cref{thm:impossibility_partial_recovery}. In the regime
\[
    \frac{\rho^{2}n}{1-\rho^{2}}
    +\frac{2\eta^{2}d}{1-\eta^{2}}
    \leq 4(1-\varepsilon)\log n,
\]
write
\begin{equation}
\label{eq:effective-epsilon}
    \frac{\rho^{2}n}{1-\rho^{2}}
    +\frac{2\eta^{2}d}{1-\eta^{2}}
    =4(1-\tilde\varepsilon)\log n,
\end{equation}
where
$\tilde\varepsilon
=\tilde\varepsilon(n,d,\rho,\eta)\geq\varepsilon$.
We will establish the high-probability lower bound
\[
    Z\geq
    \exp\!\left(
        \tilde\varepsilon n\log n(1-o(1))
    \right),
\]
which will be used to prove impossibility of partial recovery.
The proof proceeds in two steps. In the first step, we show that $\log(Z)$ is sharply concentrated around its conditional mean $\dE[\log(Z)\mid\bB,\bY]$. This is a simple consequence of the fact that, conditionally on $(\bB,\bY)$, $\log(Z)$ is a Lipschitz function of the Gaussian noise variables in our problem. We obtain that with high probability,
\begin{align*}
\log(Z)\geq\dE[\log(Z)\mid\bB,\bY]-o(n\log n) \, .
\end{align*}
In the second step, using a second moment method, we establish a lower bound for $\dE[\log(Z)\mid\bB,\bY]$ on the event $\cH_2$, as follows:
\[
\dE[\log(Z)\mid\bB,\bY] \geq \tilde{\eps} n \log n \,(1-o(1)).
\]
Combining this result with the first step completes the proof, yielding a high-probability lower bound for $Z$.
We obtain the following Lemma, whose proof is deferred to Appendix~\ref{app:proof:lem:main3}.

\begin{lemma}
\label{lem:main3}
If $d=\omega(\log n)$ and
\(
\frac{\rho^2}{1-\rho^2}n+\frac{2\eta^2}{1-\eta^2}d \,\leq\, 4(1-\varepsilon)\log n,
\)
then there exists a sequence $a_n=o\left(n\log n\right)$ such that
\begin{equation*}
Z\geq \exp\left(\tilde\varepsilon n \log(n)-a_n\right)=\exp\left(\tilde\varepsilon n \log(n)(1-o(1))\right),
\end{equation*}
with probability $1-o(1)$, where
$\tilde\varepsilon$ is defined by
\eqref{eq:effective-epsilon}.
\end{lemma}

\section{Exact recovery}\label{sec:exact}
In this section, we give the main ideas for the proof of our exact recovery results (\Cref{thm:exact_recovery}). Recall from \eqref{eq:Ppost} and \eqref{eq:hamiltonian} that
\(
\dP_{\mathrm{post}}(\pi)=\frac{e^{-V(\pi)}}{Z},
Z=\sum_{\pi'\in\mS_n}e^{-V(\pi')}\), and \(
V(\mathrm{id})=0 \, .
\)
The MAP estimator minimizes $V$, and fails at exact recovery precisely when some non-identity permutation has energy no larger than that of the identity.

\subsection{Achievability Result}
\begin{proof}[Proof of \Cref{thm:exact_recovery}, $(i)$]
As mentioned in \Cref{prop:crit_bayes}, the MAP estimator is the optimal estimator for exact recovery. Recall that we assume (without loss of generality) that $\pi^{\star}=\textrm{id}$, so that
\begin{equation*}
\dP\left(\widehat{\pi}_{\textrm{opt},0} \neq\pi^{\star}\right)=\dP\left(\textrm{MAP fails}\right)=\dP\left(\exists\pi\in\mathcal{S}_n\setminus \{\textrm{id}\} : V(\pi) \leq 0\right).
\end{equation*}
To analyze the MAP estimator, we use the partition of $\mathcal{S}_n\setminus \{\textrm{id}\}$ into $\mS_{n,t}$ sets for $2 \leq t \leq n$, and investigate the failure of the MAP estimator on each of these. We introduce the events
\begin{equation*}  
\mathcal{E}_t = \{\exists \pi \in \mS_{n,t}, \; V(\pi) \leq 0 \}.  
\end{equation*} 
We show that $\dP(\textrm{MAP fails})= \dP \left( \bigcup_{t=2}^{n} \mathcal{E}_t \right)$ is of order $o(1)$ under the conditions of \Cref{thm:exact_recovery}, $(i)$. The naive first moment method enables to deal with the union of $\mathcal{E}_t$ only when $t$ is not too close to $n$, otherwise it fails due to correlations across events $\mathcal{E}_t$---this phenomenon has already been observed by \cite{ganassali2022sharp}. Consequently, we distinguish the cases $2\leq t\leq \alpha_0 n$ for a well-chosen $\alpha_0 \in (0,1)$ (where first moment method works well), and $t > \alpha_0 n$ where one can take advantage of these correlations yielding a tighter control of the Laplace transform of $V(\pi)$, finally making this layered first moment method work. 
We sum up these ideas in the following Lemma proved in Appendix \ref{app:proof:sec:exact}, which concludes the proof.
\begin{lemma}\label{lem:sec:exact:1}
Assume that conditions of \Cref{thm:exact_recovery}, $(i)$ hold. Then, 
$$\dP \left( \bigcup_{2\leq t \leq n} \mathcal{E}_t \right) = o(1) \, .$$
\end{lemma} 
\end{proof}

\subsection{Converse Result}
\begin{proof}[Proof of \Cref{thm:exact_recovery}, $(ii)$]
To prove the converse part, we show that below the threshold, there exists, with high probability, \( \pi \in \mS_n \setminus \{\textrm{id}\} \) such that \( V(\pi) < 0 \), implying the failure of the MAP estimator, which is optimal for exact recovery. More precisely, we show that a significant number of permutations in \( \mS_{n,2} \), i.e., transpositions, satisfy \( V(\pi)< 0 \). We show this via the second moment method.

Let \( N \) be the random variable representing the number of transpositions that lead to the failure of the MAP estimator.
\begin{equation*}
N\vcentcolon=\sum_{\pi \in \mS_{n,2}}\mathbf{1}\{V(\pi)<0\}.
\end{equation*}
Since \( N \) is a random variable with a positive mean and finite variance, it follows from the Paley-Zygmund inequality that for all \( 0 < c < 1 \),
\begin{equation*}
\dP(N \geq c \ \dE[N]) \geq(1-c)^2 \frac{\dE[N]^2}{\dE\left[N^2\right]}.
\end{equation*}
Thus, if we prove that there is an event $\mathcal{H}_2$ such that \( \dE[N\mathbf{1}_{\mathcal{H}_2}] \to \infty \) and \( \dE[N^2\mathbf{1}_{\mathcal{H}_2}] \leq (1+o(1))\dE[N\mathbf{1}_{\mathcal{H}_2}]^2 \) when $n \to + \infty$, it follows that, with high probability, $N\mathbf{1}_{\mathcal{H}_2}= \omega(1)$ and $ N= \omega(1)$. We prove these results in the following Lemma, proven in Appendix \ref{app:proof:sec:exact}.

\begin{lemma}\label{lem:sec:exact:2}
Assume that $d=\omega(\log n)$ and $\frac{\rho^2n}{1-\rho^2}+\frac{\eta^2d}{1-\eta^2}=4\log n-\log\log n-a_n$, where $a_n=\omega(1)$ and $a_n=o(\log n)$. Then, there exists an event $\mathcal{H}_2$ (the same as in \Cref{lem:main2}) such that,
\( \dE[N\mathbf{1}_{\mathcal{H}_2}] \to \infty \) and \( \dE[N^2\mathbf{1}_{\mathcal{H}_2}] \leq (1+o(1))\dE[N\mathbf{1}_{\mathcal{H}_2}]^2 \, .\)
\end{lemma}
This entails \( N = \omega(1) \) w.h.p, and completes the proof.
\end{proof}

\section{Almost exact and partial recovery}\label{sec:almost-exact-partial}

 In this section, we prove Theorem \ref{thm:almost_exact_recovery}, achievability of almost exact recovery, and \Cref{thm:impossibility_partial_recovery}, impossibility of partial recovery. Recall that
\(
\dP_{\mathrm{post}}(\pi)=\frac{e^{-V(\pi)}}{Z},
Z_r(\pi)=\sum_{\pi'\in B(\pi,r)}e^{-V(\pi')}\), and \(
\overline Z_r(\pi)=\sum_{\pi'\notin B(\pi,r)}e^{-V(\pi')},
\)
so that
\(
\dP_{\mathrm{post}}(B(\pi,r))=Z_r(\pi)/Z
\)
and
\(
\dP_{\mathrm{post}}(B(\pi,r)^c)=\overline Z_r(\pi)/Z \, .
\)
\subsection{Achievability of almost exact recovery}\label{subsec:almost}

\begin{proof}[Proof of \Cref{thm:almost_exact_recovery}]
Assume that we are under the feasibility condition of \Cref{thm:almost_exact_recovery}, that is
\begin{equation}\label{eq:proof:thm:almost_exact_recovery_1}
    \frac{\rho^2}{1-\rho^2}n+\frac{2\eta^2}{1-\eta^2}d\geq 4(1+\varepsilon)\log n \, .
\end{equation}
Recall that we assume $\pi^{\star}=\textrm{id}$. In order to show that almost exact recovery is achievable, by \Cref{prop:crit_bayes},  it is enough to show that with high probability, $\mathbb{P}_{\textrm{post}}\left(B(\textrm{id},r)^c\right)=\overline Z_r(\textrm{id})/Z=o(1)$ for all $r\in(0,1)$. Therefore, we first derive a bound on the expectation of $\overline Z_r(\textrm{id})$ in the following lemma, proven in Appendix~\ref{app:proof:section5}.
\begin{lemma}
\label{lem:almost_exact}
    Assume that conditions of \Cref{thm:almost_exact_recovery} hold. Then, there exist high probability events $\mathcal{H}_1$, $\mathcal{H}_2$ (defined in \Cref{lem:main1} and \Cref{lem:main2}) such that, for any $r\in(0,1)$,
    \begin{align*}
        \mathbb{E}\left[\overline Z_r(\mathrm{id})\mathbf{1}_{\mathcal{H}_1\cap\mathcal{H}_2}\right]=o\left(\frac{1}{\log n}\right).
    \end{align*}
\end{lemma}

Based on the above lemma, Markov's inequality yields that 
with probability at least $1-o(1)$ we have
\begin{align*}
\overline Z_r(\mathrm{id})\mathbf{1}_{\mathcal{H}_1\cap\mathcal{H}_2}\leq \log n \times \mathbb{E}\left[\overline Z_r(\mathrm{id})\mathbf{1}_{\mathcal{H}_1\cap\mathcal{H}_2}\right]=o(1), 
\end{align*} and since ${\mathcal{H}_1\cap\mathcal{H}_2}$ is a high probability event, this gives $\overline Z_r(\mathrm{id})=o(1)$ with high probability.
The above implies that  $\mathbb{P}_{\textrm{post}}\left(B(\textrm{id},r)^c\right)=o(1)$ with high probability, since $Z \geq 1$. This completes the proof of \Cref{thm:almost_exact_recovery}.
\end{proof}

\subsection{Impossibility of partial recovery}
\begin{proof}[Proof of Theorem~\ref{thm:impossibility_partial_recovery}]
To prove the converse part, by Proposition \ref{prop:crit_bayes}, we need to show that for all constant $\gamma>0$ the probability of the event $\left\{\max_{\pi\in \mS_{n}}Z_r(\pi)>\gamma Z\right\}$ is $o(1)$ for any $r\in(0,1)$. We first obtain an upper bound for the expectation of $Z_r(\pi^{\star})$ in the following lemma, proven in Appendix~\ref{app:proof:section5}.

\begin{lemma}
    \label{lem:partial}
    Assume that conditions of \Cref{thm:impossibility_partial_recovery} hold. Then, there exist events $\mathcal{H}_1$, $\mathcal{H}_2$ (defined in \Cref{lem:main1} and \Cref{lem:main2}) with $\mathbb{P}\left(\mathcal{H}_1\cap\mathcal{H}_2\right)\geq 1-\exp\left(-C\log n\right)$ for a large enough constant $C>0$, and such that for any $r\in(0,1)$,
    \begin{align*}
        \mathbb{E}\left[Z_r(\pi^{\star})\mathbf{1}_{\mathcal{H}_1\cap\mathcal{H}_2}\right]\leq \exp\left(\tilde\varepsilon rn\log n(1+o(1))\right),
    \end{align*}
    where
\[
    \tilde\varepsilon
    :=1-\frac{1}{4\log n}
    \left(
        \frac{\rho^2n}{1-\rho^2}
        +\frac{2\eta^2d}{1-\eta^2}
    \right)
    \geq\varepsilon.
\]
\end{lemma}

\paragraph{First step: using \Cref{lem:partial} to control $\mathbb{P}_{\textrm{post}}\left(B(\pi^{\star},r)\right)$.}

\Cref{lem:partial} together with Markov's inequality yields that with probability at least $1-n^{-1}$,
\begin{align*}
Z_r(\pi^{\star}) \mathbf{1}_{\mathcal{H}_1\cap\mathcal{H}_2}&\leq \exp(\log n) \times \mathbb{E}\left[Z_r(\pi^{\star})\mathbf{1}_{\mathcal{H}_1\cap\mathcal{H}_2}\right]\nonumber
\leq \exp\left(\tilde\varepsilon rn\log n(1+o(1))\right) \, .
\end{align*}
Moreover,
$
\mathbb{P}\left((\mathcal{H}_1\cap\mathcal{H}_2)^c\right)
\leq \exp(-C\log n)=n^{-C}.
$
Therefore, with probability at least
$
1-n^{-1}-n^{-C}=1-o(1),
$
we have
\begin{align}
\label{eq:upper_Zr}
Z_r(\pi^{\star}) = Z_r(\pi^\star)\mathbf{1}_{\mathcal{H}_1\cap\mathcal{H}_2}\leq \exp\left(\tilde\varepsilon rn\log n(1+o(1))\right) \, .
\end{align} 

On the other hand, Lemma~\ref{lem:main3} shows that under the impossibility criterion of partial recovery, we have 
\begin{align}
\label{eq:lower_Z}
Z\geq \exp\left(\tilde\varepsilon n\log n(1-o(1))\right)
\end{align} with probability at least $1-o(1)$.
Combining \eqref{eq:lower_Z} and \eqref{eq:upper_Zr}, it results that with probability at least $1-o(1)$,
\begin{align*}
\mathbb{P}_{\textrm{post}}\left(B(\pi^{\star},r)\right)=\frac{Z_r(\pi^{\star})}{Z} \leq \exp\left(-(1-r)\tilde\varepsilon n\log n(1+o(1))\right)=o(1) \, .
\end{align*}
Because $0\leq \mathbb{P}_{\textrm{post}}\left(B(\pi^{\star},r)\right)\leq 1$, this implies
\begin{align}
\label{eq:exp-post}
    \mathbb{E}\left[\mathbb{P}_{\textrm{post}}\left(B(\pi^{\star},r)\right)\right]=o(1).
\end{align}

\paragraph{Second step: controlling $\max_{\pi\in\mS_n}
\dP_{\mathrm{post}}\bigl(B(\pi,r)\bigr)$.}
In order to complete the proof of
\Cref{thm:impossibility_partial_recovery}, it remains to show that
\[
\max_{\pi\in\mS_n}
\dP_{\mathrm{post}}\bigl(B(\pi,r)\bigr)
\xrightarrow{\dP_D}0.
\]
By Markov's inequality, it is sufficient to prove that
\[
\dE_D\left[
    \max_{\pi\in\mS_n}
    \dP_{\mathrm{post}}\bigl(B(\pi,r)\bigr)
\right]
=o(1) \, .
\]
To this end, we use the following Lemma, proven in Appendix~\ref{app:proof:section5}. This result bounds the largest $(1-\delta)-$ball mass in terms of the average ball mass at a larger radius, allowing us to control the maximal posterior ball probability through this average. Intuitively, if an overlap ball carries substantial probability mass, two independent samples have a substantial probability of both falling inside it, which forces a nontrivial mutual overlap on average. 
\begin{lemma}
\label{lem:partial-recovery-imp}
Let \(\mu\) be any probability measure on \(\cS_n\), and let \(\delta\in(0,1)\).
Define
\begin{align*}
    M_\delta
    &:=\max_{\pi\in \mS_n}
      \mu\!\left(\{\pi':\ov(\pi',\pi)\ge\delta\}\right)\\
    A_\delta
    &:=\sum_{\pi\in \mS_n}\mu(\pi)
       \mu\!\left(\{\pi':\ov(\pi',\pi)\geq \delta\}\right).
\end{align*}
Then 
\begin{equation*}
    M_\delta^2
    \leq \frac{2-\delta^2}{\delta^2}\,
       A_{\delta^2/2}.
\end{equation*}
\end{lemma}

Now, fix \(\delta\in(0,1)\), and, conditionally on the data \(D\), apply
Lemma~\ref{lem:partial-recovery-imp} to \(\mu=\Ppost\). For any $\delta \in (0,1)$, we have
\begin{align*}
    M_\delta =\max_{\pi\in \mS_n}\Ppost(B(\pi,1-\delta)) 
    \quad \textrm{and} \quad
    A_\delta = \sum_{\pi\in \mS_n}\Ppost(\pi)\Ppost(B(\pi,1-\delta)).
\end{align*}
Since the conditional law of \(\pi^\star\) given \(D\) is \(\mathbb{P}_{\text{post}}\), we have
\(
    \E_D[A_\delta]
    =\E
      \left[\Ppost(B(\pi^\star,1-\delta))\right]
\). Using \eqref{eq:exp-post} with radius
\(r = 1-\delta^2/2\), we obtain \(\E_D[A_{\delta^2/2}]=o(1)\). Therefore,
\[
    \E_D[M_\delta^2]
    \leq \frac{2-\delta^2}{\delta^2}\E_D[A_{\delta^2/2}]
    =o(1).
\]
Thus, for every fixed $r\in(0,1)$, setting $\delta=1-r$, we have
\begin{align*}
    \E\left[\max_{\pi\in\mS_n}
    \dP_{\mathrm{post}}\bigl(B(\pi,r)\bigr)\right]=\E_D[M_\delta]
     \leq \sqrt{\E_D[M_\delta^2]}
     =o(1),
\end{align*}
which completes the proof.
\end{proof}

\paragraph{AI Disclosure.}
Most of the research in this paper, including the initial ideas and proof
strategies, was developed by the authors without AI assistance. After
working on the arguments independently, the authors consulted ChatGPT for
suggestions on three specific technical points:

\begin{enumerate}
    \item For the impossibility of partial recovery in
    Theorem~\ref{thm:impossibility_partial_recovery}, an initial proof
    established the result under the condition
    \[
        \frac{\rho^2 n}{1-\rho^2}
        + \frac{2\eta^2 d}{1-\eta^2}
        \leq 2(1-\eps)\log n.
    \]
    To strengthen this result, the authors consulted ChatGPT, which suggested
    controlling $\mathcal{M}_n$ in \eqref{def:Mn} rather than
    $\mathcal{T}_n$ in \eqref{def:Tn} in the proof of
    Lemma~\ref{lem:main3}. This suggestion led to the improved condition
    with $4(1-\eps)\log n$ on the right-hand side. The subsequent analysis
    of $\mathcal{M}_n$ was developed by the authors without AI assistance.

    \item For Lemma~\ref{lem:partial-recovery-imp}, the authors had
    established a valid but more complicated auxiliary result. ChatGPT helped simplify the approach by suggesting the idea that the maximum probability over a ball could be controlled by the mean probability over the same ball.

    \item For Lemma~\ref{lem:gaussian-upper-nonasymptotic}, an original
    argument relied on a more complicated asymptotic estimate for a
    joint Gaussian tail probability. ChatGPT proposed the non-asymptotic
    bound stated in Lemma~\ref{lem:gaussian-upper-nonasymptotic}, which simplified the analysis.
\end{enumerate}

ChatGPT was also used in places to improve the wording and correct grammatical
errors. The authors take full responsibility for the mathematical
results and the final manuscript.

\bibliography{src}

\newpage
\begin{appendices}
\section{Proof of \Cref{lem:main1}}\label{app:proof:lem:main1}
\begin{proof}[Proof of \Cref{lem:main1}] 
Recall that without loss of generality, we assume $\pi^\star=\mathrm{id}$. Recall the expression of $V^*_{G}(\pi)$ and $V^*_{F}(\pi)$, for all $\pi \in \cS_n$:
\begin{align*} V^*_{G}(\pi) =\sum\limits_{\substack{1\leq i<j \leq n\\ \pi(i,j)\neq \{i,j\}}}A_{i,j}B_{i,j}, \quad \quad 
V^*_{F}(\pi) =\sum\limits_{\substack{1\leq i \leq n, \ 1\leq k\leq d\\ \pi(i)\neq i}}X_{i,k}Y_{i,k} \, .
\end{align*} 

\paragraph{First step: $V^*_{G}(\pi)$ and $V^*_{F}(\pi)$ as quadratic forms.}
The first step in proving a uniform deviation control on $V^*_{G}(\pi)$ and $V^*_{F}(\pi)$ is to express them as standard quadratic forms. 
To this end, we define $v_{G}\in\mathbb{R}^{n(n-1)}$ by arranging the edge weights of the pair $(\bA,\bB)$ in a pairwise manner:
\[ v_{G} := (A_{1,2},B_{1,2},\ldots, A_{1,n},B_{1,n},A_{2,3},B_{2,3},\ldots, A_{n-1,n}, B_{n-1,n} )^\top \in \dR^{n(n-1)} \] so that each corresponding pair $(A_{i,j},B_{i,j})$ appears consecutively. Similarly, we define $v_{F}\in\mathbb{R}^{2nd}$ by arranging the feature information $(\bX,\bY)$ in the same way:
\[ v_{F} := (X_{1,1},Y_{1,1},\ldots, X_{1,d},Y_{1,d},X_{2,1},Y_{2,1},\ldots, X_{n,d}, Y_{n,d} )^\top \in \dR^{2nd} \] so that each corresponding pair $(X_{i,k},Y_{i,k})$ appears consecutively.

Note that $v_{G}$ and $v_{F}$ are independent Gaussian vectors with zero mean and covariance matrices $\mathbf{\Sigma}_G \in \dR^{n(n-1) \times n(n-1)}$ and $\mathbf{\Sigma}_F  \in \dR^{2nd \times 2nd}$, respectively, where $\mathbf{\Sigma}_G$ and $\mathbf{\Sigma}_F$ are block diagonal with blocks given by
$\begin{bmatrix} 1 & \rho \\ \rho & 1 \end{bmatrix}$ and
$\begin{bmatrix} 1 & \eta \\ \eta & 1 \end{bmatrix}$, respectively.
With these notations, $V^*_{G}(\pi)$ and $V^*_{F}(\pi)$ write as quadratic forms, as follows:
\[
V^*_{G}(\pi) = v^{\top}_{G} \mathbf{M}_G(\pi) v_{G} \quad \mbox{and} \quad V^*_{F}(\pi) = v^{\top}_{F} \mathbf{M}_F(\pi) v_{F},
\]
where $\mathbf{M}_G(\pi) \in \dR^{n(n-1) \times n(n-1)}$ and $\mathbf{M}_F(\pi) \in \dR^{2nd \times 2nd}$ are block-diagonal matrices with blocks of size $2 \times 2$, designed as follows. For each pair $(A_{i,j},B_{i,j})$ (resp. $(X_{i,k},Y_{i,k})$), if $\pi(i,j)\neq \{i,j\}$ (resp. $\pi(i)\neq i$), then the corresponding block in $\mathbf{M}_G(\pi)$ (resp. of $\mathbf{M}_F(\pi)$) is given by $\begin{bmatrix} 0 & 1/2 \\ 1/2 & 0 \end{bmatrix}$, and otherwise the block is the null matrix.

Concentration inequalities for $V^*_{G}(\pi)$ and $V^*_{F}(\pi)$ will be obtained from the classical Hanson–Wright inequality \citep[see][for a proof]{HWI}.

\begin{lemma}[Hanson–Wright Inequality]
\label{lem:HW}
Let $z$ be a random vector with i.i.d. standard Gaussian entries, and let $\mathbf{A}$ be a deterministic square matrix. Then there exists a universal constant $c>0$ such that for all $x >0$, 
\begin{equation*}
\mathbb{P}\!\left(\,\left|z^{\top}\mathbf{A}z-\mathbb{E}\!\left[z^{\top}\mathbf{A}z\right]\right| \geq x\right)  
\leq 2 \exp\!\left(-\frac{c\, x^2}{\|\mathbf{A}\|_{F}^{2}+\|\mathbf{A}\|_{\mathrm{op}}\, x}\right).
\end{equation*}
\end{lemma}

To apply Hanson–Wright inequality in our setting, note that in distribution, $v_{G}=\mathbf{\Sigma}^{1/2}_{G}z$ where $z$ is a standard Gaussian vector. Under this representation, for a fixed $\pi$ we obtain, for $x>0$, 
\begin{align}
& \mathbb{P}\left(\left|v_{G}^T\mathbf{M}_G(\pi)v_{G}-\mathbb{E}\left[v_{G}^T\mathbf{M}_G(\pi)v_{G}\right]\right|>x\right) \nonumber\\
& \quad \quad =\mathbb{P}\!\left(\,\left|z^{\top}\mathbf{\Sigma}^{1/2}_{G}\mathbf{M}_G(\pi)\mathbf{\Sigma}^{1/2}_{G}z-\mathbb{E}\!\left[z^{\top}\mathbf{\Sigma}^{1/2}_{G}\mathbf{M}_G(\pi)\mathbf{\Sigma}^{1/2}_{G}z\right]\right|>x\right)  \nonumber
\\ \label{eq:proof:lem:main1_1}
& \quad \quad \leq 2 \exp\!\left(-\frac{c\, x^{2}}{\|\mathbf{\Sigma}^{1/2}_{G}\mathbf{M}_G(\pi)\mathbf{\Sigma}^{1/2}_{G}\|_{F}^{2}+\|\mathbf{\Sigma}^{1/2}_{G}\mathbf{M}_G(\pi)\mathbf{\Sigma}^{1/2}_{G}\|_{\mathrm{op}}\, x}\right).
\end{align}
Applying the same argument to $V^*_{F}(\pi)$ yields
\begin{equation}
\label{eq:proof:lem:main1_2}
\mathbb{P}\left(\left|V^*_{F}(\pi)-\mathbb{E}\left[V^*_{F}(\pi)\right]\right|>x\right)\leq 2 \exp\!\left(-\frac{c\, x^{2}}{\|\mathbf{\Sigma}^{1/2}_{F}\mathbf{M}_F(\pi)\mathbf{\Sigma}^{1/2}_{F}\|_{F}^{2}+\|\mathbf{\Sigma}^{1/2}_{F}\mathbf{M}_F(\pi)\mathbf{\Sigma}^{1/2}_{F}\|_{\mathrm{op}}\, x}\right).
\end{equation}
Note that, by the definitions of $\mathbf{M}_F(\pi)$ and $\mathbf{\Sigma}_{F}$, we can bound the operator norm in \eqref{eq:proof:lem:main1_2} as 
\begin{equation*}
\|\mathbf{\Sigma}^{1/2}_{F}\mathbf{M}_F(\pi)\mathbf{\Sigma}^{1/2}_{F}\|_{\mathrm{op}}
\leq \|\mathbf{M}_F(\pi)\|_{\mathrm{op}} \|\mathbf{\Sigma}_{F}^{1/2}\|_{\mathrm{op}}^2\, 
\leq \tfrac{1}{2}(1+\eta)\leq 1 \, .
\end{equation*}
The Frobenius norm in \eqref{eq:proof:lem:main1_2} can be bounded as follows:
\begin{equation*}
\|\mathbf{\Sigma}^{1/2}_{F}\mathbf{M}_F(\pi)\mathbf{\Sigma}^{1/2}_{F}\|^{2}_{F}
= \|\mathbf{\Sigma}_{F}\mathbf{M}_F(\pi)\|^{2}_{F}
\stackrel{(a)}{=}\tfrac{d\,|D_{\pi}|}{2}(1+\eta^2)
\leq d\,|D_{\pi}|,
\end{equation*}
where step $(a)$ follows from the fact that $\mathbf{\Sigma}_{F}\mathbf{M}_F(\pi)$ is block diagonal with $d|D_{\pi}|$ nonzero blocks of the form
$\begin{bmatrix} 0.5\eta & 0.5 \\  0.5 & 0.5\eta \end{bmatrix}$ and all remaining blocks are null. By the same reasoning, for $V^*_{G}(\pi)$, the operator and Frobenius norms in \eqref{eq:proof:lem:main1_1} can be controlled as follows:
\begin{equation*}
\|\mathbf{\Sigma}^{1/2}_{G}\mathbf{M}_G(\pi)\mathbf{\Sigma}^{1/2}_{G}\|_{\mathrm{op}} \leq \frac{1}{2}(1+\rho) \leq 1, \quad \mbox{and} \quad
\|\mathbf{\Sigma}^{1/2}_{G}\mathbf{M}_G(\pi)\mathbf{\Sigma}^{1/2}_{G}\|^{2}_{F} \leq |D^{E}_{\pi}| \frac{1+\rho^2}{2} \leq |D^{E}_{\pi}|,
\end{equation*} since (for the Frobenius norm) $\mathbf{\Sigma}_{G}\mathbf{M}_G(\pi)$ is block diagonal with $|D^E_{\pi}|$ nonzero blocks of the form
$\begin{bmatrix} 0.5\rho & 0.5 \\  0.5 & 0.5\rho \end{bmatrix}$ and all remaining blocks are null.  

\paragraph{Second step: proving \Cref{lem:main1} for $V^*_{F}(\pi)$.} To derive a tight concentration bound for $V^*_{F}(\pi)$, we exploit the fact that permutations sharing the same set $D_{\pi}$ yield the same value for $V^*_{F}(\pi)$. Indeed, for $\pi \in \cS_n$,
\begin{equation*}
V^*_{F}(\pi) =\sum\limits_{\substack{1\leq i \leq n, \ 1\leq k\leq d\\ \pi(i)\neq i}}X_{i,k}Y_{i,k} = \sum_{i \in D_{\pi}} \langle X_{i,:} ,Y_{i,:} \rangle \, .
\end{equation*}
We therefore, for a given subset $D \subseteq [n]$ use the notation $V^*_{F}(D)$ to denote $V^*_{F}(\pi)$ for all $\pi$ such that $D_{\pi} = D$. Based on this property, let us fix $x_t$ for each $t \in [n]$ to be specified later. We obtain
\begin{align*}
&\mathbb{P}\left(\exists t\in[n],\exists \pi\in\mathcal{S}_{n,t}: \left|V^*_{F}(\pi)-\mathbb{E}\left[V^*_{F}(\pi)\right]\right|>x_t\right)\\
& \qquad =\mathbb{P}\left(\exists t\in[n],\exists D\subseteq[n]: |D|=t, \left|V^*_{F}(D)-\mathbb{E}\left[V^*_{F}(D)\right]\right|>x_t\right)\\
& \qquad \leq\sum_{t=1}^{n}\sum_{\substack{D\subseteq[n] |D|=t}}\mathbb{P}\left(\left|V^*_{F}(D)-\mathbb{E}\left[V^*_{F}(D)\right]\right|>x_t\right) \stackrel{(a)}{\leq} \sum_{t=1}^{n}\binom{n}{t}2 \exp\!\left(-\frac{c\, x_t^{2}}{dt+ x_t}\right)&&
\end{align*}
where in (a) we apply the Hanson--Wright inequality from \Cref{lem:HW} as well as \eqref{eq:proof:lem:main1_2}. Now, by setting  
\[
x_t = C t \sqrt{\max\!\left\{d, \log\!\left(\tfrac{en}{t}\right)\right\} \log\!\left(\tfrac{en}{t}\right)}
\]   
for some sufficiently large constant $C>0$, we obtain  
\begin{align*}
\sum_{t=1}^{n}\binom{n}{t}2 \exp\!\left(-\frac{c\, x_t^{2}}{dt+ x_t}\right)
&=\sum_{t=1}^{n}\binom{n}{t}2 \exp\!\left(-\frac{c\ C^2 t^2\max\{d,\log\left(\frac{en}{t}\right)\}\log\left(\frac{en}{t}\right)}{dt+ Ct\sqrt{\max\{d,\log\left(\frac{en}{t}\right)\}\log\left(\frac{en}{t}\right)}}\right)\\
&\leq \sum_{t=1}^{n}\binom{n}{t}2 \exp\left(-C't\log\left(\frac{en}{t}\right)\right)\\
&\stackrel{(a)}{\leq} \sum_{t=1}^{n} \exp\left(-C''t\log\left(\frac{en}{t}\right)\right)\\
&\stackrel{(b)}{\leq} \exp\left(-(C''-1)\log(n)\right)=o(1),
\end{align*}
where $C',C''>1$ are constants. Step $(a)$ follows from the inequality $\binom{n}{t}\leq \left(\tfrac{en}{t}\right)^t$, and step $(b)$ follows from the fact that the minimum of $t \log\!\left(\tfrac{en}{t}\right)$ over $t \in [n]$ is attained at $t=1$, yielding $\log(n)+1$.

\paragraph{Third step: applying union bound to $V^*_{G}(\pi)$.} We now turn to the proof of \Cref{lem:main1} for $V^*_{G}$. The idea behind obtaining a tight bound is similar to that of $V^*_{F}$, namely, grouping together all permutations that yield identical values of $V^*_{G}$. However, the situation is more complicated than in the case of $V^*_{F}$, since $V^*_{G}(\pi)$ is determined by $D^{E}_{\pi}$ rather than $D_{\pi}$. Consequently, since permutations with the same $D_{\pi}$ do not necessarily share the same $D^{E}_{\pi}$, they may not produce the same value of $V^*_{G}$. Nevertheless, we know that for any $\pi$ and $\pi'$ such that $D_{\pi} = D_{\pi'}$, we have  
\[
\big|D^{E}_{\pi} \triangle D^{E}_{\pi'}\big| \leq \big|D_{\pi}\big|,
\] since the only edges in $D^E_{\pi}$ that are not in $D^E_{\pi'}$ correspond to pairs of nodes in $D_{\pi}$ that are exchanged by the permutation $\pi'$ but not by $\pi$; the number of such pairs is at most $|D_{\pi}|/2$.

Moreover, both $\big|D^{E}_{\pi}\big|$ and $\big|D^{E}_{\pi'}\big|$ are of order $\Omega(n|D_{\pi}|)$. Therefore, the symmetric difference between $D^{E}_{\pi}$ and $D^{E}_{\pi'}$ is negligible. This implies that for all permutations with the same $D_{\pi}$, the values of $V^*_{G}(\pi)$ are approximately equal. We formalize this observation in the following lemma.

\begin{lemma}
\label{lem:l1}
For $D \subseteq [n]$, define the set of all permutations whose set of unfixed points is equal to $D$ as
\[
\mathcal{P}_{D} \vcentcolon= \{\pi \in \mS_{n} : D_{\pi} = D\},
\]
and suppose that $\tilde{\pi}$ is an arbitrary permutation in $\mathcal{P}_{D}$. If $x>C|D|\log(|D|)$ for some large enough constant $C>0$, then we obtain
\[
\mathbb{P}\left(\exists \pi \in \mathcal{P}_{D} : \left| V^*_{G}(\pi) - \mathbb{E}\!\left[V^*_{G}(\pi)\right] \right| > x \right)
\leq  \, \mathbb{P}\left(\left| V^*_{G}(\tilde{\pi}) - \mathbb{E}\!\left[V^*_{G}(\tilde{\pi})\right] \right| > \frac{x}{2} \right)+\exp\left(-C'\,x\right).
\]
\end{lemma}
\begin{proof}[Proof of \Cref{lem:l1}]
For $\pi,\tilde{\pi}\in\mathcal{P}_{D}$ define
\[
V^*_{G}\left(\pi\setminus\tilde{\pi}\right)\vcentcolon=\sum\limits_{\substack{1\leq i<j \leq n\\ \{i,j\}\in D^{E}_{\pi}\setminus D^{E}_{\tilde{\pi}}}}B_{i,j}A_{i,j}.
\]
Thus, for all $\pi$ in $\mathcal{P}_{D}$ it holds that $V^*_{G}(\pi)=V^*_{G}(\tilde{\pi})+V^*_{G}(\pi\setminus\tilde{\pi})-V^*_{G}(\tilde{\pi}\setminus\pi)$. As previously noted, since $\left|D^{E}_{\pi}\triangle D^{E}_{\tilde{\pi}}\right|$ is relatively small, the term $V^*_{G}(\pi\setminus\tilde{\pi})-V^*_{G}(\tilde{\pi}\setminus\pi)$ does not exhibit significant fluctuations. To make this precise, suppose that we fix $\pi\in\mathcal{P}_{D}\setminus\{\tilde{\pi}\}$, then we have
\begin{align*}
&\mathbb{P}\left(\left| V^*_{G}(\pi) - \mathbb{E}\!\left[V^*_{G}(\pi)\right] \right| > x,\  \left| V^*_{G}(\tilde{\pi}) - \mathbb{E}\!\left[V^*_{G}(\tilde{\pi})\right] \right| \leq \frac{x}{2} \right)\\
&=\mathbb{P}\left(\left| V^*_{G}(\tilde{\pi}) - \mathbb{E}\!\left[V^*_{G}(\tilde{\pi})\right] +V^*_{G}(\pi\setminus\tilde{\pi})-V^*_{G}(\tilde{\pi}\setminus\pi) - \mathbb{E}\!\left[V^*_{G}(\pi\setminus\tilde{\pi})-V^*_{G}(\tilde{\pi}\setminus\pi)\right]\right| > x,\  \left| V^*_{G}(\tilde{\pi}) - \mathbb{E}\!\left[V^*_{G}(\tilde{\pi})\right] \right| \leq \frac{x}{2} \right)\\
&\leq \mathbb{P}\left(\left| V^*_{G}(\pi\setminus\tilde{\pi}) - \mathbb{E}\!\left[V^*_{G}(\pi\setminus\tilde{\pi})\right]\right| > \frac{x}{4} \right)+\mathbb{P}\left(\left| V^*_{G}(\tilde{\pi}\setminus\pi) - \mathbb{E}\left[V^*_{G}(\tilde{\pi}\setminus\pi)\right]\right| > \frac{x}{4} \right)\\
&\stackrel{(a)}{\leq}2\exp\left(-\frac{c\,x^2}{8|D|+4x}\right)+2\exp\left(-\frac{c\,x^2}{8|D|+4x}\right)\stackrel{(b)}{\leq}\exp\left(-C''\,x\right),
\end{align*}
where in (a) we use the Hanson-Wright inequality of \Cref{lem:HW} and the fact that $\left|D^{E}_{\pi}\setminus D^{E}_{\tilde{\pi}}\right|\leq|D|/2$ which enables us to improve the denominator in the exponential term, replacing $n|D|$ with $|D|$; in $(b)$ we use the assumption that $x\geq C|D|\log\left(|D|\right)$. Now, we can prove \Cref{lem:l1}.
\begin{align*}
&\mathbb{P}\left(\exists \pi \in \mathcal{P}_{D} : \left| V^*_{G}(\pi) - \mathbb{E}\!\left[V^*_{G}(\pi)\right] \right| > x \right)\\&\leq\mathbb{P}\left(\left| V^*_{G}(\tilde{\pi}) - \mathbb{E}\!\left[V^*_{G}(\tilde{\pi})\right] \right| > \frac{x}{2} \right)+\mathbb{P}\left(\exists \pi \in \mathcal{P}_{D}\setminus\{\tilde{\pi}\} : \left| V^*_{G}(\pi) - \mathbb{E}\!\left[V^*_{G}(\pi)\right] \right| > x,\ \left| V^*_{G}(\tilde{\pi}) - \mathbb{E}\!\left[V^*_{G}(\tilde{\pi})\right] \right| \leq\frac{x}{2} \right)\\
&\leq \mathbb{P}\left(\left| V^*_{G}(\tilde{\pi}) - \mathbb{E}\!\left[V^*_{G}(\tilde{\pi})\right] \right| > \frac{x}{2} \right)+\sum_{\pi\in \mathcal{P}_{D}\setminus\{\tilde{\pi}\}}\mathbb{P}\left(\left| V^*_{G}(\pi) - \mathbb{E}\!\left[V^*_{G}(\pi)\right] \right| > x,\ \left| V^*_{G}(\tilde{\pi}) - \mathbb{E}\!\left[V^*_{G}(\tilde{\pi})\right] \right| \leq\frac{x}{2} \right)\\
&\stackrel{(a)}{\leq} \mathbb{P}\left(\left| V^*_{G}(\tilde{\pi}) - \mathbb{E}\!\left[V^*_{G}(\tilde{\pi})\right] \right| > \frac{x}{2} \right)+e^{|D|\log(|D|)} \ \mathbb{P}\left(\left| V^*_{G}(\pi) - \mathbb{E}\!\left[V^*_{G}(\pi)\right] \right| > x,\ \left| V^*_{G}(\tilde{\pi}) - \mathbb{E}\!\left[V^*_{G}(\tilde{\pi})\right] \right| \leq\frac{x}{2} \right)\\
&\stackrel{(b)}{\leq} \mathbb{P}\left(\left| V^*_{G}(\tilde{\pi}) - \mathbb{E}\!\left[V^*_{G}(\tilde{\pi})\right] \right| > \frac{x}{2} \right)+\exp\left(-C'x\right).
\end{align*}
Here, step $(a)$ follows from the bound $|\mathcal{P}_{D}| \leq |D|! \leq |D|^{|D|}$, 
while step (b) holds under the assumption that $x > C |D| \log(|D|)$.
\end{proof}

Based on \Cref{lem:l1}, we finally prove \Cref{lem:main1} for $V_G^*$. Set, for all $t \in [n]$, $x_t=Ct\sqrt{n\log\left(\frac{en}{t}\right)}$ for some sufficiently large constant $C>0$.
\begin{align*}
\mathbb{P}&\left(\exists t\in[n],\exists \pi\in\mathcal{S}_{n,t}: \left|V^*_{G}(\pi)-\mathbb{E}\left[V^*_{G}(\pi)\right]\right|>x_t\right)\\&\leq\sum_{t=1}^{n}\sum_{\substack{D\subseteq[n]\\ |D|=t}}\mathbb{P}\left(\exists \pi \in \mathcal{P}_{D} : \left| V^*_{G}(\pi) - \mathbb{E}\!\left[V^*_{G}(\pi)\right] \right| > x_t \right)\\
&\stackrel{(a)}{\leq} \sum_{t=1}^{n}\sum_{\substack{D\subseteq[n]\\ |D|=t}}\mathbb{P}\left(\left| V^*_{G}(\tilde{\pi}) - \mathbb{E}\!\left[V^*_{G}(\tilde{\pi})\right] \right| > \frac{x_t}{2} \right)+\exp\left(-C'\,x_t\right)\\
&\leq \sum_{t=1}^{n}\binom{n}{t}2\exp\left(-\frac{c\ C^2 t^2n\log\left(\frac{en}{t}\right)}{4nt+2 Ct\sqrt{n\log\left(\frac{en}{t}\right)}}\right)+\sum_{t=1}^{n}\binom{n}{t}\exp\left(-C'Ct\sqrt{n\log\left(\frac{en}{t}\right)}\right)\\
&\stackrel{(b)}{\leq} \sum_{t=1}^{n}\exp\left(-C'' t\log\left(\frac{en}{t}\right)\right)+\sum_{t=1}^{n}\exp\left(-C''t\sqrt{n}\right)=o(1),
\end{align*}
where (a) follows from \Cref{lem:l1}, noting that $x_t>Ct\log n \geq Ct \log t$, and in $(b)$ we used $n \geq \sqrt{n \log(en/t)}$ in the first part of RHS, and $\sqrt{n \log(en/t)} \geq \sqrt{n}$ in the second part. Combining the bounds for $V^*_{F}$ and $V^*_{G}$ completes the proof of \Cref{lem:main1}.
\end{proof}

\section{Proof of Lemma~\ref{lem:main2}}\label{app:proof:lem:main2}
\begin{proof}[Proof of \Cref{lem:main2}]
The key idea in proving this lemma is to condition the exponential terms on $\mathbf{B}$ and $\mathbf{Y}$. Recall that without loss of generality, we assumed $\pi^{\star} = \mathrm{id}$. Thus, we have $A_{i,j}=\rho B_{i,j}+\sqrt{1-\rho^2}\xi_{i,j}$ and $X_{i,k}=\eta Y_{i,k}+\sqrt{1-\eta^2}\xi'_{i,k}$, where $\xi_{i,j}$ and $\xi'_{i,k}$ are independent standard Gaussian random variables. In this case, we have
\begin{align*}
V_{G}(\pi)|\mathbf{B}&\sim\rho\sum_{\substack{1\leq i< j \leq n\\ \pi(i,j)\neq\{i,j\}}}B_{\pi(i,j)}B_{i,j}+\sqrt{1-\rho^2}\mathcal{N}\left(0,\sbp\right)\\
V_{F}(\pi)|\mathbf{Y}&\sim\eta\sum_{\substack{1\leq i \leq n, \ 1\leq k\leq d\\ \pi(i)\neq i}}Y_{\pi(i),k}Y_{i,k}+\sqrt{1-\eta^2}\mathcal{N}\left(0,\syp\right)
\end{align*}
where  
\begin{equation}
\label{def:variance}
\sbp\vcentcolon=\sum_{\substack{1\leq i< j \leq n\\ \pi(i,j)\neq\{i,j\}}}B_{\pi(i,j)}^2,\qquad \syp\vcentcolon=\sum_{\substack{1\leq i \leq n, \ 1\leq k\leq d\\ \pi(i)\neq i}}Y_{\pi(i),k}^2.
\end{equation}

The second step is the introduction of an event $\mathcal{H}_2$, under which the following quantities are well concentrated around their means, for all permutations $\pi,\pi' \in \mS_n$:
\begin{align}
\cbp&\vcentcolon=\sum_{\substack{1\leq i< j \leq n\\ \pi(i,j)\neq\{i,j\}}}B_{\pi(i,j)}B_{\pi'(i,j)},\label{def:cbp}\\
\cyp&\vcentcolon=\sum_{\substack{1\leq i \leq n, \ 1\leq k\leq d\\ \pi(i)\neq i}}Y_{\pi(i),k}Y_{\pi'(i),k}.\label{def:cyp}
\end{align}
In particular, for $\pi=\pi'$ we have $\cbp=\sbp$ and $\cyp=\syp$. In the following lemma, we formally construct the event $\mathcal{H}_2$ and show that it occurs with high probability. The idea of the proof is similar to the analysis of the event $\mathcal{H}_1$ in \Cref{lem:main1}, where the Hanson–Wright inequality was applied. The only difference is that we can no longer exploit the fact that the random variables depend only on the set of unfixed points of $\pi$, as they are almost unique for each pair of permutations $\pi$ and $\pi'$. Thus, to obtain a uniform concentration bound over all pairs $(\pi,\pi')$, we need to slightly relax the bounds compared to \Cref{lem:main1} in order to ensure this uniformity with high probability.

\begin{lemma}
\label{lem:sub2}
Define
\begin{align*}
\mathcal H_2^{a}
&:=\Bigg\{\forall t\in[n],\ \forall\pi\in\mS_{n,t}:
|\sbp-\dE[\sbp]|\leq Ct\sqrt{n\log \left(\frac{en}{t}\right)},\\ &\qquad\qquad\qquad\qquad\qquad \quad |\syp-\dE[\syp]|\leq Ct\sqrt{\max\left\{d,\log\left(\frac{en}{t}\right)\right\}\log\left(\frac{en}{t}\right)}\Bigg\},\\
\mathcal H_2^{b}&:=\Bigg\{\forall t\in[n], \forall \pi\in\mS_{n,t}:\\ 
&
\qquad\ \left|\cbp-\mathbb{E}\left[\cbp\right]\right|\leq Ct\sqrt{n\log\left(n\right)}\ \text{ for all } \pi'\in\mS_n \text{ with } |D_{\pi'}|\leq 2t,\\  \ 
&\qquad\ \ \left|\cyp-\mathbb{E}\left[\cyp\right]\right|\leq Ct\sqrt{\max\{d,\log\left(n\right)\}\log\left(n\right)}\ \text{ for all } \pi'\in\mS_n\Bigg\},
\end{align*}
and let $\mathcal H_2:=\mathcal H_2^{a}\cap\mathcal H_2^{b}$,
where $C>0$ is a universal constant. Then the event $\mathcal{H}_2$ occurs with probability at least $1-o(1)$.

\end{lemma}

\begin{proof}[Proof of \Cref{lem:sub2}]
We prove this lemma by showing $\mathbb{P}\left(\mathcal{H}^{a}_2\right)=1-o(1)$ and $\mathbb{P}\left(\mathcal{H}^{b}_2\right)=1-o(1)$. 

For $\cH^{a}_2$, the proof is similar to Lemma~\ref{lem:main1}, the concentration bounds for $\sbp$ and $\syp$ can be sharpened by replacing $\log n$ with $\log\!\left(\frac{e n}{t}\right)$ for $\pi \in \mS_{n,t}$. This refinement follows from the fact that for all permutations $\pi$ sharing the same set $\{\pi(i) : i \in D_{\pi}\}$, the value of $\syp$ is identical, while $\sbp$ differs only marginally. Consequently, we may apply the same argument used in the proof of event $\mathcal{H}_1$ in Lemma~\ref{lem:main1}.

For $\cH^{b}_2$, we define vectors $b\in\mathbb{R}^{\frac{n(n-1)}{2}}$ and $y\in\mathbb{R}^{nd}$ by vectorizing the edge information $B_{i,j}$ and the feature information $Y_{i,k}$, respectively. Therefore, $b$ and $y$ are independent standard Gaussian vectors. Analogous to \Cref{lem:main1}, we can express $\cbp$ and $\cyp$ as quadratic forms as follows.
\[
\cbp=b^\top\!\left((\mathbf{1}_N\mathbf{1}_N^\top-I_N)\odot\Pi_B\right)^\top\!\Pi'_B b\]
and
\[\cyp=y^\top\!\left((\mathbf{1}_{N'}\mathbf{1}_{N'}^\top-I_{N'})\odot\Pi_Y\right)^\top\!\Pi'_Y y,
\]
where $\Pi_B$ and $\Pi_Y$ (resp. $\Pi'_B$ and $\Pi'_Y$) are permutation matrices associated with the permutation $\pi$ (resp. $\pi'$) for $b$ and $y$, respectively. Here, we define $N\vcentcolon=n(n-1)/2$ and $N'\vcentcolon=nd$, and $\odot$ denotes the Hadamard (elementwise) product. Note that
\begin{align*}
\left\| \left((\mathbf{1}_N\mathbf{1}_N^\top-I_N)\odot\Pi_B\right)^\top\!\Pi'_B \right\|_{\textrm{op}}&\leq \left\| (\mathbf{1}_N\mathbf{1}_N^\top-I_N)\odot\Pi_B\right\|_{\textrm{op}}\left\| \Pi'_B \right\|_{\textrm{op}}\stackrel{(a)}{=} \left\| \Pi_B-\mathrm{diag}(\Pi_B) \right\|_{\textrm{op}} \stackrel{(b)}{\leq} 1, \\
\left\| \left((\mathbf{1}_{N'}\mathbf{1}_{N'}^\top-I_{N'})\odot\Pi_Y\right)^\top\!\Pi'_Y \right\|_{\textrm{op}}&\leq \left\| (\mathbf{1}_{N'}\mathbf{1}_{N'}^\top-I_{N'})\odot\Pi_Y\right\|_{\textrm{op}}\left\| \Pi'_Y \right\|_{\textrm{op}}\stackrel{(c)}{=} \left\| \Pi_Y-\mathrm{diag}(\Pi_Y) \right\|_{\textrm{op}} \stackrel{(d)}{\leq} 1,
\end{align*} 
where in $(a)$ and $(c)$ we use the fact that the operator norm of a permutation matrix is $1$, while $(b)$ and $(d)$ follow because removing the diagonal entries of a permutation matrix cannot increase its maximum row or column sum. Therefore,
\[
\|\Pi_B-\mathrm{diag}(\Pi_B)\|_{\mathrm{op}}
\leq
\sqrt{\|\Pi_B-\mathrm{diag}(\Pi_B)\|_1\|\Pi_B-\mathrm{diag}(\Pi_B)\|_\infty}\leq 1 \, .
\]

Moreover, we have
\begin{align*}
\left\| \left((\mathbf{1}_N\mathbf{1}_N^\top-I_N)\odot\Pi_B\right)^\top\!\Pi'_B \right\|_{F}&= \left\| (\mathbf{1}_N\mathbf{1}_N^\top-I_N)\odot\Pi_B\right\|_{F}=\sqrt{\left|D^E_{\pi}\right|}\leq\sqrt{|D_{\pi}|n} \\
\left\| \left((\mathbf{1}_{N'}\mathbf{1}_{N'}^\top-I_{N'})\odot\Pi_Y\right)^\top\!\Pi'_Y \right\|_{F}&= \left\| (\mathbf{1}_{N'}\mathbf{1}_{N'}^\top-I_{N'})\odot\Pi_Y\right\|_{F}=\sqrt{ \left| D_{\pi}\right|d}
\end{align*}
Applying Hanson-Wright inequality (\Cref{lem:HW}) with, for all $t \in [n]$, $x_t=Ct\sqrt{n\log n}$ for $\cbp$ and $x_t=Ct\sqrt{\max\{d,\log n\}\log n}$ for $\cyp$, and taking a union bound over the pairs $(\pi,\pi')$ with $\pi \in\mS_{n,t}$ and $\pi'\in\mS_n$, we obtain 
\begin{align*}
& \mathbb{P}\left(\left(\mathcal{H}^b_2\right)^c\right) \leq \sum_{t=1}^{n} \sum_{\substack{\pi \in \mS_{n,t},\pi'\in\mS_n\\ |D_{\pi'}|\leq 2t}} \dP(\left|\cbp-\mathbb{E}\left[\cbp\right]\right| > Ct\sqrt{n\log\left(n\right)}) \\
& \qquad \qquad  \qquad + \sum_{t=1}^{n} \sum_{\pi \in \mS_{n,t},\pi'\in\mS_n}  \dP(\left|\cyp-\mathbb{E}\left[\cyp\right]\right| > Ct\sqrt{\max\{d,\log\left(n\right)\}\log\left(n\right)}) \\
& \stackrel{(a)}{\leq} \sum_{t=1}^{n} 2 |\mS_{n,t}|2n^{2t} \exp\!\left(-\frac{c\, C^2t^2 n \log n}{t n+ Ct \sqrt{n \log n}}\right) + \sum_{t=1}^{n} 2 |\mS_{n,t}|n^{t} \exp\!\left(-\frac{c\, C^2t^2 \max\{d,\log\left(n\right)\} \log(n)}{t d+ Ct \sqrt{\max\{d,\log\left(n\right)\} \log (n)}}\right) \\
 & \overset{(b)}{\leq} \sum_{t=1}^{n} \exp\left( 3t \log (n)-C't \log (n)  \right) + \sum_{t=1}^{n} \exp\left( 2t \log (n)-C''t \log (n)  \right)\\
 & \leq \sum_{t=1}^{n} \exp(-C'''t\log(n)) = o(1),
\end{align*} 
where in $(a)$ we use the following counting argument. For the graph part (first sum), a permutation $\pi'$ with $|D_{\pi'}|=s$ is determined by the set $D_{\pi'}$ and a permutation of it, so there are at most $\binom{n}{s}s!\leq n^s$ of them; hence there are at most $\sum_{s\leq 2t}n^s\leq 2n^{2t}$ permutations $\pi'$ with $|D_{\pi'}|\leq 2t$. For the feature part (second sum), the sum defining $\cyp$ is restricted to indices with $\pi(i)\neq i$, so $\cyp$ depends on $\pi'$ only through its restriction $\pi'|_{D_\pi}$, an injective map from $D_\pi$ to $[n]$; hence the union over $\pi'\in\mS_n$ involves at most $n^t$ distinct events. In $(b)$, we use the assumption $d=\omega(\log n)$ to bound the second
sum.
This completes the proof.


\end{proof}

Based on \Cref{lem:sub2}, we can now complete the proof of \Cref{lem:main2}. We have 
\begin{align*}  
\mathbb{E}\left[e^{\frac{\rho}{1-\rho^2} V_G(\pi)}\mathbf{1}_{\mathcal{H}_2}\right]  
&=\mathbb{E}\left[\mathbb{E}\left[e^{\frac{\rho}{1-\rho^2} V_G(\pi)}\mathbf{1}_{\mathcal{H}_2}\,\Big|\,\mathbf{B},\mathbf{Y}\right]\right]\\  
&\stackrel{(a)}{=}\mathbb{E}\left[\exp\left(\frac{\rho^2}{1-\rho^2}\sum_{\substack{1\leq i< j \leq n\\ \pi(i,j)\neq\{i,j\}}}B_{\pi(i,j)}B_{i,j}+\frac{\rho^2}{2(1-\rho^2)} \sbp\right)\mathbf{1}_{\mathcal{H}_2}\right]\\  
&\stackrel{(b)}{\leq}\exp\!\left(\frac{\rho^2}{2(1-\rho^2)}\left(\left|D^E_{\pi}\right|+C\left|D_{\pi}\right|\sqrt{n\log\!\left(n\right)}\right)\right),  
\end{align*} 
where in $(a)$ we use that $\mathbf{1}_{\mathcal{H}_2}$ is $(\mathbf{B},\mathbf{Y})$-measurable and that, conditionally on $(\mathbf{B},\mathbf{Y})$, $V_{G}(\pi)$ has the Gaussian law given above (recall that $(\bA,\bB)$ is independent of $(\bX,\bY)$), and in $(b)$ we apply \Cref{lem:sub2} together with the facts that $\mathbb{E}\left[\sum_{\substack{1\leq i< j \leq n\\ \pi(i,j)\neq\{i,j\}}}B_{\pi(i,j)}B_{i,j}\right]=\mathbb{E}\left[\mathrm{Cov}_{\scaleto{\mathbf{B}}{3.9pt}}(\pi,\textrm{id})\right]=0$ and $\mathbb{E}\!\left[\sbp\right]=\left|D^E_{\pi}\right|$. The second part of \Cref{lem:main2} is proved in the exact same way.
\end{proof}

\section{Proof of Lemma~\ref{lem:main3}}\label{app:proof:lem:main3}
\begin{proof}[Proof of \Cref{lem:main3}]
\, 
\paragraph{Proof of \Cref{lem:main3} -- Step 1: concentrating $\log(Z)$ via conditional Lipschitz concentration.}
The first step of the proof is to derive a concentration bound for $\log(Z)$ conditionally on $(\bB,\bY)$, which is given in the following.
\begin{lemma}
\label{lem:sub1_main3}
Suppose $d=\omega(\log n)$ and
\(
\frac{\rho^2}{1-\rho^2}n+\frac{2\eta^2}{1-\eta^2}d \,\leq\, 4(1-\varepsilon)\log n
\). For any sequence $a_n=o\left(n\log n\right)$, on the event $\cH_2$ of Lemma~\ref{lem:sub2},
\[
\left|\log(Z)-\dE[\log(Z)\mid\bB,\bY]\right|\leq a_n
\]
with conditional probability given $(\bB,\bY)$ at least $1-\exp\left(-b_n\right)$ where $b_n=\Theta\left(\frac{a^2_n}{n\log n}\right)$. 
\end{lemma}
\begin{proof}[Proof of \Cref{lem:sub1_main3}]
The proof is based on the following standard theorem on concentration of Lipschitz functions of Gaussian random variables.
\begin{theorem}[{\citealp[Theorem~3.16]{wainwright2019high}}]
\label{Th:sub_sub1}
Assume that $X$ is distributed according to a strongly log-concave distribution on $\dR^D$ with parameter $\gamma>0$. Then, for any function $f:\mathbb{R}^D\to\mathbb{R}$ that is $L$-Lipschitz with respect to the Euclidean norm, we have
\[
\dP\!\left(\left|f(X)-\dE[f(X)]\right|\geq t\right)
\leq 2\exp\!\left(-\frac{\gamma t^2}{4L^2}\right).
\]
\end{theorem}
In our setting, we work conditionally on $(\mathbf{B},\mathbf{Y})$. Recall that $A_{i,j}=\rho B_{i,j}+\sqrt{1-\rho^2}\xi_{i,j}$ and $X_{i,k}=\eta Y_{i,k}+\sqrt{1-\eta^2}\xi'_{i,k}$, where the noise variables $\xi_{i,j}$ and $\xi'_{i,k}$ are independent standard Gaussian random variables, independent of $(\mathbf{B},\mathbf{Y})$. Collecting all noise variables into a single vector $\boldsymbol{\xi}$, conditionally on $(\mathbf{B},\mathbf{Y})$, $\boldsymbol{\xi}$ is a standard Gaussian random vector, whose distribution is strongly log-concave with parameter $\gamma=1$, and $\log(Z)$ is a function of $\boldsymbol{\xi}$.

Next, we compute the gradient of $\log(Z)$ with respect to the vector $\boldsymbol{\xi}$. Since $\partial_{\xi_{i,j}}\log(Z)=\sqrt{1-\rho^2}\,\partial_{A_{i,j}}\log(Z)$ and $\partial_{\xi'_{i,k}}\log(Z)=\sqrt{1-\eta^2}\,\partial_{X_{i,k}}\log(Z)$, it suffices to compute the derivatives with respect to $A_{i,j}$ and $X_{i,k}$. Throughout the computation, we recall the expression of $V(\pi)$:

{
\begin{multline*}
V(\pi) = \frac{\rho}{1-\rho^2} \left[\sum\limits_{\substack{1\leq i<j \leq n\\ \pi(i,j)\neq \{i,j\}}}B_{i,j}A_{i,j} - \sum\limits_{\substack{1\leq i<j \leq n\\ \pi(i,j)\neq\{i,j\}}}B_{\pi(i,j)}A_{i,j} \right] \\ + \frac{\eta}{1-\eta^2} \left [\sum\limits_{\substack{1\leq i \leq n, \ 1\leq k\leq d\\ \pi(i)\neq i}}Y_{i,k}X_{i,k} - \sum\limits_{\substack{1\leq i \leq n, \ 1\leq k\leq d\\ \pi(i)\neq i}}Y_{\pi(i),k}X_{i,k}\right],
\end{multline*} } and that 
\[
Z=\sum_{\pi\in\mS_n}\exp(-V(\pi)).
\] 
First, consider the derivatives with respect to the variables $A_{i,j}$. We have
\begin{align*}
\partial_{A_{i,j}}\log(Z)
&=\frac{\rho}{1-\rho^2}\sum_{\pi:\pi(i,j)\neq\{i,j\}}(B_{\pi(i,j)}-B_{i,j})\frac{e^{-V(\pi)}}{Z}\\
&\stackrel{(a)}{=}-\frac{\rho}{1-\rho^2}B_{i,j}
+\frac{\rho}{1-\rho^2}\sum_{1\leq i'< j'\leq n}B_{i',j'}
\left(\sum_{\pi:\pi(i,j)=\{i',j'\}}\frac{e^{-V(\pi)}}{Z}\right),
\end{align*} 
where in $(a)$ we extend the sum to all $\pi\in\mS_n$ by adding the zero terms corresponding to permutations satisfying $\pi(i,j)=\{i,j\}$, for which $B_{\pi(i,j)}=B_{i,j}$.
Similarly, for the variables $X_{i,k}$, we obtain
\begin{align*}
\partial_{X_{i,k}}\log(Z)
&=-\frac{\eta}{1-\eta^2}Y_{i,k}
+\frac{\eta}{1-\eta^2}\sum_{1\leq i'\leq n}Y_{i',k}
\left(\sum_{\pi:\pi(i)=i'}\frac{e^{-V(\pi)}}{Z}\right).
\end{align*}

Now, if we denote the $(\xi_{i,j})$ component of $\boldsymbol{\xi}$ by $\boldsymbol{\xi}_G$ (graph noise) and the $(\xi'_{i,k})$ component by $\boldsymbol{\xi}_F$ (feature noise), and by $b$ and $y$ the vectorizations of $\mathbf{B}$ and $\mathbf{Y}$, we can express the gradient of $\log(Z)$ as
\begin{align*}
\nabla_{\boldsymbol{\xi}_G}\log(Z)&=\frac{\rho}{\sqrt{1-\rho^2}}(-I+S_G)\,b\\
\nabla_{\boldsymbol{\xi}_F}\log(Z)&=\frac{\eta}{\sqrt{1-\eta^2}}(-I+S_F)\,y
\end{align*}
where $S_G$ and $S_F$ are doubly stochastic matrices. Therefore, we obtain
\begin{align}
\left\| \nabla_{\boldsymbol{\xi}}\log(Z) \right\|^2_2&=\frac{\rho^2}{1-\rho^2}\left\| (-I+S_G)\,b \right\|^2_2+\frac{\eta^2}{1-\eta^2}\left\| (-I+S_F)\,y \right\|^2_2 \nonumber\\
&\leq \frac{2\rho^2}{1-\rho^2}\left(\left\| b \right\|^2_2+\left\| S_G \ b \right\|^2_2\right)+\frac{2\eta^2}{1-\eta^2}\left(\left\| y \right\|^2_2+\left\| S_F \ y \right\|^2_2\right)\nonumber\\
&\stackrel{(a)}{\leq} \frac{4\rho^2}{1-\rho^2}\left\| b \right\|^2_2+\frac{4\eta^2}{1-\eta^2}\left\|y \right\|^2_2
\label{eq:bound_grad_with_Ts}
\end{align}
where the first inequality uses $\|u+v\|_2^2\leq 2\|u\|_2^2+2\|v\|_2^2$, and in $(a)$ we used the fact that for any doubly stochastic matrix $A$ and arbitrary vector $q$, we have $\left\|Aq\right\|_2\leq\left\|q\right\|_2$.

Note that the right-hand side of \eqref{eq:bound_grad_with_Ts} depends only on $(\mathbf{B},\mathbf{Y})$ and not on $\boldsymbol{\xi}$. Moreover, it is controlled on the event $\cH_2$ of Lemma~\ref{lem:sub2}: applying $\cH_2^a$ with $t=n$ to an $n$-cycle $\pi$, for which $\sbp=\|b\|_2^2$ and $\syp=\|y\|_2^2$, we have $\|b\|_2^2\leq \frac{n(n-1)}{2}+Cn\sqrt{n}$ and $\|y\|_2^2\leq nd+Cn\sqrt{d}$.
Then, by \eqref{eq:bound_grad_with_Ts}, on the event $\cH_2$, conditionally on $(\mathbf{B},\mathbf{Y})$, the function $\boldsymbol{\xi}\mapsto\log(Z)$ is $L-$Lipschitz on the whole space, where $L$ satisfies
\begin{align*}
L^2&=\frac{2\rho^2n^2}{1-\rho^2}(1+o(1))
+\frac{4\eta^2nd}{1-\eta^2}(1+o(1))\\
&\stackrel{(a)}{\leq} 8n\log(n)(1+o(1)) \leq Cn\log(n) \, .
\end{align*}
In $(a)$, we use the assumption of the lemma, $\frac{\rho^2 n}{1-\rho^2}+\frac{2\eta^2 d}{1-\eta^2}\leq 4\log n$.
Finally, we apply \Cref{Th:sub_sub1} with $\gamma=1$ to $f(\boldsymbol{\xi})=\log(Z(\boldsymbol{\xi}))$, conditionally on $(\mathbf{B},\mathbf{Y})$, and obtain, recalling $a_n=o(n \log n)$, on the event $\cH_2$,
\begin{align*}
\dP\!\left(\left|\log(Z)-\dE[\log(Z)\mid\mathbf{B},\mathbf{Y}]\right|\geq a_n\,\middle|\,\mathbf{B},\mathbf{Y}\right)
&\leq 2\exp\!\left(-\frac{a_n^2}{4Cn\log n}\right)\\
&= \exp\left(-\Theta\left(\frac{a^2_n}{n\log n}\right)\right),
\end{align*}
which completes the proof of \Cref{lem:sub1_main3}.
\end{proof}

\paragraph{Proof of \Cref{lem:main3} -- Step 2: lower bound for $\dE[\log(Z)\mid\bB,\bY]$.}
We now proceed to the second step in the proof of Lemma~\ref{lem:main3}, which involves deriving a lower bound for $\dE[\log(Z)\mid\bB,\bY]$ on the event $\cH_2$. 

Recall that $\tilde\varepsilon$ is defined by
\[
    \frac{\rho^2n}{1-\rho^2}
    +\frac{2\eta^2d}{1-\eta^2}
    =4(1-\tilde\varepsilon)\log n,
\]
and satisfies $\tilde\varepsilon\geq\varepsilon$.
Since Lemma~\ref{lem:main1} already establishes that $V^*_F(\pi)$ and $V^*_G(\pi)$ concentrate well, it remains to show that $V_F(\pi)$ and $V_G(\pi)$ can attain large values for some permutations $\pi$. Hence, we lower bound $Z$ by retaining only the terms $e^{-V(\pi)}$ corresponding to the well-chosen permutations defined below.
\begin{equation}
\label{def:Tn}
\mathcal{T}_n := \left\{\pi \in \mathcal{S}_{n,n} : \sum_{1\leq i< j\leq n}B_{\pi(i,j)}\xi_{i,j}\geq \frac{\rho n^2}{2\sqrt{1-\rho^2}}, \sum_{\substack{1\leq i \leq n\\1\leq k\leq d }}Y_{\pi(i),k}\xi'_{i,k} \,\geq\, \frac{\eta  n d}{\sqrt{1-\eta^2}}\right\},
\end{equation}
where \(\xi_{i,j}\) and \(\xi'_{i,k}\) denote the noise components in the decompositions \(A_{i,j} = \rho B_{i,j} + \sqrt{1-\rho^2}\,\xi_{i,j}\) and \(X_{i,k} = \eta Y_{i,k} + \sqrt{1-\eta^2}\,\xi'_{i,k}\), respectively.

The key quantity of interest is the size of this set, $|\mathcal{T}_n|$. We will prove that this cardinality is greater than $\exp\!\left(\tilde\varepsilon n \log n (1-o(1))\right)$ with probability at least $\exp(-o(n \log n))$. The cardinality of \(\mathcal{T}_n\) is difficult to control directly because of the feature signal constraints: two very different permutations may satisfy the feature constraint thanks to a single very large $Y_{i,k}$ term. This makes the variance of \(|\mathcal{T}_n|\) too large to get a satisfactory result via the second-moment method. We therefore consider a smaller set in which each vertex must satisfy its own feature constraint, whose cardinality is easier to control.
Let
\begin{align}
\label{def:Mn}
    \mathcal{M}_n\vcentcolon=\left\{\pi\in\mS_{n,n}:\sum_{1\leq i< j\leq n}B_{\pi(i,j)}\xi_{i,j}\geq \frac{\rho n^2}{2\sqrt{1-\rho^2}}, \ \forall i\in [n]: \sum_{\substack{1\leq k\leq d }}Y_{\pi(i),k}\xi'_{i,k} \,\geq\, \frac{\eta d}{\sqrt{1-\eta^2}}\right\}
\end{align}
Clearly, $\mathcal{M}_n\subseteq\mathcal{T}_n$. Hence, if we prove that $|\mathcal{M}_n| \geq \exp\!\left(\tilde\varepsilon n \log n (1-o(1))\right)$ with probability at least $\exp(-o(n \log n))$ then a similar bound holds for $|\mathcal{T}_n|$. 

In the sequel, we will denote $M :=|\mathcal{M}_n| $. 

Our approach to lower bound $M$ is to use a conditional second moment method. Traditional second moment method builds upon Paley–Zygmund inequality, whose general form is
\begin{equation}
\label{pz}
\dP\!\left(M \geq c\,\dE[M]\right) \;\geq\; (1-c)^2 \frac{\dE[M]^2}{\dE[M^2]},
\end{equation}
for any $0<c<1$. 

In our setting however, the conditional second moment method will be used. Instead of working directly with $M$, we apply the Paley–Zygmund inequality to \(M\), conditionally on \(\mathbf{B},\mathbf{Y}\), which yields
\begin{equation}\label{cpz}
\dP\!\left(M \geq c\,\dE[M \mid \mathbf{B},\mathbf{Y}] \,\Big|\, \mathbf{B},\mathbf{Y}\right)  
\;\geq\; (1-c)^2 \frac{\dE[M \mid \mathbf{B},\mathbf{Y}]^2}{\dE[M^2 \mid \mathbf{B},\mathbf{Y}]}.
\end{equation}
Taking expectation with respect to $\mathbf{B},\mathbf{Y}$ gives
\begin{equation} \label{spz}
\dP\!\left(M \geq c\,\dE[M \mid \mathbf{B},\mathbf{Y}]\right)  
\;\geq\; (1-c)^2 \, \dE\!\left[ \frac{\dE[M \mid \mathbf{B},\mathbf{Y}]^2}{\dE[M^2 \mid \mathbf{B},\mathbf{Y}]} \right].
\end{equation}

The conditional second-moment method is better suited to our purpose: as we show below, \(\mathbb{E}[M\mid\mathbf B,\mathbf Y]\) satisfies the desired lower bound, while the right-hand side of \eqref{spz} is at least as large as that of \eqref{pz}. Indeed, we have
\begin{equation*} \dE\!\left[ \frac{\dE[M|\mathbf{B},\mathbf{Y}]^2}{\dE[M^2|\mathbf{B},\mathbf{Y}]} \right] \dE[M^2] = \dE\!\left[ \frac{\dE[M|\mathbf{B},\mathbf{Y}]^2}{\dE[M^2|\mathbf{B},\mathbf{Y}]}\right] \dE\left[\dE[M^2|\mathbf{B},\mathbf{Y}] \right] \stackrel{(a)}{\geq} \dE\!\left[ \dE[M|\mathbf{B},\mathbf{Y}] \right]^2 = \dE[M]^2,
\end{equation*}
where $(a)$ follows from the Cauchy–Schwarz inequality. 

Applying the conditional second moment method yields the following.
\begin{lemma}
\label{lem:subsubsub2}
If $d=\omega(\log n)$ and \(
\frac{\rho^2}{1-\rho^2}n+\frac{2\eta^2}{1-\eta^2}d \,\leq\, 4(1-\varepsilon)\log n,
\)
then
\begin{align*}
\dE[\log(Z)\mid\bB,\bY]\geq\tilde\varepsilon n\log(n)(1-o(1))\quad\text{on the event }\cH_2,
\end{align*}
where the $o(1)$ term is deterministic and $\tilde\varepsilon\geq\varepsilon$ is defined by
\[
    \frac{\rho^2n}{1-\rho^2}
    +\frac{2\eta^2d}{1-\eta^2}
    =4(1-\tilde\varepsilon)\log n.
\]
\end{lemma}
\begin{proof}[Proof of \Cref{lem:subsubsub2}]
\, 
\paragraph{Proof of \Cref{lem:subsubsub2} -- Step 1: from $M$ to $\log(Z)$.}
Assume that we have established that, on the event $\cH_2$,
\begin{equation}\label{eq:lem:subsubsub2:bound_M}
\dP\Big(M\geq \exp\!\big(\tilde\varepsilon n \log n (1-o(1))\big)\,\Big|\,\bB,\bY\Big)\geq e^{-s_n}
\end{equation}
for some deterministic sequence $s_n=o(n\log n)$ with $s_n\geq 1$.
In this first step, we show how the lower bound on \(M\) in \eqref{eq:lem:subsubsub2:bound_M} yields a lower bound on \(\dE[\log(Z)\mid\bB,\bY]\). The remaining steps will then establish this lower bound \eqref{eq:lem:subsubsub2:bound_M}.
Recall that $A_{i,j}=\rho B_{i,j}+\sqrt{1-\rho^2}\xi_{i,j}$ and
$X_{i,k}=\eta Y_{i,k}+\sqrt{1-\eta^2}\xi'_{i,k}$. For
$\pi\in \cT_n$, we have
\begin{align*}
V_G^*(\pi)-V_G(\pi)
&=\sum_{1\leq i< j\leq n}A_{i,j}B_{i,j}
-\rho\sum_{1\leq i< j\leq n}B_{\pi(i,j)}B_{i,j}
-\sqrt{1-\rho^2}\sum_{1\leq i< j\leq n}B_{\pi(i,j)}\xi_{i,j},\\
V_F^*(\pi)-V_F(\pi)
&=\sum_{i\in[n],k\in[d]}X_{i,k}Y_{i,k}
-\eta\sum_{i\in[n],k\in[d]}Y_{\pi(i),k}Y_{i,k}
-\sqrt{1-\eta^2}\sum_{i\in[n],k\in[d]}Y_{\pi(i),k}\xi'_{i,k}.
\end{align*}
We control the first sums of the RHS on an event $\cA_n$ defined below, the second sums of the RHS under $\mathcal{H}_2$, and the noise sums
by the definition of $\mathcal{T}_n$. For the first sums, denoting by $b$ and $y$ the vectorizations of $\bB$ and $\bY$, we have
\[
\frac{\rho}{1-\rho^2}\sum_{1\leq i<j\leq n}A_{i,j}B_{i,j}+\frac{\eta}{1-\eta^2}\sum_{i\in[n],\,k\in[d]}X_{i,k}Y_{i,k}
=\frac{\rho^2}{1-\rho^2}\|b\|_2^2+\frac{\eta^2}{1-\eta^2}\|y\|_2^2+G_n,
\]
where $G_n:=\frac{\rho}{\sqrt{1-\rho^2}}\sum_{1\leq i<j\leq n}B_{i,j}\xi_{i,j}+\frac{\eta}{\sqrt{1-\eta^2}}\sum_{i\in[n],\,k\in[d]}Y_{i,k}\xi'_{i,k}$. On the event $\cH_2$, we have $\|b\|_2^2\leq \frac{n(n-1)}{2}+Cn\sqrt n$ and $\|y\|_2^2\leq nd+Cn\sqrt d$ (see the proof of \Cref{lem:sub1_main3}), so the first two terms are at most $\frac{\rho^2n^2}{2(1-\rho^2)}+\frac{\eta^2nd}{1-\eta^2}+o(n\log n)$. Moreover, conditionally on $(\bB,\bY)$, $G_n$ is a centered Gaussian random variable with variance $\frac{\rho^2}{1-\rho^2}\|b\|_2^2+\frac{\eta^2}{1-\eta^2}\|y\|_2^2\leq 2n\log n(1+o(1))$ on $\cH_2$. Let $h_n:=\sqrt{10\,s_n\,n\log n}=o(n\log n)$ and $\cA_n:=\{G_n\leq h_n\}$. Then, on the event $\cH_2$, $\dP(\cA_n^c\mid\bB,\bY)\leq e^{-h_n^2/(5n\log n)}= e^{-2s_n}$.

For the second sums, on the event $\mathcal H_2$, for each
$\pi\in\mathcal S_{n,n}$, we have
\begin{align*}
\frac{\rho^2}{1-\rho^2}\sum_{1\leq i<j\leq n}B_{\pi(i,j)}B_{i,j}
+\frac{\eta^2}{1-\eta^2}\sum_{i\in[n],\,k\in[d]}Y_{\pi(i),k}Y_{i,k}
&\geq \frac{\rho^2}{1-\rho^2}\mathrm{Cov}_{\scaleto{\mathbf{B}}{3.9pt}}(\pi,\textrm{id})+\frac{\eta^2}{1-\eta^2}\mathrm{Cov}_{\scaleto{\mathbf{Y}}{3.9pt}}(\pi,\textrm{id})\\
&\stackrel{(a)}{\geq}-\frac{\rho^2}{1-\rho^2}Cn\sqrt{n\log n}-\frac{\eta^2}{1-\eta^2}Cn\sqrt{d\log n}\\
&\stackrel{(b)}{\geq}-4C\sqrt{n\left(\log n\right)^{3}}-4Cn\sqrt{\frac{\left(\log n\right)^{3}}{d}}\\
&\stackrel{(c)}{=}o(n\log n),
\end{align*}
where $(a)$ follows from the deviation bounds on
$\mathcal H_2$, $\mathrm{Cov}_{\scaleto{\mathbf{B}}{3.9pt}}(\pi,\textrm{id})\geq \mathbb{E}\left[\mathrm{Cov}_{\scaleto{\mathbf{B}}{3.9pt}}(\pi,\textrm{id})\right]-Cn\sqrt{n\log n}$ and $\mathrm{Cov}_{\scaleto{\mathbf{Y}}{3.9pt}}(\pi,\textrm{id})\geq\mathbb{E}\left[\mathrm{Cov}_{\scaleto{\mathbf{Y}}{3.9pt}}(\pi,\textrm{id})\right]-Cn\sqrt{d\log n}$, and the fact that $\mathbb{E}\left[\mathrm{Cov}_{\scaleto{\mathbf{B}}{3.9pt}}(\pi,\textrm{id})\right]=\mathbb{E}\left[\mathrm{Cov}_{\scaleto{\mathbf{Y}}{3.9pt}}(\pi,\textrm{id})\right]=0$ for $\pi\in\mS_{n,n}$; $(b)$ follows from $\frac{\rho^2n}{1-\rho^2}\leq 4\log n$ and $\frac{\eta^2d}{1-\eta^2}\leq 4\log n$; $(c)$ follows from $d=\omega(\log n)$. 
For the third sums, by the definition of $\mathcal{T}_n$ we have
\begin{align*}
    \frac{\rho}{\sqrt{1-\rho^2}}\sum_{1\leq i< j\leq n}B_{\pi(i,j)}\xi_{i,j}+\frac{\eta}{\sqrt{1-\eta^2}}\sum_{i\in[n],k\in[d]}Y_{\pi(i),k}\xi'_{i,k}\geq \frac{\rho^2n^2}{2(1-\rho^2)}+\frac{\eta^2dn}{1-\eta^2}.
\end{align*}

Now, we have
\begin{align*}
&Z\mathbf{1}_{\cA_n\cap \cH_2}
\geq\sum_{\pi\in\mathcal{T}_n}\exp\left(-\frac{\rho}{1-\rho^2}(V_G^*(\pi)-V_G(\pi))-\frac{\eta}{1-\eta^2}(V_F^*(\pi)-V_F(\pi))\right)\mathbf{1}_{\cA_n\cap \cH_2}\\
&\stackrel{}{\geq}|\mathcal{T}_n|\mathbf{1}_{\cA_n\cap \cH_2}\exp\left(-\frac{\rho^2 n^2}{2(1-\rho^2)}-\frac{\eta^2nd}{1-\eta^2}-h_n-o(n\log n)+\frac{\rho^2n^2}{2(1-\rho^2)}+\frac{\eta^2dn}{1-\eta^2}\right)\\
&\stackrel{(a)}{\geq}M\mathbf{1}_{\cA_n\cap \cH_2}\exp\left(-o(n\log n)\right)\\
&\stackrel{(b)}{\geq}\exp\left(\tilde\varepsilon n\log(n)(1-o(1))\right),
\end{align*}
where $(a)$ follows since $M\leq|\mathcal{T}_n|$ and $h_n=o(n\log n)$.

Now, according to \eqref{eq:lem:subsubsub2:bound_M}, and since $\dP(\cA_n^c\mid\bB,\bY)\leq e^{-2s_n}$ on $\cH_2$, on the event $\cH_2$ we have that $M\geq \exp\!\big(\tilde\varepsilon n \log n (1-o(1))\big)$ and $\cA_n$ hold simultaneously with conditional probability given $(\bB,\bY)$ at least $e^{-s_n}-e^{-2s_n}\geq\frac12e^{-s_n}$, and so with the same probability, the above inequality implies that $\log(Z)\geq \tilde\varepsilon n \log n (1-o(1))$.

Therefore, combining this probabilistic lower bound
with the concentration inequality in Lemma~\ref{lem:sub1_main3} applied with $a_n:=\sqrt{K's_n n\log n}=o(n\log n)$, where $K'>0$ is a sufficiently large constant so that $\exp(-b_n)<\frac12e^{-s_n}$ (recall $s_n\geq1$), we obtain, on the event $\cH_2$,
\begin{align*}
\dP\left(\log(Z)\geq\dE[\log(Z)\mid\bB,\bY]+a_n\,\middle|\,\bB,\bY\right)
&\leq \exp\left(-b_n\right)<\tfrac12e^{-s_n}\\
&\leq \dP\left(\log(Z)\geq \tilde\varepsilon n \log n (1-o(1))\,\middle|\,\bB,\bY\right).
\end{align*}
The strict inequality forces $\dE[\log(Z)\mid\bB,\bY]+a_n>\tilde\varepsilon n \log n (1-o(1))$, that is, \(\dE[\log(Z)\mid\bB,\bY] \geq \tilde\varepsilon n \log n (1-o(1))\) on the event $\cH_2$.

We now proceed to establish the probabilistic lower bound on \(M\mathbf{1}_{\cH_2}\) from \eqref{eq:lem:subsubsub2:bound_M} in order to complete the proof. As mentioned earlier, we use the second-moment method via the Paley–Zygmund inequality. We therefore need to control the first and second moments of \(M\mathbf{1}_{\cH_2}\), which we do in the following steps.

\paragraph{Proof of \Cref{lem:subsubsub2} -- Step 2.1: lower bound for the first moment of \(M\mathbf{1}_{\cH_2}\)}
The first step is to derive a bound on the conditional expectation \(\dE[M\mathbf{1}_{\cH_2}\mid \bby]\). Noting that $\sum_{1\leq i< j\leq n}B_{\pi(i,j)}\xi_{i,j}$ conditioned on $\bB$ is a centered normal distribution with variance $\sbp$. Similarly, $\sum_{\substack{1\leq k\leq d }}Y_{\pi(i),k}\xi'_{i,k}$ for each $i\in [n]$ conditioned on $\bY$ is a centered normal distribution with variance $\sum_{\substack{1\leq k\leq d }}Y^2_{\pi(i),k}$. Moreover, all these normal distributions are independent of each other. Thus, since $\cH_2$ is $(\bB,\bY)$-measurable, on the event $\cH_2$ we have
\begin{align*}
\dE&[M\mathbf{1}_{\cH_2}\mid\bby]
\\&=\sum_{\pi\in\mathcal{S}_{n,n}}\dP\left(\mathcal{N}\left(0,\sbp\right)\geq \frac{\rho n^2}{2\sqrt{1-\rho^2}},\cH_2\right)\prod_{i\in[n]}^{}\dP\left(\mathcal{N}\left(0,\sum_{\substack{1\leq j\leq d }}Y^2_{\pi(i),j}\right) \,\geq\, \frac{\eta d}{\sqrt{1-\eta^2}},\cH_2\right)\\
&\stackrel{(a)}{\geq}|\mathcal{S}_{n,n}|\ \dP\left(\mathcal{N}\!\left(0,1\right)\geq \frac{1}{\sqrt{\frac{n^2}{2}(1-o(1))}}\frac{\rho n^2}{2\sqrt{1-\rho^2}}\right)\left(\dP\left(\mathcal{N}\!\left(0,1\right)\geq \frac{1}{\sqrt{d(1-o(1))}}\frac{\eta d}{\sqrt{1-\eta^2}}\right)\right)^n\\
&\stackrel{(b)}{\geq}|\mathcal{S}_{n,n}|\ \frac{1}{\sqrt{2\pi}\left(1+\frac{\rho n}{\sqrt{2(1-\rho^2)}}\right)}\exp\!\left(-\frac{\rho^2n^2(1+o(1))}{4(1-\rho^2)}\right)\left(\frac{1}{\sqrt{2\pi}\left(1+\frac{\eta \sqrt{d}}{\sqrt{1-\eta^2}}\right)}\exp\!\left(-\frac{\eta^2d(1+o(1))}{2(1-\eta^2)}\right)\right)^n\\
&\geq|\mathcal{S}_{n,n}|\exp\left(-\frac{n}{4}\left(\frac{\rho^2n}{1-\rho^2}+\frac{2\eta^2d}{1-\eta^2}\right)(1+o(1))-o(n\log n)\right)\\
&\stackrel{(c)}{=}\exp\!\left(\tilde\varepsilon n\log n(1-o(1))\right),
\end{align*}
where $(a)$ follows by applying Lemma~\ref{lem:sub2} together with the fact that $d=\omega(\log(n))$; $(b)$ follows from the fact that $\dP(\mathcal{N}(0,1)\geq t)\geq \tfrac{1}{\sqrt{2\pi}(1+t)}\exp(-t^2/2)$ for $t\geq 0$; and $(c)$ uses the impossibility assumption, \(
\frac{\rho^2}{1-\rho^2}n+\frac{2\eta^2}{1-\eta^2}d = 4(1-\tilde\varepsilon)\log n
\), as well as the fact that $|\mathcal{S}_{n,n}|=\exp\!\big(n\log n(1-o(1))\big)$.

\paragraph{Proof of \Cref{lem:subsubsub2} -- Step 2.2: upper bound for the second moment}

To estimate the probabilities of events defining $\mathcal{M}_n$, we introduce
normalized Gaussian noise sums.
For $\pi\in S_{n,n}$ and nonempty $I\subseteq[n]$, define
\[
    W_G(\pi):=
    \frac{
        \sum_{1\leq i<j\leq n}B_{\pi(i,j)}\xi_{i,j}
    }{\sqrt{n^2/2}},
    \qquad
    W_F^I(\pi):=
    \frac{
        \sum_{i\in I}\sum_{k\in[d]}Y_{\pi(i),k}\xi'_{i,k}
    }{\sqrt{|I|d}}.
\]
Conditional on $B,Y$, these variables are centered Gaussian,
with
\[
    W_G(\pi)\mid \bB
    \sim\mathcal N(0,1+g_n^{(\pi)}),
    \qquad
    W_F^I(\pi)\mid \bY
    \sim\mathcal N(0,1+f_n^{I,(\pi)}).
\]

where, on the event $\cH_2$, we have $g^{(\pi)}_n=O\left(\frac{1}{\sqrt{n}}\right)$ and $f^{I,(\pi)}_n=O\left(\sqrt{\frac{\log\left(\frac{en}{|I|}\right)}{d}}\right)$. Moreover, for $\pi,\pi'\in\mathcal{S}_{n,n}$, we have
\begin{align}
\dE[W_G(\pi)W_G(\pi')\mathbf{1}_{\mathcal{H}_2}\mid \bB] 
&=
\frac{\sum_{1\leq i< j\leq n}B_{\pi(i,j)}B_{\pi'(i,j)}}{n^2/2}\mathbf{1}_{\mathcal{H}_2}=\frac{\cbp}{n^2/2}\mathbf{1}_{\mathcal{H}_2}\nonumber\\
&\stackrel{(a)}{\leq} \frac{\,\lvert F^E_{\pi\cap\pi'}\rvert +Cn\sqrt{n\log(n)}}{n^2/2},\nonumber\\
&\stackrel{(b)}{\leq}\left(\frac{|F_{\pi\cap\pi'}|}{n}\right)^2+O\left(\sqrt{\frac{\log n}{n}}\right), \label{corg}
\end{align}
where
\[
F_{\pi\cap\pi'} \vcentcolon= \{i\in[n]: \pi(i)=\pi'(i)\}, 
\qquad 
F^E_{\pi\cap\pi'} \vcentcolon= \{(i,j):1\leq i<j\leq n,\ \pi(i,j)=\pi'(i,j)\}.
\]  and $(a)$ follows from Lemma~\ref{lem:sub2} and $(b)$ from $|F^E_{\pi\cap\pi'}|=\tfrac{|F_{\pi\cap\pi'}|^2}{2}+O(n)$. 
Similarly for the feature signal, for each $i\in [n]\setminus F_{\pi\cap\pi'}$ we have
\begin{equation}
\label{eqcor}
\dE[W^{\{i\}}_F(\pi)W^{\{i\}}_F(\pi')\mathbf{1}_{\mathcal{H}_2}\mid \bY]\ \leq\ \frac{\sum_{ k\in[d]}Y_{\pi(i),k}Y_{\pi'(i),k}}{d}\mathbf{1}_{\mathcal{H}_2}\leq \frac{0+C\sqrt{d\log n}}{d}\leq O\left(\sqrt{\frac{\log n}{d}}\right).
\end{equation}

By the conditional Paley–Zygmund inequality~\eqref{cpz}, applied on the event $\cH_2$, in order to prove \eqref{eq:lem:subsubsub2:bound_M} it suffices to control the second moment of $M\mathbf{1}_{\cH_2}$ and prove that
\begin{equation*}
\dE[M^2\mathbf{1}_{\cH_2}\mid\bby]\ \leq\ \exp\left(o\left(n\log n\right)\right)\,\dE[M\mathbf{1}_{\cH_2}\mid\bby]^2.
\end{equation*}
To this end, we write
\begin{align}
\dE[M^2\mathbf{1}_{\cH_2}\mid\bby]
&=\sum_{\pi,\pi'\in\mathcal{S}_{n,n}} \dE\!\left[\mathbf{1}_{\pi\in\mathcal{M}_n}\mathbf{1}_{\pi'\in\mathcal{M}_n}\mathbf{1}_{\cH_2}\mid\bby\right]\nonumber\\
&=\sum_{\substack{\pi,\pi'\in\mathcal{S}_{n,n}\\|F_{\pi\cap\pi'}|\leq\frac{n}{\log n}}}\dE\!\left[\mathbf{1}_{\pi\in\mathcal{M}_n}\mathbf{1}_{\pi'\in\mathcal{M}_n}\mathbf{1}_{\cH_2}\mid\bby\right]\nonumber\\
&+\sum_{\substack{\pi,\pi'\in\mathcal{S}_{n,n}\\ |F_{\pi\cap\pi'}|> \frac{n}{\log n}}} \dE\!\left[\mathbf{1}_{\pi\in\mathcal{M}_n}\mathbf{1}_{\pi'\in\mathcal{M}_n}\mathbf{1}_{\cH_2}\mid\bby\right].\label{equ1}
\end{align} 

Conditional on $B,Y$, the graph noise $(\xi_{i,j})_{i<j}$ and
the feature noise vectors $(\xi'_{i,k})_{k\in[d]}$, $i\in[n]$,
are mutually independent. Consequently, the graph pair
$(W_G(\pi),W_G(\pi'))$ and the feature pairs
$(W_F^{\{i\}}(\pi),W_F^{\{i\}}(\pi'))$, indexed by $i\in[n]$,
are mutually independent under this conditional law.
For each $i\in F_{\pi\cap\pi'}$, we have $\pi(i)=\pi'(i)$,
so the corresponding feature constraints for the two permutations
coincide and need to be imposed only once. Considering separately the set of vertices in
$F_{\pi\cap\pi'}$ and the remaining ones, using the definition of $\mathcal M_n$
in \eqref{def:Mn}, we obtain 
\begin{align}
&\mathbb E\!\left[
\mathbf1_{\{\pi\in\mathcal M_n\}}
\mathbf1_{\{\pi'\in\mathcal M_n\}}
\mathbf1_{\mathcal H_2}
\,\middle|\,B,Y
\right]
\nonumber\\
&\quad\le
\mathbf1_{\mathcal H_2}\,
\mathbb P\!\left(
W_G(\pi)\ge \frac{\rho n}{\sqrt{2(1-\rho^2)}},
\;
W_G(\pi')\ge \frac{\rho n}{\sqrt{2(1-\rho^2)}}
\,\middle|\,B,Y
\right)
\nonumber\\
&\qquad\times
\mathbb P\!\left(
W_F^{F_{\pi\cap\pi'}}(\pi)
\ge
\frac{\eta\sqrt{d|F_{\pi\cap\pi'}|}}{\sqrt{1-\eta^2}}
\,\middle|\,B,Y
\right)
\nonumber\\
&\qquad\times
\prod_{i\in[n]\setminus F_{\pi\cap\pi'}}
\mathbb P\!\left(
W_F^{\{i\}}(\pi)\ge\frac{\eta\sqrt d}{\sqrt{1-\eta^2}},
\;
W_F^{\{i\}}(\pi')\ge\frac{\eta\sqrt d}{\sqrt{1-\eta^2}}
\,\middle|\,B,Y
\right).\label{co}
\end{align}
The decomposition in \eqref{equ1} separates pairs of permutations
according to their overlap. For pairs in the first group, the
small overlap allows us to compare their joint membership
probability in $\mathcal M_n$, conditional on $B,Y$, with the
product of their individual conditional probabilities, up to a
factor $\exp(o(n\log n))$. For pairs in the second group, the
dependence can be stronger, so we combine a direct bound on the
joint probability with an upper bound on the number of pairs
having each overlap size. The following lemma provides the
Gaussian joint-tail bounds needed for these two estimates.
Unlike the asymptotic bound in~\cite[Lemma~3.3]{ganassali2022sharp},
it does not require the thresholds to diverge, allowing us to
treat both groups without assuming that each signal term is
$\Omega(\log n)$.

\begin{lemma}
\label{lem:gaussian-upper-nonasymptotic}
Let $W_1$ and $W_2$ be jointly Gaussian random variables with zero
means and variances $1+b_n^{(1)}$ and $1+b_n^{(2)}$, where $b_n^{(1)}=o(1)$ and $b_n^{(2)}=o(1)$. Assume their correlation
coefficient is at most $\alpha_n$, where $\alpha_n\in[0,1]$. The following statements hold for every $t_n\geq0$:
\begin{enumerate}
    \item[(i)] Generally, 
    \begin{align*}
    \mathbb P(W_1\geq t_n,W_2\geq t_n)
    \leq
    \exp\left(
        -\frac{t_n^2}{(1+o(1))(1+\alpha_n)}
    \right).
    \end{align*}

    \item[(ii)] Particularly, if $\alpha_n\to 0$, then for sufficiently large $n$ we have
    \begin{align*}
        \mathbb P(W_1\geq t_n,W_2\geq t_n)
        \leq{}
        \exp\!\left(o(t_n^2)+O(\log(\sqrt{2}+t_n))\right)
        \mathbb P(W_1\geq t_n)\mathbb P(W_2\geq t_n).
    \end{align*}
\end{enumerate}
\end{lemma}

\begin{proof}\,
\paragraph{Proof of \Cref{lem:gaussian-upper-nonasymptotic} -- Part $(i)$.} 
If the correlation coefficient is negative, Slepian's inequality gives $\dP(W_1\geq t_n,W_2\geq t_n)\leq\dP(W_1\geq t_n)\dP(W_2\geq t_n)$, which implies both statements. Otherwise, since the left-hand side is non-decreasing in the correlation coefficient (Slepian's inequality), we may assume that the correlation coefficient equals $\alpha_n$. Since $W_1+W_2$ is a Gaussian random variable with variance $2+b^{(1)}_n+b^{(2)}_n+2\sqrt{1+b^{(1)}_n}\sqrt{1+b^{(2)}_n}\alpha_n\leq 2(1+\max\{|b^{(1)}_n|,|b^{(2)}_n|\})(1+\alpha_n)=2(1+o(1))(1+\alpha_n)$ we can bound the joint probability as follows:
\begin{align*}
\mathbb{P}(W_1 > t_n, W_2 > t_n)&\leq \dP(W_1+W_2 > 2t_n)\\
&\stackrel{(a)}{\leq} \exp\left(-\frac{4t^2_n}{2\textrm{Var}\left(W_1+W_2\right)}\right)\\
&\leq \exp\left(-\frac{t^2_n}{(1+o(1))(1+\alpha_n)}\right)
\end{align*}
where in $(a)$ we used the fact that $\mathbb{P}\left(\mathcal{N}(0,1)>x\right)\leq \exp(-x^2/2)$ for all $x\geq 0$.

\paragraph{Proof of \Cref{lem:gaussian-upper-nonasymptotic} -- Part $(ii)$.} 
Let $\kappa_n\vcentcolon=\max\{|b^{(1)}_n|,|b^{(2)}_n|\}$. The proof of part (i) gives
\begin{align}
\label{eq:upper-b-joint}
    \mathbb{P}(W_1 > t_n, W_2 > t_n)\leq \exp\left(-\frac{t^2_n}{(1+\kappa_n)(1+\alpha_n)}\right).
\end{align}
On the other hand, for each $W_i$ we have
\begin{align}
    \mathbb{P}\left(W_i>t_n\right)&=\mathbb{P}\left(\mathcal{N}(0,1)>\frac{t_n}{\sqrt{1+b^{(i)}_n}}\right)\nonumber\\
    &\geq \mathbb{P}\left(\mathcal{N}(0,1)>\frac{t_n}{\sqrt{1-\kappa_n}}\right)\nonumber\\
    &\stackrel{(a)}{\geq}\frac{1}{\sqrt{2\pi}\left(1+\frac{t_n}{\sqrt{1-\kappa_n}}\right)}\exp\left(-\frac{t^2_n}{2(1-\kappa_n)}\right)\nonumber\\
    &\stackrel{(b)}{\geq}\frac{1}{2\sqrt{\pi}\left(\sqrt{2}+t_n\right)}\exp\left(-\frac{t^2_n}{2(1-\kappa_n)}\right),\label{eq:lower-b-product}
\end{align}
where in $(a)$ we used the fact that $\mathbb{P}\left(\mathcal{N}(0,1)>x\right)\geq \frac{1}{\sqrt{2\pi}(1+x)}\exp(-x^2/2)$ for all $x\geq 0$, in $(b)$ we used the fact that for sufficiently large $n$ we have $\kappa_n\leq 1/2$.

Combining the upper bound in
\eqref{eq:upper-b-joint} with the lower bound
in \eqref{eq:lower-b-product} yields
\begin{align*}
\frac{\mathbb P(W_1\geq t_n,W_2\geq t_n)}
{\mathbb P(W_1\geq t_n)\mathbb P(W_2\geq t_n)}&\leq
4\pi(\sqrt{2}+t_n)^2
\exp\!\left(
\left(
\frac1{1-\kappa_n}
-\frac1{(1+\kappa_n)(1+\alpha_n)}
\right)t_n^2
\right)\\
&\stackrel{(a)}{\leq}4\pi(\sqrt{2}+t_n)^2\exp\left((3\kappa_n+\alpha_n)t_n^2\right)\\
&=\exp\left(O(\log(\sqrt{2}+t_n))+o(t_n^2)\right)
\end{align*}
where in $(a)$ we used the fact that for large enough n we have $\frac1{1-\kappa_n}
-\frac1{(1+\kappa_n)(1+\alpha_n)}=\frac{2\kappa_n}{1-\kappa^2_n}+\frac{\alpha_n}{(1+\kappa_n)(1+\alpha_n)}\leq 3\kappa_n+\alpha_n$.
\end{proof}

We now use Lemma~\ref{lem:gaussian-upper-nonasymptotic} to bound the sum over each group in~\eqref{equ1}.

\paragraph{Small overlap: $|F_{\pi\cap\pi'}|\leq\frac{n}{\log n}$.} For each pair $\pi,\pi'\in\mS_{n,n}$ such that $|F_{\pi\cap\pi'}|\leq\frac{n}{\log n}$, from~\eqref{corg} we have
\begin{align*}
\dE[W_G(\pi)W_G(\pi')\mid\bB]\mathbf{1}_{\mathcal{H}_2}\leq\frac{1}{\log^2 n}+O\left(\sqrt{\frac{\log n}{n}}\right) = o(1) \, .
\end{align*}  
Thus, by part $(ii)$ of Lemma~\ref{lem:gaussian-upper-nonasymptotic} and the fact that $\frac{\rho^2n^2}{2(1-\rho^2)}\leq 2n\log n$, we have
\begin{align*}
\dP&\left(W_G(\pi)>\frac{\rho n}{\sqrt{2(1-\rho^2)}},\ W_G(\pi')>\frac{\rho n}{\sqrt{2(1-\rho^2)}}\mid\bB\right)\mathbf{1}_{\mathcal{H}_2}\\
&\leq  \exp\left( o\left(n\log n\right)\right)\dP\!\left(W_G(\pi)>\frac{\rho n}{\sqrt{2(1-\rho^2)}}\mid\bB\right)\dP\!\left(W_G(\pi')>\frac{\rho n}{\sqrt{2(1-\rho^2)}}\mid\bB\right)\mathbf{1}_{\mathcal{H}_2}
\end{align*}
Similarly, by \eqref{eqcor} and the assumption that $d=\omega(\log n)$, for each $i\in [n]\setminus F_{\pi\cap\pi'}$, we have
\begin{align*}
\dE[W^{\{i\}}_F(\pi)W^{\{i\}}_F(\pi')\mid \bY]\mathbf{1}_{\mathcal{H}_2}=O\left(\sqrt{\frac{\log n}{d}}\right)\to 0.
\end{align*}
Therefore, applying part $(ii)$ of Lemma~\ref{lem:gaussian-upper-nonasymptotic} together with the fact that $\frac{\eta^2d}{1-\eta^2}\leq 2\log n$, we obtain
\begin{align*}
\dP&\left(W^{\{i\}}_F(\pi)>\frac{\eta \sqrt{d}}{\sqrt{1-\eta^2}},\ W^{\{i\}}_F(\pi')>\frac{\eta \sqrt{d}}{\sqrt{1-\eta^2}}\mid \bY\right)\mathbf{1}_{\mathcal{H}_2}\\
&\leq  \exp\left( o\left(\log n\right)\right)\dP\!\left(W^{\{i\}}_F(\pi)>\frac{\eta \sqrt{d}}{\sqrt{1-\eta^2}}\mid \bY\right)\dP\!\left(W^{\{i\}}_F(\pi')>\frac{\eta \sqrt{d}}{\sqrt{1-\eta^2}}\mid \bY\right)\mathbf{1}_{\mathcal{H}_2}
\end{align*} 

Also, to compare \eqref{co} with the product of the marginal probabilities, we must account for the feature contribution from $F_{\pi\cap\pi'}$. Since these conditions coincide for $\pi$ and $\pi'$, their probability appears only once in the joint probability, whereas it appears twice in the product of the marginals. We therefore need to bound the reciprocal of this probability to control the joint probability by the product of the marginals. Conditional on $\bY$, we have
$W^{F_{\pi\cap\pi'}}_F(\pi)\sim
\mathcal{N}\left(0,1+f^{I,(\pi)}_n\right)$,
where the variance is $1+o(1)$ on the event $\mathcal{H}_2$.
Hence, we have
\begin{align*}
    \dP\left(W^{F_{\pi\cap\pi'}}_F(\pi)\geq\frac{\eta \sqrt{d |F_{\pi\cap\pi'}|}}{\sqrt{1-\eta^2}}\mid\bY\right)^{-1}\mathbf{1}_{\mathcal{H}_2}&\stackrel{(a)}{\leq}\sqrt{2\pi}\left(1+\frac{\eta \sqrt{d |F_{\pi\cap\pi'}|}}{\sqrt{1-\eta^2}\sqrt{1+f^{I,(\pi)}_n}}\right)\exp\left(\frac{\frac{\eta^2d |F_{\pi\cap\pi'}|}{1-\eta^2}}{2\left(1+f^{I,(\pi)}_n\right)}\right)\mathbf{1}_{\mathcal{H}_2}\\
&\stackrel{(b)}{=} \exp\left(o(n\log n)\right)
\end{align*}
where $(a)$ follows from the Gaussian tail bound $\mathbb{P}\left(\mathcal{N}(0,1)>x\right)\geq \frac{1}{\sqrt{2\pi}(1+x)}\exp(-x^2/2)$ for all $x\geq 0$ and $(b)$ follows from $\frac{\eta^2d}{1-\eta^2}\leq 2\log n$, $|F_{\pi\cap\pi'}|=o(n)$ and $f^{I,(\pi)}_n=o(1)$. Note that the product of the marginals actually contains $\prod_{i\in F_{\pi\cap\pi'}}p_i(\pi)p_i(\pi')$, where $p_i(\pi):=\dP\big(W^{\{i\}}_F(\pi)\geq\frac{\eta\sqrt d}{\sqrt{1-\eta^2}}\mid\bY\big)$, which can be smaller than the square of the above probability; however, the same Gaussian tail bound together with $\frac{\eta^2d}{1-\eta^2}\leq 2\log n$ gives $p_i(\pi)\geq n^{-1-o(1)}$ on $\mathcal{H}_2$, so the reciprocal of this product is at most $e^{2|F_{\pi\cap\pi'}|\log n(1+o(1))}=e^{o(n\log n)}$ as well, since $|F_{\pi\cap\pi'}|\leq n/\log n$.

Now, plugging all these bounds together in Equation~\eqref{co}, we have 
\begin{align*}
    \dE\left[\mathbf{1}_{\pi\in\mathcal{M}_n}\mathbf{1}_{\pi'\in\mathcal{M}_n}\mathbf{1}_{\cH_2}\mid\bby\right]\leq e^{o(n\log n)}\dE\left[\mathbf{1}_{\pi\in\mathcal{M}_n}\mathbf{1}_{\cH_2}\mid\bby\right]\dE\left[\mathbf{1}_{\pi'\in\mathcal{M}_n}\mathbf{1}_{\cH_2}\mid\bby\right].
\end{align*}
Thus,
\begin{align*}
    \sum_{\substack{\pi,\pi'\in\mathcal{S}_{n,n}\\|F_{\pi\cap\pi'}|\leq \frac{n}{\log n}}}\hspace{-0.5cm}\dE\!\left[\mathbf{1}_{\pi\in\mathcal{M}_n}\mathbf{1}_{\pi'\in\mathcal{M}_n}\mathbf{1}_{\cH_2}\mid\bby\right]&\stackrel{}{\leq}e^{o(n\log n)}\hspace{-0.3cm}\sum_{\substack{\pi,\pi'\in\mathcal{S}_{n,n}}}\hspace{-0.3cm}\dE\left[\mathbf{1}_{\pi\in\mathcal{M}_n}\mathbf{1}_{\cH_2}\mid\bby\right]\dE\left[\mathbf{1}_{\pi'\in\mathcal{M}_n}\mathbf{1}_{\cH_2}\mid\bby\right]\\
    &\stackrel{}{=}e^{o(n\log n)}\dE[M\mathbf{1}_{\cH_2}|\bby]^2.
\end{align*}

\paragraph{Large overlap: $\frac{n}{\log n} < |F_{\pi\cap\pi'}| \leq n$.}
As explained before, this group consists of pairs of permutations with strong overlap, but the size of the whole group is small, which helps to control it and show that it is at most $e^{o(n\log n)}\,\dE[M\mathbf{1}_{\cH_2}|\bby]^2$, which is all we need. More precisely, based on Lemma~\ref{lem:gaussian-upper-nonasymptotic} part $(i)$, for $\pi,\pi'\in\mS_{n,n}$ such that $|F_{\pi\cap\pi'}|=m>\frac{n}{\log n}$, we have
\begin{align*}
    &\dP\!\left(W_G(\pi)>\frac{\rho n}{\sqrt{2(1-\rho^2)}},\ W_G(\pi')>\frac{\rho n}{\sqrt{2(1-\rho^2)}}\mid \bB\right)\mathbf{1}_{\mathcal{H}_2}\leq \exp\left(-\frac{\frac{\rho^2n^2}{2(1-\rho^2)}}{1+\left(\frac{m}{n}\right)^2(1+o(1))}\right)\\
    & \dP\!\left(W^{\{i\}}_F(\pi)>\frac{\eta \sqrt{d}}{\sqrt{1-\eta^2}},\ W^{\{i\}}_F(\pi')>\frac{\eta \sqrt{d}}{\sqrt{1-\eta^2}}\mid \bY\right)\mathbf{1}_{\mathcal{H}_2}\leq \exp\left(-\frac{\frac{\eta^2d}{1-\eta^2}}{1+o(1)}\right)\\
    & \dP\left(W^{F_{\pi\cap\pi'}}_F(\pi)\geq\frac{\eta \sqrt{d |F_{\pi\cap\pi'}|}}{\sqrt{1-\eta^2}}\mid \bY\right)\mathbf{1}_{\mathcal{H}_2}\stackrel{(a)}{\leq} \exp\left(-\frac{\frac{\eta^2dm}{1-\eta^2}}{2(1+o(1))}\right),
\end{align*}
where in $(a)$ we used the fact that $\mathbb{P}\left(\mathcal{N}(0,1)>x\right)\leq \exp(-x^2/2)$ for all $x\geq 0$.

Substituting these estimates into
\eqref{co}, and noting that there
are $n-m$ vertices outside $F_{\pi\cap\pi'}$, we obtain
\begin{align*}
    \dE\left[\mathbf{1}_{\pi\in\mathcal{M}_n}\mathbf{1}_{\pi'\in\mathcal{M}_n}\mathbf{1}_{\cH_2}\mid\bby\right]&\leq \exp\left(-\frac{\frac{\rho^2n^2}{2(1-\rho^2)}}{\left(1+\left(\frac{m}{n}\right)^2\right)(1+o(1))}-\frac{\frac{\eta^2d(n-m)}{1-\eta^2}}{1+o(1)}-\frac{\frac{\eta^2dm}{1-\eta^2}}{2(1+o(1))}\right)\\
    &\stackrel{}{\leq}\exp\left(-\frac{\frac{\rho^2n^2}{2(1-\rho^2)}}{\left(1+\left(\frac{m}{n}\right)^2\right)(1+o(1))}-\frac{\frac{\eta^2dn (1-\frac{m}{2n})}{1-\eta^2}}{1+o(1)}\right)\\
    &\stackrel{(a)}{\leq}\exp\left(-\frac{n}{2}\left(\frac{\rho^2n}{1-\rho^2}+\frac{2\eta^2d}{1-\eta^2}\right)\left(1-\frac{m}{2n}\right)(1-o(1))\right)\\
    &=\exp\left(-2(1-\tilde\varepsilon)n\log n \left(1-\frac{m}{2n}\right)(1-o(1))\right)
\end{align*}
where in $(a)$ we used the inequality $1-\frac{x}{2}\leq\frac{1}{1+x^2}$ for $x\in [0,1]$.

Therefore,
\begin{align*}
&\frac{1}{\dE[M\mathbf{1}_{\cH_2}|\bby]^2}\sum_{\substack{\pi,\pi'\in\mathcal{S}_{n,n}\\ |F_{\pi\cap\pi'}|>\frac{n}{\log n}}}\dE\left[\mathbf{1}_{\pi\in\mathcal{M}_n}\mathbf{1}_{\pi'\in\mathcal{M}_n}\mathbf{1}_{\cH_2}\mid\bby\right]\\
& \stackrel{(a)}{\leq}\exp\left(-2\tilde\varepsilon n\log n (1-o(1))\right)\times\\&\qquad\qquad\sum_{m>\frac{n}{\log n}}\exp\left((2n\log n -m\log m)(1-o(1))\right)\exp\left(-2(1-\tilde\varepsilon)n\log n \left(1-\frac{m}{2n}\right)(1-o(1))\right)\\
&\stackrel{(b)}{\leq} e^{o(n\log n)}\sum_{m>\frac{n}{\log n}} \exp\left(-\tilde\varepsilon m\log n\,(1+o(1))\right)\\
&\leq e^{o(n\log n)},
\end{align*}  
where $(a)$ follows from
$\dE[M\mathbf{1}_{\cH_2}\mid\bby]
\geq \exp\!\left(\tilde\varepsilon n\log n(1+o(1))\right)$,
the fact that Lemma~\ref{lem:gaussian-upper-nonasymptotic}
gives the same bound for all pairs of permutations with
$|F_{\pi\cap\pi'}|=m$, and the counting bound
\[
\left|\left\{(\pi,\pi')\in\mathcal S_{n,n}^2:
|F_{\pi\cap\pi'}|=m\right\}\right|
\leq |\mathcal S_{n,n}|\frac{n!}{m!}
\leq \frac{(n!)^2}{m!}
=\exp\left(2n\log n-m\log m+O(n)\right).
\]
In $(b)$, we use the fact that
$m\log m=m\log n(1-o(1))$ for
$\frac{n}{\log n}<m\leq n$, and we collect in the factor $e^{o(n\log n)}$ the additive errors produced by the $(1\pm o(1))$ factors multiplying terms of order $n\log n$. These errors are not $o(m\log n)$ when $m$ is close to $\frac{n}{\log n}$, which is why this group is only bounded by $e^{o(n\log n)}$ and not by $o(1)$.

Based on the proposed bounds for the two groups in \eqref{equ1}, we can conclude that 
\begin{align*}
\dE[M^2\mathbf{1}_{\cH_2}\mid\bby]\leq (1+o(1))\exp\left(o\left(n\log n\right)\right)\dE[M\mathbf{1}_{\cH_2}\mid\bby]^2.
\end{align*}
Therefore, by the conditional Paley–Zygmund inequality \eqref{cpz} with $c=1/2$, on the event $\cH_2$ we have $\dP\big(M\geq\frac12\dE[M\mid\bB,\bY]\,\big|\,\bB,\bY\big)\geq\frac14e^{-o(n\log n)}$, where the $o(n\log n)$ term is deterministic since all the bounds above hold uniformly on $\cH_2$. Together with Step~2.1, this gives \eqref{eq:lem:subsubsub2:bound_M}, and Step~1 completes the proof of Lemma~\ref{lem:subsubsub2}.
\end{proof}

\paragraph{Proof of \Cref{lem:main3} -- Step 3: wrapping things up.}
Lemma~\ref{lem:subsubsub2} gives
$\dE[\log(Z)\mid\bB,\bY]\ge\tilde\varepsilon n\log n-o(n\log n)$ on the event $\cH_2$.
Combining this with the lower-tail concentration estimate
of Lemma~\ref{lem:sub1_main3}, for a suitable positive sequence
$a_n=o(n\log n)$ with $b_n\to\infty$ (e.g.\ $a_n=(n\log n)^{3/4}$), we obtain
\(
\log Z\ge\tilde\varepsilon n\log n-o(n\log n)-a_n
\)
with probability at least
\(
1-\dP(\cH_2^c)-\exp\left(-b_n\right)=1-o(1),
\)
where $b_n=\Theta\left(\frac{a_n^2}{n\log n}\right)$. Exponentiating gives the desired lower bound on $Z$ and completes the proof of Lemma~\ref{lem:main3}.
\end{proof}

\section{Proof of Lemmas of \Cref{sec:exact}}\label{app:proof:sec:exact}
\subsection{Proof of Lemma \ref{lem:sec:exact:1}}

\begin{proof}[Proof of Lemma \ref{lem:sec:exact:1}]
\, 

\paragraph{First step: bounding $\dP(\cE_t)$ for $2 \leq t \leq \alpha_0 n$.}
We take $\alpha_0=\frac{\varepsilon}{1+\varepsilon}$ and first deal with the case $2 \leq t \leq \alpha_0 n$. The Hamiltonian, $V(\pi)$ in \eqref{eq:hamiltonian}, can be rewritten as follows
\begin{align}
V(\pi)&\stackrel{(a)}{=}\frac{\rho^2}{2(1-\rho^2)}\sum_{\substack{1\leq i< j \leq n\\ \pi(i,j)\neq\{i,j\}}}\left(B_{\pi(i,j)}-B_{i,j}\right)^2-\frac{\rho}{\sqrt{1-\rho^2}}\left(\sum_{\substack{1\leq i< j \leq n\\ \pi(i,j)\neq\{i,j\}}}\left(B_{\pi(i,j)}-B_{i,j}\right)\xi_{i,j}\right)\nonumber\\ \label{Vp}
&+\frac{\eta^2}{2(1-\eta^2)}\sum_{\substack{1\leq i \leq n, \ 1\leq k\leq d\\ \pi(i)\neq i}}\left(Y_{\pi(i),k}-Y_{i,k}\right)^2-\frac{\eta}{\sqrt{1-\eta^2}}\left(\sum_{\substack{1\leq i \leq n, \ 1\leq k\leq d\\ \pi(i)\neq i}}\left(Y_{\pi(i),k}-Y_{i,k}\right)\xi'_{i,k}\right)
\end{align}
where $(a)$ follows from $A_{i,j}=\rho B_{i,j}+\sqrt{1-\rho^2}\xi_{i,j}$ and $X_{i,k}=\eta Y_{i,k}+\sqrt{1-\eta^2}\xi'_{i,k}$ where $\xi_{i,j}$ and $\xi'_{i,k}$ are independent standard Gaussian random variables, together with the following equalities
\begin{align*}
    \sum_{\substack{1\leq i< j \leq n\\ \pi(i,j)\neq\{i,j\}}}B^2_{i,j}-B_{\pi(i,j)}B_{i,j}&=\frac{1}{2} \sum_{\substack{1\leq i< j \leq n\\ \pi(i,j)\neq\{i,j\}}}\left(B_{\pi(i,j)}-B_{i,j}\right)^2\\
    \sum_{\substack{1\leq i \leq n, \ 1\leq k\leq d\\ \pi(i)\neq i}}Y^2_{i,k}-Y_{\pi(i),k}Y_{i,k}&=\frac{1}{2}\sum_{\substack{1\leq i \leq n, \ 1\leq k\leq d\\ \pi(i)\neq i}}\left(Y_{\pi(i),k}-Y_{i,k}\right)^2
\end{align*}

Now, we control each term in \eqref{Vp}. Consider the event $\mathcal{H}_2$ defined in Lemma~\ref{lem:sub2}. Under $\mathcal{H}_2$, which holds with high probability, for each $\pi$ in $\mS_{n,t}$ we have the following:
\begin{align}
    \sum_{\substack{1\leq i< j \leq n\\ \pi(i,j)\neq\{i,j\}}}\left(B_{\pi(i,j)}-B_{i,j}\right)^2&=\sum_{\substack{1\leq i< j \leq n\\ \pi(i,j)\neq\{i,j\}}}B^2_{\pi(i,j)}+B^2_{i,j}-2B_{\pi(i,j)}B_{i,j}\nonumber\\
&=\sbp+\sigma^2_{\scaleto{\mathbf{B}}{3.9pt}}(\pi^{-1})-2\mathrm{Cov}_{\scaleto{\mathbf{B}}{3.9pt}}(\pi,\textrm{id})\nonumber\\
&\stackrel{(a)}{\leq}\mathbb{E}\left[\sbp\right]+\mathbb{E}\left[\sigma^2_{\scaleto{\mathbf{B}}{3.9pt}}(\pi^{-1})\right]-2\mathbb{E}\left[\mathrm{Cov}_{\scaleto{\mathbf{B}}{3.9pt}}(\pi,\textrm{id})\right]+Ct\sqrt{n\log n}\nonumber\\ \label{eq:varb}
&\stackrel{(b)}{=}2t\left(n-\frac{t}{2}\right)(1+o(1)),
\end{align}
where $(a)$ holds with high probability under the event $\mathcal{H}_2$ and $(b)$ follows from $\mathbb{E}\left[\sbp\right]=\mathbb{E}\left[\sigma^2_{\scaleto{\mathbf{B}}{3.9pt}}(\pi^{-1})\right]=|D^E_{\pi}|=|D_{\pi}|(n-|D_{\pi}|/2)-O(|D_{\pi}|)$ and $\mathbb{E}\left[\mathrm{Cov}_{\scaleto{\mathbf{B}}{3.9pt}}(\pi,\textrm{id})\right]=0$.

 Similarly, under the event $\mathcal{H}_2$, we have, with high probability,
\begin{align}
\sum_{\substack{1\leq i \leq n, \ 1\leq k\leq d\\ \pi(i)\neq i}}\left(Y_{\pi(i),k}-Y_{i,k}\right)^2&=\sum_{\substack{1\leq i \leq n, \ 1\leq k\leq d\\ \pi(i)\neq i}}Y^2_{\pi(i),k}+Y^2_{i,k}-2Y_{\pi(i),k}Y_{i,k}\nonumber\\
&=\syp+\sigma^2_{\scaleto{\mathbf{Y}}{3.9pt}}(\pi^{-1})-2\mathrm{Cov}_{\scaleto{\mathbf{Y}}{3.9pt}}(\pi,\textrm{id})\nonumber\\
&\stackrel{(a)}{\leq}\mathbb{E}\left[\syp\right]+\mathbb{E}\left[\sigma^2_{\scaleto{\mathbf{Y}}{3.9pt}}(\pi^{-1})\right]-2\mathbb{E}\left[\mathrm{Cov}_{\scaleto{\mathbf{Y}}{3.9pt}}(\pi,\textrm{id})\right]+Ct\sqrt{d\log n}\nonumber\\\label{eq:vary}
&\stackrel{(b)}{=}2td(1+o(1)),
\end{align}
where $(a)$ holds with high probability under the event $\mathcal{H}_2$ along with the fact that $d=\omega(\log n)$ and $(b)$ follows from $\mathbb{E}\left[\syp\right]=\mathbb{E}\left[\sigma^2_{\scaleto{\mathbf{Y}}{3.9pt}}(\pi^{-1})\right]=td$, $\mathbb{E}\left[\mathrm{Cov}_{\scaleto{\mathbf{Y}}{3.9pt}}(\pi,\textrm{id})\right]=0$ and $d=\omega(\log n)$.

Moreover, conditional on $\mathbf{B}$, the second term in the RHS of \eqref{Vp} is a centered Gaussian random variable
\begin{align*}
\sum_{\substack{1\leq i< j \leq n\\ \pi(i,j)\neq\{i,j\}}}\left(B_{\pi(i,j)}-B_{i,j}\right)\xi_{i,j}&\sim\mathcal{N}\left(0,\sum_{\substack{1\leq i< j \leq n\\ \pi(i,j)\neq\{i,j\}}}\left(B_{\pi(i,j)}-B_{i,j}\right)^2\right).
\end{align*}

Similarly, for the $\xi'_{i,k}$ part conditionally on $\mathbf{Y}$ we have 
\begin{align*}
\sum_{\substack{1\leq i \leq n, \ 1\leq k\leq d\\ \pi(i)\neq i}}\left(Y_{\pi(i),k}-Y_{i,k}\right)\xi'_{i,k}\sim\mathcal{N}\left(0,\sum_{\substack{1\leq i \leq n, \ 1\leq k\leq d\\ \pi(i)\neq i}}\left(Y_{\pi(i),k}-Y_{i,k}\right)^2\right)
\end{align*}

Now, we can rewrite $V(\pi)\leq 0$ condition on $\mathbf{B}$ and $\mathbf{Y}$ as
\begin{align*}
    \mathcal{N}(0,1)\geq\sqrt{\frac{\rho^2}{4(1-\rho^2)}\left(\sum_{\substack{1\leq i< j \leq n\\ \pi(i,j)\neq\{i,j\}}}\left(B_{\pi(i,j)}-B_{i,j}\right)^2\right)+\frac{\eta^2}{4(1-\eta^2)}\left(\sum_{\substack{1\leq i \leq n, \ 1\leq k\leq d\\ \pi(i)\neq i}}\left(Y_{\pi(i),k}-Y_{i,k}\right)^2\right)}.
\end{align*}
Since we controlled every term in the right hand side above under the event $\mathcal{H}_2$, we have for each $\pi$ in $\mS_{n,t}$,
\begin{align}
\dP\left(V(\pi)\leq0 \cap \mathcal{H}_2\right)&=\mathbb{E}_{\bby}\left[\dP\left(V(\pi)\leq0 \cap \mathcal{H}_2|\bby\right)\right]\nonumber\\
&\leq \mathbb{E}_{\bby}\left[\dP\left(\mathcal{N}\left(0,1\right)\geq\sqrt{\frac{1}{2}\left(\frac{\rho^2}{1-\rho^2}t\left(n-\frac{t}{2}\right)+\frac{\eta^2}{1-\eta^2}td\right)}(1-o(1))\right)\right]\label{ppee}\\
&\stackrel{(a)}{\leq} \exp\left(-\frac{1}{4}\left(\frac{\rho^2}{1-\rho^2}t\left(n-\frac{t}{2}\right)+\frac{\eta^2}{1-\eta^2}td\right)(1-o(1))\right)\nonumber\\
&\stackrel{(b)}{\leq} \exp\left(-\frac{1}{4}\left(\frac{\rho^2}{1-\rho^2}tn\left(1-\frac{\alpha_0}{2}\right)+\frac{\eta^2}{1-\eta^2}td\left(1-\frac{\alpha_0}{2}\right)\right)(1-o(1))\right)\nonumber\\
&\stackrel{(c)}{\leq} \exp\left(-(1+\varepsilon)\left(1-\frac{\alpha_0}{2}\right)t\log n(1-o(1))\right)\nonumber
\end{align}
where in $(a)$ we used the bound $\dP\left(\mathcal{N}(0,1)\geq x\right)\leq\exp(-x^2/2)$ , $(b)$ follows from $t\leq \alpha_0 n$, and $(c)$ follows from $\frac{\rho^2n}{1-\rho^2} + \frac{\eta^2d}{1-\eta^2} \geq 4(1 + \varepsilon) \log n$. 

Now, using union bound, we obtain 
\begin{align*}
\dP \left( \bigcup_{2\leq t \leq \alpha_0 n} \mathcal{E}_t \right)& \leq \dP\left(\mathcal{H}_2^c\right) + \dP \left( \bigcup_{2\leq t \leq \alpha_0 n} \mathcal{E}_t \cap \mathcal{H}_2\right)\\
&\leq o(1)+ \sum_{t=2}^{\lfloor\alpha_0 n\rfloor} |\mS_{n,t}| \ \exp\left(-(1+\varepsilon)\left(1-\frac{\alpha_0}{2}\right)t\log n(1+o(1))\right)\\
&\stackrel{(a)}{\leq} o(1)+\sum_{t=2}^{\lfloor\alpha_0 n\rfloor}  \exp\left(t\log n-(1+\varepsilon)\left(1-\frac{\alpha_0}{2}\right)t\log n(1+o(1))\right)\\
&\stackrel{(b)}{=}o(1)+\sum_{t=2}^{\lfloor\alpha_0 n\rfloor}  \exp\left(-\frac{\varepsilon}{2}t\log n(1+o(1))\right)=o(1),
 \end{align*} where in $(a)$ we used the bound $|\mS_{n,t}|\leq n^t= \exp(t\log n)$ and $(b)$ follows from $\alpha_0=\frac{\varepsilon}{1+\varepsilon}$.

\paragraph{Second step: bounding $\dP(\cE_t)$ for $t > \alpha_0 n$.}

We now deal with the remaining case $t >\alpha_0 n$. In this case, proving that
\[
\overline Z_{\alpha_0}(\mathrm{id})
: =
\sum_{\pi\in B(\mathrm{id},\alpha_0)^c}
\exp\bigl(-V(\pi)\bigr)
=o(1)
\]with high probability implies that, with high probability, $V(\pi)>0$ for every $\pi\in\mathcal S_{n,t}$ with $t>\alpha_0 n$.
Indeed, suppose that there exists some $\pi'\in\mathcal S_{n,t}$, with $t>\alpha_0n$, such that $V(\pi')\leq 0$. This would imply $\overline Z_{\alpha_0}(\mathrm{id}) \geq \exp\bigl(-V(\pi')\bigr) \geq 1$, a contradiction.

To control $\overline Z_{\alpha_0}(\mathrm{id})$, we return to the
Hamiltonian representation of $V$ in \eqref{eq:hamiltonian}:
\begin{align*}
    V(\pi)=\frac{\rho}{1-\rho^2}\left(V^*_{G}(\pi)-V_{G}(\pi)\right)+\frac{\eta}{1-\eta^2}\left(V^*_{F}(\pi)-V_{F}(\pi)\right).
\end{align*}
We now control each term in this expression separately. We begin with
$V^*_{G}(\pi)$. On the event $\mathcal H_1$ from
Lemma~\ref{lem:main1}, uniformly over all
$\pi\in\mS_{n,t}$ with $t>\alpha_0n$, we have
\begin{align}
V^*_{G}(\pi)&\geq\mathbb{E}\left[V^*_{G}(\pi)\right]-Ct\sqrt{n\log\left(\frac{en}{t}\right)}\nonumber\\&\stackrel{(a)}{\geq}\mathbb{E}\left[V^*_{G}(\pi)\right]-Ct\sqrt{n\log\left(e/\alpha_0\right)}\nonumber\\
&\stackrel{(b)}{\geq}\rho\left(t\left(n-\frac{t}{2}\right)-t\right)-Ct\sqrt{n\log\left(e/\alpha_0\right)}\label{eq:comp1}
\end{align}
where $(a)$ follows from $t>\alpha_0 n$ and $(b)$ follows from $\mathbb{E}\left[V^*_{G}(\pi)\right]=\rho|D^E_{\pi}|\geq\rho(|D_{\pi}|(n-|D_{\pi}|/2)-|D_{\pi}|)$.

Similarly for $V^*_{F}(\pi)$ under the event $\mathcal{H}_1$ we have
\begin{align}
V^*_{F}(\pi)&\stackrel{(a)}{\geq}\mathbb{E}\left[V^*_{F}(\pi)\right]-Ct\sqrt{d\log\left(\frac{en}{t}\right)}\nonumber\\&\stackrel{}{\geq}\mathbb{E}\left[V^*_{F}(\pi)\right]-Ct\sqrt{d\log\left(e/\alpha_0\right)}\nonumber\\
&\stackrel{(b)}{=}\eta td-Ct\sqrt{d\log\left(e/\alpha_0\right)}\label{eq:comp2}
\end{align}
where $(a)$ follows from the definition of event $\mathcal{H}_1$ along with the fact that $d=\omega(\log n)$; $(b)$ is based on $\mathbb{E}\left[V^*_{F}(\pi)\right]=\eta d|D_\pi|$.

For the remaining terms, $V_G(\pi)$ and $V_F(\pi)$, we apply
Lemma~\ref{lem:main2}. For every $\pi\in\mS_{n,t}$ we have
\begin{align}
\dE\left[e^{\frac{\rho}{1-\rho^2} V_G(\pi)}\mathbf{1}_{\mathcal{H}_2}\right]
& \leq \exp\left({\frac{\rho^2}{2(1-\rho^2)}\left(\left|D^E_{\pi}\right|+C t\sqrt{n\log n}\right)}\right)\nonumber\\\label{eq:comp3}
&\stackrel{}{=}\exp\left({\frac{\rho^2}{2(1-\rho^2)}\left(t\left(n-\frac{t}{2}\right)(1+o(1))\right)}\right)
\end{align} and
\begin{align}
\dE\left[e^{\frac{\eta}{1-\eta^2} V_F(\pi)}\mathbf{1}_{\mathcal{H}_2}\right]
&  \leq\exp\left({\frac{\eta^2}{2(1-\eta^2)}\left(dt+Ct\sqrt{d\log n}\right)}\right)\nonumber\\
&=\exp\left({\frac{\eta^2}{2(1-\eta^2)}\left(dt(1+o(1))\right)}\right)\label{eq:comp4}
\end{align}
where the estimates follow from $t> \alpha_0 n$ and $d=\omega(\log n)$.

Having controlled all four terms appearing in $V(\pi)$, we now combine
the preceding bounds. Thus, for every $\pi\in\mS_{n,t}$ with
$t>\alpha_0n$, we have
{\footnotesize
\begin{align}
&\dE\left[\exp(-V(\pi))\mathbf{1}_{\mathcal{H}_1\cap\mathcal{H}_2}\right]=\dE\left[\exp\left(-\frac{\rho}{1-\rho^2}\left(V^*_{G}(\pi)-V_{G}(\pi)\right)-\frac{\eta}{1-\eta^2}\left(V^*_{F}(\pi)-V_{F}(\pi)\right)\right)\mathbf{1}_{\mathcal{H}_1\cap\mathcal{H}_2}\right]\label{eq:expV}\\
&\stackrel{(a)}{\leq}\exp\left(-\left(\frac{\rho^2t\left(n-\frac{t}{2}\right)}{1-\rho^2}+\frac{\eta^2td}{1-\eta^2}\right)(1-o(1))\right)\dE\left[\exp\left(\frac{\rho}{1-\rho^2}V_{G}(\pi)+\frac{\eta}{1-\eta^2}V_{F}(\pi)\right)\mathbf{1}_{\mathcal{H}_2}\right]\nonumber\\
&\stackrel{(b)}{\leq}\hspace{-0.1cm}(1+o(1))\exp\left(\hspace{-0.1cm}-\left(\frac{\rho^2t\left(n-\frac{t}{2}\right)}{1-\rho^2}+\frac{\eta^2td}{1-\eta^2}\right)(1-o(1))\right)\hspace{-0.1cm}\dE\left[\exp\left(\frac{\rho}{1-\rho^2}V_{G}(\pi)\right)\mathbf{1}_{\mathcal{H}_2}\right]\hspace{-0.1cm}\dE\left[\exp\left(\frac{\eta}{1-\eta^2}V_{F}(\pi)\right)\mathbf{1}_{\mathcal{H}_2}\right]\nonumber\\ \label{rtrt2}
&\stackrel{(c)}{\leq} \exp\left(-\frac{1}{2}\left(\frac{\rho^2}{1-\rho^2}t\left(n-\frac{t}{2}\right)+\frac{\eta^2}{1-\eta^2}td\right)(1-o(1))\right)  \\
&\stackrel{(d)}{\leq} \exp\left(-\frac{t}{4}\left(\frac{\rho^2}{1-\rho^2}n+\frac{\eta^2}{1-\eta^2}d\right)(1-o(1))\right)\nonumber\\
&\stackrel{(e)}{\leq} \exp\left(-(1+\varepsilon)t\log n(1-o(1))\right)\nonumber
\end{align}
}
where in $(a)$ we used \eqref{eq:comp1} and \eqref{eq:comp2}, whose deviation terms (multiplied by $\frac{\rho}{1-\rho^2}$ and $\frac{\eta}{1-\eta^2}$) are $o(t\log n)$ plus a vanishing fraction of the term in the exponent, which is at least $2(1+\varepsilon)t\log n$, so they are absorbed in the $(1-o(1))$ factor; in $(b)$ we used independence between $V_G(\pi)$ and $V_F(\pi)$, the indicator $\mathbf{1}_{\mathcal{H}_2}$ only costing a factor $1+o(1)$, $(c)$ follows from bounds in \eqref{eq:comp3} and \eqref{eq:comp4}, $(d)$ is based on $n-t/2\geq n/2$ and $(e)$ follows from the fact that $\frac{\rho^2n}{1-\rho^2} + \frac{\eta^2d}{1-\eta^2} \geq 4(1 + \varepsilon) \log n$.

Now, we can control the posterior distribution outside the ball $B(\textrm{id},\alpha_0)$. To this end, first control its expectation:
\begin{align*}
\dE\left[\overline Z_{\alpha_0}(\mathrm{id})\mathbf{1}_{\mathcal{H}_1\cap\mathcal{H}_2}\right]&=\sum_{t=\left\lceil \alpha_0 n \right\rceil}^{n}\sum_{\pi\in\mS_{n,t}}\dE\left[\exp(-V(\pi))\mathbf{1}_{\mathcal{H}_1\cap\mathcal{H}_2}\right]\\
&\leq \sum_{t=\left\lceil \alpha_0 n \right\rceil}^{n}|\mS_{n,t}|\exp\left(-(1+\varepsilon)t\log n(1-o(1))\right)\\
&\leq \sum_{t=\left\lceil \alpha_0 n \right\rceil}^{n}\exp\left(-\varepsilon t\log n(1-o(1))\right)=o\left(\frac{1}{\log n}\right).
\end{align*} 

Since $\mathcal{H}_1$ and $\mathcal{H}_2$ are high probability events, together with Markov's inequality, we have
\begin{align*}
\overline Z_{\alpha_0}(\mathrm{id})=\overline Z_{\alpha_0}(\mathrm{id})\mathbf{1}_{\mathcal{H}_1\cap\mathcal{H}_2}\leq \log n \times \dE\left[\overline Z_{\alpha_0}(\mathrm{id})\mathbf{1}_{\mathcal{H}_1\cap\mathcal{H}_2}\right]=o(1),
\end{align*} with high probability.

This shows that $\dP \left( \bigcup_{t>\alpha_0 n} \mathcal{E}_t \right)$ is also of order of $o(1)$ and completes the proof of \Cref{lem:sec:exact:1}.
\end{proof}

\subsection{Proof of \Cref{lem:sec:exact:2}}

\begin{proof}[Proof of \Cref{lem:sec:exact:2}]
\, 

\paragraph{First step: controlling $\mathbb{E}\left[N\mathbf{1}_{\cH_2}\right]$. }
To show these two statements, without loss of generality we assume that
\begin{equation*}
\frac{\rho^2}{1-\rho^2}n+\frac{\eta^2}{1-\eta^2}d=4 \log n-\log \log n-a_n,
\end{equation*}
where $a_n=\omega(1)$ and $a_n=o(\log n)$. The assumption $a_n=o(\log n)$ is used in the second step below, where it gives $r_n^2=4\log n(1-o(1))$. It is without loss of generality for proving \Cref{thm:exact_recovery}, $(ii)$, which is the only use of \Cref{lem:sec:exact:2}. Indeed, the optimal probability of exact recovery is non-decreasing in $\rho$: for $\rho\leq\rho'<1$ with $\rho'>0$, data with correlation $\rho'$ can be turned into data with correlation $\rho$ and the same $\pi^\star$ by replacing $\bB$ with $\frac{\rho}{\rho'}\bB+\sqrt{1-\rho^2/\rho'^2}\,\bB'$, where $\bB'$ is an independent copy of $\bB$. Hence, if $a_n>\log\log n$, we may increase $\rho$ until $a_n=\log\log n$. 

For the first moment analysis, based on Equation~\eqref{Vp} for $t=2$ conditional on $(\bby)$, we have
\begin{align*}
V(\pi)=\frac{1}{2}H(\bB,\bY)-\mathcal{N}\left(0,H(\bB,\bY)\right),
\end{align*}
where,
\begin{align*}
H(\bB,\bY)\vcentcolon=\frac{\rho^2}{1-\rho^2}\left(\sum_{\substack{1\leq i< j \leq n\\ \pi(i,j)\neq\{i,j\}}}\left(B_{i,j}-B_{\pi(i,j)}\right)^2\right)+\frac{\eta^2}{1-\eta^2}\left(\sum_{\substack{1\leq i \leq n, \ 1\leq k\leq d\\ \pi(i)\neq i}}\left(Y_{i,k}-Y_{\pi(i),k}\right)^2\right).
\end{align*}
We have
\begin{align*}
\dE[N]&\stackrel{}{=} \binom{n}{2} \dE_{\mathbf{B,Y}}\biggl[
\dP\Bigl(
\mathcal{N}(0,1) \geq \frac{1}{2}\sqrt{H(\bB,\bY)}\Bigr)
\biggr]\\
& \stackrel{(a)}{\geq}(1-o(1)) \dE_{\mathbf{B,Y}}\biggl[\frac{n^2/2}{ \sqrt{2 \pi} \left(1+\frac12\sqrt{H(\bB,\bY)}\right)} \exp \left(-\frac{H(\bB,\bY)}{8} \right) \biggr]\\ 
&\stackrel{(b)}{\geq}\frac{(1-o(1))n^2/2}{ \sqrt{2 \pi} \left(1+\frac12\sqrt{\dE_{\mathbf{B,Y}}[H(\bB,\bY)]}\right)} \exp \left(-\frac{\dE_{\mathbf{B,Y}}[H(\bB,\bY)]}{8}\right)\\
&\stackrel{(c)}{\geq}\frac{(1-o(1))n^2/2}{\sqrt{2\pi}\left(1+2\sqrt{\log n}\right)}\exp\left(-2 \log n+\frac{\log \log n}{2}+\frac{a_n}{2}\right)\\
&=\frac{(1-o(1))}{4 \sqrt{2 \pi}} \exp \left(\frac{a_n}{2}\right) \rightarrow \infty,&&
\end{align*}
where in $(a)$ we used the fact that $\mathbb{P}\left(\mathcal{N}(0,1)\geq x\right)\geq\frac{1}{\sqrt{2\pi}(1+x)}\exp(-x^2/2)$ for all $x\geq 0$, in $(b)$ we used Jensen inequality since $f: t \mapsto \frac{1}{1+\sqrt{t}/2}\exp(-t/8)$ is convex for $t>0$, and $(c)$ follows from
\begin{align*}
\dE_{\mathbf{B,Y}}[H(\bB,\bY)]=\frac{\rho^2}{1-\rho^2}4(n-2)+\frac{\eta^2}{1-\eta^2}4d\leq 4\left(\frac{\rho^2n}{1-\rho^2}+\frac{\eta^2d}{1-\eta^2}\right)=4\left(4 \log n-\log \log n-a_n\right).
\end{align*}
Moreover, note that in the proof on Lemma~\ref{lem:sub2} we can choose a large enough constant in the definition of the event $\mathcal{H}_2$ such that $\mathbb{P}\left(\left(\mathcal{H}_2\right)^c\right)\leq n^{-3}$. Thus,
\begin{align*}
    \mathbb{E}\left[N\mathbf{1}_{\cH_2}\right]\geq \mathbb{E}\left[N\right]-n^2\mathbb{P}\left(\left(\mathcal{H}_2\right)^c\right)\geq \mathbb{E}\left[N\right]-\frac{1}{n}\geq\frac{(1-o(1))}{4 \sqrt{2 \pi}} \exp \left(\frac{a_n}{2}\right) \rightarrow \infty.
\end{align*}

\paragraph{Second step: second moment analysis. }

For the second moment analysis, we expand the second moment as follows:
\begin{equation}
\label{qq1}
\dE[N^2\mathbf{1}_{\cH_2}]=\dE[N\mathbf{1}_{\cH_2}]+\sum_{ \substack{\pi ,\pi '\in\mathcal{S}_{n,2}\\|D_{\pi}\cap D_{\pi'}|=0}} \dP\left(V(\pi)<0, V\left(\pi'\right)<0, \mathcal{H}_2\right)+\sum_{ \substack{\pi ,\pi '\in\mathcal{S}_{n,2}\\|D_{\pi}\cap D_{\pi'}|=1}} \dP\left(V(\pi)<0, V\left(\pi'\right)<0, \mathcal{H}_2\right).
\end{equation}
Moreover, combining the representation of the Hamiltonian $V(\pi)$
in \eqref{Vp} with the preceding bounds on its constituent terms,
each summand above can be bounded as follows:
\begin{equation*}
\dP\left(V(\pi)<0, V\left(\pi^{\prime}\right)<0, \mathcal{H}_2\right) = \mathbb{E}_{\bby}\left[\dP\left(G_{\pi}>r_n,G_{\pi'}>r_n\right)\mathbf{1}_{\cH_2}\right],
\end{equation*}
where $r_n=\sqrt{\frac{\rho^2}{1 - \rho^2} n + \frac{\eta^2}{1 - \eta^2} d}=\sqrt{4\log n-\log\log n-a_n}$, and $G_\pi:=\frac{2r_n}{H_\pi}\left(\frac{H_\pi}{2}-V(\pi)\right)$ with $H_\pi:=H(\bB,\bY)$ as above, so that $V(\pi)<0$ if and only if $G_\pi>r_n$. Conditional on $(\bby)$, $G_\pi$ and $G_{\pi'}$ are two centered Gaussian random variables with variances $\frac{4r_n^2}{H_\pi}=1+O\left(\sqrt{\frac{\log n}{n}}+\sqrt{\frac{\log n}{d}}\right)$ on the event $\mathcal{H}_2$ and with the following correlation coefficient
\begin{align*}
&c_n=(1-o(1))\times\\
&\frac{\frac{\rho^2}{1-\rho^2}\sum\limits_{\{i,j\}\in D^E_{\pi}\cap D^E_{\pi'}}\left(B_{\pi(i,j)}-B_{i,j}\right)\left(B_{\pi'(i,j)}-B_{i,j}\right)+\frac{\eta^2}{1-\eta^2}\sum\limits_{k\in[d],i\in D_{\pi}\cap D_{\pi'}}\left(Y_{\pi(i),k}-Y_{i,k}\right)\left(Y_{\pi'(i),k}-Y_{i,k}\right)}{\sqrt{\frac{\rho^2}{1-\rho^2}\hspace{-0.5cm}\sum\limits_{\substack{1\leq i< j \leq n\\ \pi(i,j)\neq\{i,j\}}}\hspace{-0.5cm}\left(B_{\pi(i,j)}-B_{i,j}\right)^2+\frac{\eta^2}{1-\eta^2}\hspace{-0.5cm}\sum\limits_{\substack{i\in[n],k\in[d]\\ \pi(i)\neq i}}\hspace{-0.5cm}\left(Y_{\pi(i),k}-Y_{i,k}\right)^2}\sqrt{\frac{\rho^2}{1-\rho^2}\hspace{-0.5cm}\sum\limits_{\substack{1\leq i< j \leq n\\ \pi'(i,j)\neq\{i,j\}}}\hspace{-0.5cm}\left(B_{\pi'(i,j)}-B_{i,j}\right)^2+\frac{\eta^2}{1-\eta^2}\hspace{-0.4cm}\sum\limits_{\substack{i\in[n],k\in[d]\\ \pi'(i)\neq i}}\hspace{-0.5cm}\left(Y_{\pi'(i),k}-Y_{i,k}\right)^2}}
\end{align*}

We now control the terms appearing in the covariance. In
\eqref{eq:varb} and \eqref{eq:vary}, we previously established upper
bounds for the denominator terms. Corresponding lower bounds of
the same asymptotic order can be obtained by an analogous argument. More precisely,
\begin{align*}
    \sum_{\substack{1\leq i< j \leq n\\ \pi(i,j)\neq\{i,j\}}}\left(B_{\pi(i,j)}-B_{i,j}\right)^2
&=\sbp+\sigma^2_{\scaleto{\mathbf{B}}{3.9pt}}(\pi^{-1})-2\mathrm{Cov}_{\scaleto{\mathbf{B}}{3.9pt}}(\pi,\textrm{id})\nonumber\\
&\stackrel{}{\geq}\mathbb{E}\left[\sbp\right]+\mathbb{E}\left[\sigma^2_{\scaleto{\mathbf{B}}{3.9pt}}(\pi^{-1})\right]-2\mathbb{E}\left[\mathrm{Cov}_{\scaleto{\mathbf{B}}{3.9pt}}(\pi,\textrm{id})\right]-C\sqrt{n\log n}\nonumber\\ 
&\stackrel{}{\geq}2t\left(n-\frac{t}{2}\right)-C\sqrt{n\log n},
\end{align*}
And similarly,
\begin{align*}
    \sum_{\substack{1\leq i \leq n, \ 1\leq k\leq d\\ \pi(i)\neq i}}\left(Y_{\pi(i),k}-Y_{i,k}\right)^2
&=\syp+\sigma^2_{\scaleto{\mathbf{Y}}{3.9pt}}(\pi^{-1})-2\mathrm{Cov}_{\scaleto{\mathbf{Y}}{3.9pt}}(\pi,\textrm{id})\nonumber\\
&\stackrel{}{\geq}\mathbb{E}\left[\syp\right]+\mathbb{E}\left[\sigma^2_{\scaleto{\mathbf{Y}}{3.9pt}}(\pi^{-1})\right]-2\mathbb{E}\left[\mathrm{Cov}_{\scaleto{\mathbf{Y}}{3.9pt}}(\pi,\textrm{id})\right]-C\sqrt{d\log n}\nonumber\\
&\stackrel{}{\geq}2td-C\sqrt{d\log n}
\end{align*}
It therefore remains to derive high-probability
upper bounds for the numerator terms.

Let $\pi$ and $\pi'$ be two distinct transpositions. Since
$\pi'(i,j)=(i,j)$ for every
$\{i,j\}\in D_\pi^E\setminus D_{\pi'}^E$, the corresponding summands
vanish. We may therefore extend the sum from
$D_\pi^E\cap D_{\pi'}^E$ to $D_\pi^E$ and write
\begin{align*}
&\sum_{\{i,j\}\in D^E_{\pi}\cap D^E_{\pi'}}
\left(B_{\pi(i,j)}-B_{i,j}\right)
\left(B_{\pi'(i,j)}-B_{i,j}\right)
\\
&\quad=
\sum_{\{i,j\}\in D^E_{\pi}}
    B_{\pi(i,j)}B_{\pi'(i,j)}
-\sum_{\{i,j\}\in D^E_{\pi}}
    B_{\pi(i,j)}B_{i,j}
+\sum_{\{i,j\}\in D^E_{\pi}}
    B_{i,j}^2
-\sum_{\{i,j\}\in D^E_{\pi}}
    B_{i,j}B_{\pi'(i,j)}
\\
&\quad\stackrel{(a)}{=}
\mathrm{Cov}_{\scaleto{\mathbf B}{3.9pt}}(\pi,\pi')
-\mathrm{Cov}_{\scaleto{\mathbf B}{3.9pt}}(\pi,\mathrm{id})
+\sbp
-\mathrm{Cov}_{\scaleto{\mathbf B}{3.9pt}}
    (\pi^{-1},\pi'\circ\pi^{-1})
\\
&\quad\stackrel{(b)}{\leq}
|D^E_\pi|
-|D^E_\pi\setminus D^E_{\pi'}|
+C\sqrt{n\log n}
\\
&\quad=
|D^E_\pi\cap D^E_{\pi'}|
+C\sqrt{n\log n}.
\end{align*}
Here, $(a)$ follows from the definitions of
$\mathrm{Cov}_{\scaleto{\mathbf B}{3.9pt}}$ and $\sbp$. For the final
term, we use the change of variables induced by $\pi^{-1}$ to obtain
\[
\sum_{\{i,j\}\in D^E_{\pi}}
B_{i,j}B_{\pi'(i,j)}
=
\sum_{\{i,j\}\in D^E_{\pi^{-1}}}
B_{\pi^{-1}(i,j)}
B_{(\pi'\circ\pi^{-1})(i,j)}.
\]
To obtain $(b)$, we apply the corresponding upper bounds
from the event $\mathcal H_2$, together with
\begin{align*}
\mathbb E\left[
\mathrm{Cov}_{\scaleto{\mathbf B}{3.9pt}}(\pi,\pi')
\right]
&=0,\\
\mathbb E\left[
\mathrm{Cov}_{\scaleto{\mathbf B}{3.9pt}}(\pi,\mathrm{id})
\right]
&=0,\\
\mathbb E[\sbp]
&=|D^E_\pi|,\\
\mathbb E\left[
\mathrm{Cov}_{\scaleto{\mathbf B}{3.9pt}}
(\pi^{-1},\pi'\circ\pi^{-1})
\right]
&=|D^E_\pi\setminus D^E_{\pi'}|.
\end{align*}
The first equality above uses the fact that $\pi$ and $\pi'$ are
distinct transpositions. 

Similarly, we have
\begin{align*}
    &\sum\limits_{k\in[d],i\in D_{\pi}\cap D_{\pi'}}\left(Y_{\pi(i),k}-Y_{i,k}\right)\left(Y_{\pi'(i),k}-Y_{i,k}\right)\\ &=\sum\limits_{k\in[d],i\in D_{\pi}}Y_{\pi(i),k}Y_{\pi'(i),k}-\sum\limits_{k\in[d],i\in D_{\pi}}Y_{\pi(i),k}Y_{i,k}+\sum\limits_{k\in[d],i\in D_{\pi}}Y^2_{i,k}-\sum\limits_{k\in[d],i\in D_{\pi}}Y_{i,k}Y_{\pi'(i),k}\\
&=\mathrm{Cov}_{\scaleto{\mathbf Y}{3.9pt}}(\pi,\pi')
-\mathrm{Cov}_{\scaleto{\mathbf Y}{3.9pt}}(\pi,\mathrm{id})
+\syp
-\mathrm{Cov}_{\scaleto{\mathbf Y}{3.9pt}}
    (\pi^{-1},\pi'\circ\pi^{-1})\\
&\leq d\left|D_\pi\right|-d\left|D_\pi\setminus D_{\pi'}\right|+C\sqrt{d\log n}=d\left|D_\pi\cap D_{\pi'}\right|+C\sqrt{d\log n}
\end{align*}

Now, we can upper bound the covariance. Since $|D^E_{\pi}\cap D^E_{\pi'}|=4$ when $|D_{\pi}\cap D_{\pi'}|=0$ and $|D^E_{\pi}\cap D^E_{\pi'}|=n-2$ when $|D_{\pi}\cap D_{\pi'}|=1$, then we have 
\begin{enumerate}[$(i)$]
\item If $|D_{\pi}\cap D_{\pi'}|=0$ then
\begin{align*}
c_n&\leq \frac{\frac{\rho^2}{1 - \rho^2} \left(4+C\sqrt{n\log n}\right) }{\frac{\rho^2}{1 - \rho^2} \left(4n-C\sqrt{n\log n}\right) + \frac{\eta^2}{1 - \eta^2} \left(4d-C\sqrt{d\log n}\right)}\\
&\leq \frac{4+C\sqrt{n\log n}}{4n-C\sqrt{n\log n}}=O\left(\sqrt{\frac{\log n}{n}}\right) \, .
\end{align*} 
\item If $|D_{\pi}\cap D_{\pi'}|=1$ then
\begin{align*}
c_n&\leq \frac{\frac{\rho^2}{1 - \rho^2} \left(n-2+C\sqrt{n\log n}\right) + \frac{\eta^2}{1 - \eta^2} \left(d+C\sqrt{d\log n}\right)}{\frac{\rho^2}{1 - \rho^2} \left(4n-C\sqrt{n\log n}\right) + \frac{\eta^2}{1 - \eta^2} \left(4d-C\sqrt{d\log n}\right)}\\
&\leq \max\left\{\frac{n-2+C\sqrt{n\log n}}{4n-C\sqrt{n\log n}},\frac{ d+C\sqrt{d\log n}}{4d-C\sqrt{d\log n}}\right\}\stackrel{(a)}{\sim} \frac{1}{4} \, . 
\end{align*} 
where in $(a)$ we used the fact that $d=\omega(\log n)$.
\end{enumerate}

We now bound each summand in the second-moment expression
\eqref{qq1}. 

To bound the first sum in \eqref{qq1}, we use the following
lemma, which bounds the joint upper-tail probability of two
Gaussian random variables whose correlation coefficient tends
to zero, at a threshold $t_n\to\infty$.
\begin{lemma}
\label{lem:cor}
Let $W_1$ and $W_2$ be jointly Gaussian random variables with zero mean and variances 
$1 + b^{(1)}_n$ and $1 + b^{(2)}_n$, respectively, where $b^{(1)}_n = o(1)$ and $b^{(2)}_n = o(1)$. 
Assume their correlation coefficient is at most $\alpha_n$, where $\alpha_n \in [0,1]$ and $\alpha_n\to 0$. Then, for any sequence $t_n \to \infty$, we have

    \[
        \dP(W_1 > t_n,\, W_2 > t_n) 
        \leq \exp\left(-\frac{C^2t^2_n}{2\left(1+b^{(1)}_n\right)}\right)+\exp(O(\alpha_n t^2_n))\dP\left(W_1>t_n\right)\dP\left(W_2>t_n\right),
    \]
where $C>1$ is a constant.
\end{lemma}
\begin{proof}[Proof of \Cref{lem:cor}]
If the correlation coefficient is negative, Slepian's inequality gives $\dP(W_1>t_n,W_2>t_n)\leq\dP(W_1>t_n)\dP(W_2>t_n)$ and the claim follows. Otherwise, since the left-hand side is non-decreasing in the correlation coefficient (Slepian's inequality), we may assume that the correlation coefficient equals $\alpha_n$. One can easily verify that $W_1$ and $W_2$ can be replaced by random variables $\sqrt{1+b^{(1)}_n}W$ and $\sqrt{1+b^{(2)}_n}\left(\alpha_n W+\sqrt{1-\alpha^2_n}W'\right)$ where $W$ and $W'$ are two independent standard Gaussian random variables. Moreover, we have the usual upper bound $\dP\left(W>Ct_n\right)\leq\exp\left(-C^2t^2_n/2\right)$. This gives
\begin{align*}
\dP\left(W_1>t_n, W_2>t_n\right)
&=\mathbb P(W_1>Ct_n,W_2>t_n)+\mathbb P(t_n<W_1\le Ct_n,W_2>t_n)\\ &\leq\dP\left(W_1>Ct_n\right)+\dP\left(W_1>t_n\right)\dP\left(W_2>t_n\Big|t_n<W_1\leq Ct_n\right)\\
&\leq \exp\left(-\frac{C^2t^2_n}{2\left(1+b^{(1)}_n\right)}\right)+\dP\left(W_1>t_n\right)\dP\left(\sqrt{1+b^{(2)}_n}W'>\frac{t_n-C\alpha_nt_n
\sqrt{\frac{1+b_n^{(2)}}{1+b_n^{(1)}}}}
{\sqrt{1-\alpha_n^2}}\right)\\
&\leq \exp\left(-\frac{C^2t^2_n}{2\left(1+b^{(1)}_n\right)}\right)+\dP\left(W_1>t_n\right)\dP\left(W_2>t_n-O\left(\alpha_n t_n\right)\right)\\
& \stackrel{(a)}{\leq} \exp\left(-\frac{C^2t^2_n}{2\left(1+b^{(1)}_n\right)}\right)+\exp(O(\alpha_n t^2_n))\dP\left(W_1>t_n\right)\dP\left(W_2>t_n\right),
\end{align*} 
where in $(a)$ we used the Gaussian tail-ratio bound $\dP(W_2>t_n-\delta_n)\leq(1+o(1))\,e^{t_n\delta_n/(1+b^{(2)}_n)}\,\dP(W_2>t_n)$, valid for $0\leq\delta_n=O(\alpha_n t_n)=o(t_n)$ and $t_n\to\infty$ (Mills-ratio asymptotics).
\end{proof}

Since $c_n r_n^2\to 0$ for the first sum in \eqref{qq1},
we apply \Cref{lem:cor}, together with the assumption that
$d=\omega(\log n)$, to bound this sum as follows:
\begin{align*}
&\sum_{ \substack{\pi ,\pi '\in\mathcal{S}_{n,2}\\|D_{\pi}\cap D_{\pi'}|=0}}\dP\left(V(\pi)<0, V\left(\pi^{\prime}\right)<0, \mathcal{H}_2\right)=\sum_{ \substack{\pi ,\pi '\in\mathcal{S}_{n,2}\\|D_{\pi}\cap D_{\pi'}|=0}}\mathbb{E}_{\bby}\left[\dP\left(G_{\pi}>r_n,G_{\pi'}>r_n\right)\mathbf{1}_{\cH_2}\right]\\
&\stackrel{(a)}{\leq} (1+o(1))\binom{n}{2}\binom{n-2}{2}\left[C'e^{-\frac{8r_n^2}{(1+o(1))}}+(1-o(1))\mathbb{E}_{\bby}\left[\dP\left(G_\pi>r_n\right)\mathbf{1}_{\cH_2}\right]\mathbb{E}_{\bby}\left[\dP\left(G_\pi'>r_n\right)\mathbf{1}_{\cH_2}\right]\right]\\
&\stackrel{(b)}{\leq} (1+o(1))\dE[N\mathbf{1}_{\cH_2}]^2,
\end{align*}
where $(a)$ follows from \Cref{lem:cor}, applied conditionally on $(\bby)$ under the event $\cH_2$ with $t_n=r_n$ (recall that $c_nr_n^2\to0$), together with \eqref{eq:product} below, and in $(b)$ we used the fact that $n^4e^{-\frac{8r^2_n}{1+o(1)}} \leq n^4e^{-32\log n (1-o(1))}=o(1)$.

We now prove the bound
\begin{equation}\label{eq:product}
\dE_{\bby}\left[\dP\left(G_\pi>r_n\right)\dP\left(G_{\pi'}>r_n\right)\mathbf{1}_{\cH_2}\right]\leq(1+o(1))\,\dE_{\bby}\left[\dP\left(G_\pi>r_n\right)\mathbf{1}_{\cH_2}\right]\dE_{\bby}\left[\dP\left(G_{\pi'}>r_n\right)\mathbf{1}_{\cH_2}\right],
\end{equation}
used above for two transpositions $\pi=(a\ b)$ and $\pi'=(c\ d)$ with $|D_{\pi}\cap D_{\pi'}|=0$. Conditional on $(\bby)$, we have $\dP\left(G_\pi>r_n\right)=f(H_\pi)$, where $f(x):=\dP\left(\mathcal{N}(0,1)\geq\frac{1}{2}\sqrt{x}\right)$. The variables $H_\pi$ and $H_{\pi'}$ are not independent, since both depend on the four edges between $\{a,b\}$ and $\{c,d\}$. We therefore write $H_\pi=\widetilde H_\pi+\Delta_\pi$, where
\begin{align*}
\Delta_\pi\vcentcolon=\frac{2\rho^2}{1-\rho^2}\left(\left(B_{a,c}-B_{b,c}\right)^2+\left(B_{a,d}-B_{b,d}\right)^2\right)\geq 0
\end{align*}
gathers the terms of $H_\pi$ involving these four edges, and similarly $H_{\pi'}=\widetilde H_{\pi'}+\Delta_{\pi'}$. Then $\widetilde H_\pi$, $\widetilde H_{\pi'}$ and $(\Delta_\pi,\Delta_{\pi'})$ are independent. Moreover, $f$ is non-increasing and, for all $x,\delta\geq0$,
\begin{equation}\label{eq:ratio}
f(x+\delta)\geq f(x)\exp\left(-\frac{\delta}{8}-\frac{\sqrt{\delta}}{2}\right),
\end{equation}
which follows by integrating the derivative of $\log f$ and using the bound $\dP\left(\mathcal{N}(0,1)\geq x\right)\geq\frac{1}{\sqrt{2\pi}(1+x)}\exp(-x^2/2)$ for $x\geq 0$, as in the first step. Hence,
\begin{align*}
\dE_{\bby}\left[f(H_\pi)f(H_{\pi'})\mathbf{1}_{\cH_2}\right]
&\stackrel{(a)}{\leq}\dE_{\bby}\left[f(\widetilde H_\pi)\right]\dE_{\bby}\left[f(\widetilde H_{\pi'})\right]\\
&\stackrel{(b)}{\leq}(1+o(1))\,\dE_{\bby}\left[f(H_\pi)\right]\dE_{\bby}\left[f(H_{\pi'})\right]\\
&\stackrel{(c)}{\leq}(1+o(1))\,\dE_{\bby}\left[f(H_\pi)\mathbf{1}_{\cH_2}\right]\dE_{\bby}\left[f(H_{\pi'})\mathbf{1}_{\cH_2}\right],
\end{align*}
where $(a)$ follows from $\Delta_\pi,\Delta_{\pi'}\geq0$, the monotonicity of $f$ and the independence of $\widetilde H_\pi$ and $\widetilde H_{\pi'}$; $(b)$ follows from \eqref{eq:ratio}, the independence of $\widetilde H_\pi$ and $\Delta_\pi$, and Jensen's inequality, which give
\begin{align*}
\dE_{\bby}\left[f(H_\pi)\right]\geq\dE_{\bby}\left[f(\widetilde H_\pi)\right]\dE_{\bby}\left[e^{-\frac{\Delta_\pi}{8}-\frac{\sqrt{\Delta_\pi}}{2}}\right]&\geq\dE_{\bby}\left[f(\widetilde H_\pi)\right]e^{-\frac{\rho^2}{1-\rho^2}-\sqrt{\frac{2\rho^2}{1-\rho^2}}}\\&=(1-o(1))\,\dE_{\bby}\left[f(\widetilde H_\pi)\right],
\end{align*}
since $\dE[\Delta_\pi]=\frac{8\rho^2}{1-\rho^2}$ and $\frac{\rho^2}{1-\rho^2}\leq\frac{4\log n}{n}$; and $(c)$ follows from $\dE_{\bby}\left[f(H_\pi)\right]\leq\dE_{\bby}\left[f(H_\pi)\mathbf{1}_{\cH_2}\right]+\dP\left(\left(\mathcal{H}_2\right)^c\right)$, where $\dP\left(\left(\mathcal{H}_2\right)^c\right)\leq n^{-3}$, while $\dE_{\bby}\left[f(H_\pi)\right]=\dE[N]/\binom{n}{2}\gg n^{-3}$ by the first step.

For the second summation in \eqref{qq1}, applying Lemma \ref{lem:gaussian-upper-nonasymptotic} part $(i)$ and considering $d=\omega(\log n)$ we obtain:
\begin{align*}
& \sum_{ \substack{\pi ,\pi '\in\mathcal{S}_{n,2}\\|D_{\pi}\cap D_{\pi'}|=1}} \dP\left(V(\pi)<0, V\left(\pi^{\prime}\right)<0, \mathcal{H}_2\right)=\sum_{ \substack{\pi ,\pi '\in\mathcal{S}_{n,2}\\|D_{\pi}\cap D_{\pi'}|=1}}\mathbb{E}_{\bby}\left[\dP\left(G_{\pi}>r_n,G_{\pi'}>r_n\right)\mathbf{1}_{\cH_2}\right]\\
&  \leq (1-o(1)) \frac{n(n-1)}{2} \times 2(n-2) \exp \left(-\frac{r_n^2}{(1+o(1))(1+c_n)}\right)\\
& \leq  n^3\exp \left(-(1-o(1))\frac{16}{5} \log n+o(\log n)\right)=o(1)=o\left(\dE[N\mathbf{1}_{\cH_2}]^2\right).
\end{align*}
This completes the proof of the fact that $\dE[N^2\mathbf{1}_{\cH_2}]\leq(1+o(1))\dE[N\mathbf{1}_{\cH_2}]^2$. 
\end{proof}

\section{Proof of Lemmas of \Cref{sec:almost-exact-partial}}\label{app:proof:section5}
\subsection{Proof of Lemma \ref{lem:almost_exact}}
\begin{proof}[Proof of Lemma \ref{lem:almost_exact}]
Fix $r$ in $(0,1)$. Going back to \eqref{rtrt2} in Appendix \ref{app:proof:sec:exact}, we have, for $t \in [rn,n]$, for $\pi \in \mS_{n,t}$,
\begin{align*}
\mathbb{E}\left[\exp(-V(\pi))\mathbf{1}_{\mathcal{H}_1\cap\mathcal{H}_2}\right]
& \leq  \exp\left(-\frac{1}{2}\left(\frac{\rho^2}{1-\rho^2}t\left(n-\frac{t}{2}\right)+\frac{\eta^2}{1-\eta^2}td\right)(1-o(1))\right)\\
&\leq \exp\left(-\frac{1}{4}\left(\frac{\rho^2}{1-\rho^2}nt+\frac{2\eta^2}{1-\eta^2}dt\right)(1-o(1))\right)\\
& \leq \exp\left(-(1+\varepsilon)t\log n(1-o(1))\right)
\end{align*}
where the last inequality follows from $\frac{\rho^2}{1-\rho^2}n+\frac{2\eta^2}{1-\eta^2}d\geq 4(1+\varepsilon)\log n$.

Using the previous bound,
\begin{align*}
\mathbb{E}\left[\overline Z_r(\textrm{id}) \mathbf{1}_{\mathcal{H}_1\cap\mathcal{H}_2}\right]&=\sum_{t=rn}^{n}\sum_{\pi\in\mS_{n,t}}\mathbb{E}\left[\exp(-V(\pi))\mathbf{1}_{\mathcal{H}_1\cap\mathcal{H}_2}\right]\\
&\leq \sum_{t=rn}^{n}|\mS_{n,t}|\exp\left(-(1+\varepsilon)t\log n(1-o(1))\right)\\
&\stackrel{(a)}{\leq} \sum_{t=rn}^{n}\exp\left(-\varepsilon t\log n(1-o(1))\right)=o\left(\frac{1}{\log n}\right).
\end{align*}
where in $(a)$ we used the fact $|\mS_{n,t}| \leq n^t$. 
\end{proof}

\subsection{Proof of Lemma \ref{lem:partial}}
\begin{proof}[Proof of Lemma \ref{lem:partial}]
Write
\begin{equation}\label{eq:proof:thm:impossibility_partial_recovery_1}
    \frac{\rho^2}{1-\rho^2}n+\frac{2\eta^2}{1-\eta^2}d \,=\, 4(1-\tilde{\eps})\log n,
\end{equation} with $\tilde{\eps}$ depending on $n,d,\rho,\eta$, with $\tilde{\eps} \geq \eps$ by assumption. 

Let us fix $r\in (0,1)$. Similar to the proof \Cref{lem:almost_exact} in last section, for each $\pi$ in $\mS_{n,t}$ with $\alpha n\leq t\leq rn$, where $\alpha$ is a constant to be chosen later, we have
\begin{align*}
\mathbb{E}\left[\exp(-V(\pi))\mathbf{1}_{\mathcal{H}_1\cap\mathcal{H}_2}\right]
& \leq \exp\left(-\frac{1}{4}\left(\frac{\rho^2}{1-\rho^2}nt+\frac{2\eta^2}{1-\eta^2}dt\right)(1-o(1))\right)\\
& = \exp\left(-(1-\tilde{\eps})t\log n(1-o(1))\right) \, .
\end{align*} 
Consequently, 
\begin{align}
\sum_{t=\left\lceil \alpha n \right\rceil}^{rn}\sum_{\pi\in\mS_{n,t}}\mathbb{E}\left[\exp(-V(\pi))\mathbf{1}_{\mathcal{H}_1\cap\mathcal{H}_2}\right]
&\leq \sum_{t=\left\lceil \alpha n \right\rceil}^{rn}|\mS_{n,t}|\exp\left(-(1-\tilde{\eps})t\log n(1-o(1))\right)\nonumber\\
&\leq \sum_{t=\left\lceil \alpha n \right\rceil}^{rn}\exp\left(\tilde{\eps} t\log n(1-o(1))\right)\nonumber\\ \label{eq:5.2.1}
&\leq \exp\left(\tilde{\eps} rn\log n(1+o(1))\right) \, .
\end{align}
On the other hand, for each $\pi\in\mS_{n,t}$ with $t<\alpha n$, we can still use the same inequality, but we must also account for the deviation terms, since they are no longer necessarily negligible relative to the expectation terms. More precisely, by \eqref{eq:comp1} and \eqref{eq:comp2}, conditional on the event $\mathcal{H}_1$, we have, with high probability,
\begin{align*}
\frac{\rho}{1-\rho^2}V^*_{G}(\pi)&\geq\frac{\rho}{1-\rho^2}\mathbb{E}\left[V^*_{G}(\pi)\right]-\frac{\rho}{1-\rho^2}Ct\sqrt{n\log\left(\frac{en}{t}\right)}\\
&\stackrel{}{\geq}\frac{\rho^2}{1-\rho^2}\left(t\left(n-\frac{t}{2}\right)-t\right)-\frac{\rho}{1-\rho^2}Ct\sqrt{n\log\left(en\right)}\\
&\stackrel{(a)}{\geq}\frac{\rho^2}{1-\rho^2} t\left(n-\frac{t}{2}\right)-C't\log n
\end{align*}
where in $(a)$, $C'>0$ is a constant and we used the fact that
\[
\frac{\rho}{1-\rho^2}\sqrt{n}=\frac{1}{\sqrt{1-\rho^2}}\sqrt{\frac{\rho^2n}{1-\rho^2}}\leq(1+o(1))\sqrt{4\log n} \, .
\] Similarly, we have
\begin{align*}
\frac{\eta}{1-\eta^2}V^*_{F}(\pi)&\geq\frac{\eta}{1-\eta^2}\mathbb{E}\left[V^*_{F}(\pi)\right]-\frac{\eta}{1-\eta^2}Ct\sqrt{d\log\left(\frac{en}{t}\right)}\\
&\stackrel{}{\geq}\frac{\eta^2}{1-\eta^2}td-\frac{\eta}{1-\eta^2}Ct\sqrt{d\log\left(en\right)}\\
&\stackrel{(a)}{\geq}\frac{\eta^2}{1-\eta^2} td-C't\log n
\end{align*}
where in $(a)$ we used the fact that $\frac{\eta}{1-\eta^2}\sqrt{d}=\frac{1}{\sqrt{1-\eta^2}}\sqrt{\frac{\eta^2d}{1-\eta^2}}\leq(1+o(1))\sqrt{4\log n}$.
Substituting these bounds into \eqref{eq:expV} and treating the
remaining terms as in its derivation, we obtain
\begin{align*}
\mathbb{E}\left[\exp(-V(\pi))\mathbf{1}_{\mathcal{H}_1\cap\mathcal{H}_2}\right]
& \leq \exp\left(-\frac{1}{4}\left(\frac{\rho^2}{1-\rho^2}nt+\frac{2\eta^2}{1-\eta^2}dt\right)+C_0t\log n\right)\\
& = \exp\left(-(1-\tilde{\eps})t\log n(1-o(1))+C_0t\log n\right),
\end{align*}
where $C_0>0$ is a big enough constant. 
Thus,
\begin{align*}
\sum_{t=1}^{\left\lfloor \alpha n \right\rfloor}\sum_{\pi\in\mS_{n,t}}\mathbb{E}\left[\exp(-V(\pi))\mathbf{1}_{\mathcal{H}_1\cap\mathcal{H}_2}\right] 
&\leq \sum_{t=1}^{\left\lfloor \alpha n \right\rfloor}|\mS_{n,t}|\exp\left(-(1-\tilde{\eps})t\log n(1-o(1))+C_0t\log n\right)\\
&\leq \sum_{t=1}^{\left\lfloor \alpha n \right\rfloor}\exp\left(\tilde{\eps}t\log n(1+o(1))+C_0t\log n\right)\\
&\stackrel{}{\leq} \exp\left((\tilde{\eps}+C_0)\alpha n\log n(1+o(1))\right) \, .
\end{align*}
Now, setting $\alpha=\frac{\tilde{\eps}r}{\tilde{\eps}+C_0}\leq r$, we have 
\begin{align}
\label{eq:5.2.2}
    \sum_{t=1}^{\left\lfloor \alpha n \right\rfloor}\sum_{\pi\in\mS_{n,t}}\mathbb{E}\left[\exp(-V(\pi))\mathbf{1}_{\mathcal{H}_1\cap\mathcal{H}_2}\right]\leq \exp\left(\tilde{\eps}r n\log n(1+o(1))\right)
\end{align}

Finally, combining \eqref{eq:5.2.1} and \eqref{eq:5.2.2} yields
\begin{align*}
\mathbb{E}\left[Z_r(\pi^{\star})\mathbf{1}_{\mathcal{H}_1\cap\mathcal{H}_2}\right]&=\sum_{t=0}^{rn}\sum_{\pi\in\mS_{n,t}}\mathbb{E}\left[\exp(-V(\pi))\mathbf{1}_{\mathcal{H}_1\cap\mathcal{H}_2}\right]\\
&\stackrel{}{\leq} \exp\left(\tilde{\eps} rn\log n(1+o(1))\right) \, .
\end{align*} 
\end{proof}

\subsection{Proof of Lemma \ref{lem:partial-recovery-imp}}
\begin{proof}[Proof of Lemma \ref{lem:partial-recovery-imp}]
Fix \(\pi\in \mS_n\). Note that $\{\pi':\ov(\pi',\pi)\ge\delta\}=B(\pi,1-\delta)$, and define
\[
    m_{\delta,\pi}:=\mu(B(\pi,1-\delta)).
\]
Suppose that $m_{\delta,\pi}>0$, and let $\pi_1$ and $\pi_2$ be drawn
independently from \(\mu(\,\cdot\mid B(\pi,1-\delta))\), that is, from $\mu$ when conditioned to be in $B(\pi,1-\delta)$. Then, we have
\begin{align*}
\E[\ov(\pi_1,\pi_2)]&=\E\left[\frac{1}{n}\sum_{i=1}^{n}\mathbf{1}_{\{\pi_1(i)=\pi_2(i)\}}\right]\\
&\geq \E\left[\frac{1}{n}\sum_{i=1}^{n}\mathbf{1}_{\{\pi_1(i)=\pi(i)\}}\mathbf{1}_{\{\pi_2(i)=\pi(i)\}}\right]\\
&\stackrel{(a)}{=}\frac{1}{n}\sum_{i=1}^{n}\E\left[\mathbf{1}_{\{\pi_1(i)=\pi(i)\}}\right]^2\\
&\stackrel{(b)}{\geq} \left(\E\left[\frac{1}{n}\sum_{i=1}^{n}\mathbf{1}_{\{\pi_1(i)=\pi(i)\}}\right]\right)^2\\
&=\left(\E\left[\textrm{overlap}(\pi_1,\pi)\right]\right)^2\stackrel{(c)}{\geq}\delta^2
\end{align*}
where in $(a)$ we used the fact that $\pi_1$ and $\pi_2$ are i.i.d.,
in $(b)$ we used Cauchy–Schwarz inequality, and $(c)$ follows from the fact
that $\pi_1\in B(\pi,1-\delta)$ implies
$\ov(\pi_1,\pi)\geq\delta$. Besides, since the overlap lies in \([0,1]\), we have
\begin{align*}
\delta^2
&\leq \E[\ov(\pi_1,\pi_2)]
\\
&=
\E\left[
\ov(\pi_1,\pi_2)
\mathbf{1}_{\left\{
\ov(\pi_1,\pi_2)\geq \frac{\delta^2}{2}
\right\}}
\right]
+
\E\left[
\ov(\pi_1,\pi_2)
\mathbf{1}_{\left\{
\ov(\pi_1,\pi_2)<\frac{\delta^2}{2}
\right\}}
\right]
\\
&\leq
\dP(\ov(\pi_1,\pi_2)\geq {\delta^2}/{2})
+
\frac{\delta^2}{2}
\left[
1-
\dP(\ov(\pi_1,\pi_2)\geq {\delta^2}/{2})
\right],
\end{align*} which implies $\dP(\ov(\pi_1,\pi_2)\geq {\delta^2}/{2})\geq  \frac{\delta^2}{2-\delta^2}$. Moreover, we have
\begin{align*}
&\dP(\ov(\pi_1,\pi_2)\geq {\delta^2}/{2})\\
&\qquad =
\frac{1}{\mu\left(B(\pi,1-\delta)\right)^2}
\sum_{\pi_1\in B(\pi,1-\delta)}
\mu(\pi_1)
\mu\left(
\left\{
\pi_2\in B(\pi,1-\delta):
\ov(\pi_1,\pi_2)
\geq\frac{\delta^2}{2}
\right\}
\right)
\\
&\qquad \leq
\frac{1}{\mu\left(B(\pi,1-\delta)\right)^2}
\sum_{\pi_1\in B(\pi,1-\delta)}
\mu(\pi_1)
\mu\left(
\left\{
\pi_2:
\ov(\pi_1,\pi_2)
\geq\frac{\delta^2}{2}
\right\}
\right)
\\
&\qquad\leq
\frac{A_{\delta^2/2}}
{\mu\left(B(\pi,1-\delta)\right)^2}.
\end{align*}

Therefore,
\[
    A_{\delta^2/2}
    \ge m_{\delta,\pi}^2\frac{\delta^2}{2-\delta^2}.
\]
Note that the inequality still holds if $m_{\delta,\pi}=0$. Finally, maximizing over \(\pi\) yields
\[
    M_\delta^2
    \le \frac{2-\delta^2}{\delta^2}\,
       A_{\delta^2/2},
\] the desired bound.
\end{proof}
\end{appendices}

\end{document}